\documentclass{article}

\usepackage{microtype}
\usepackage{graphicx}
\usepackage{subcaption}
\usepackage{booktabs} 
\usepackage{mdframed}
\usepackage{tikz}
\usetikzlibrary{calc,decorations.pathreplacing}

\usepackage{hyperref}
\usepackage{xspace}

\usepackage{amsmath}
\usepackage{amssymb}

\usepackage{enumitem}

\usepackage[preprint]{icml2026}

\usepackage{mathtools}
\usepackage{amsthm}
\usepackage{thm-restate}

\usepackage{relsize}
\usepackage{capt-of}
\usepackage{longtable}
\usepackage{tcolorbox}
\usepackage{mathrsfs}

\usepackage{color, comment, environ, bbm}
\usepackage{bm}
\usepackage{multirow}
\usepackage{siunitx}
\renewcommand{\algorithmicensure}{\textbf{Output:}}

\newcommand{\ourmethod}{transport clustering\xspace}

\newcommand{\ourmethodshort}{\texttt{TC}\xspace}

\newcommand{\defeq}{\triangleq}

\DeclareMathOperator*{\conv}{conv}

\DeclareMathOperator{\rank}{rank}
\DeclareMathOperator{\diag}{diag}

\DeclareMathOperator*{\argmin}{arg\,min}
\DeclareMathOperator*{\argmax}{arg\,max}

\graphicspath{ {./images/} }

\DeclareMathOperator*{\tr}{tr}

\newcommand{\mb}{\mathbf}

\newcommand{\LOT}{\texttt{LOT}\xspace}
\newcommand{\FRLC}{\texttt{FRLC}\xspace}
\newcommand{\LIN}{\texttt{LatentOT}\xspace}

\newcommand{\FC}{\texttt{FactoredOT}\xspace}

\usepackage[capitalize,noabbrev]{cleveref}

\theoremstyle{plain}
\newtheorem{theorem}{Theorem}[section]
\newtheorem{proposition}[theorem]{Proposition}
\newtheorem{lemma}[theorem]{Lemma}
\newtheorem{corollary}[theorem]{Corollary}
\theoremstyle{definition}
\newtheorem{definition}[theorem]{Definition}

\theoremstyle{remark}
\newtheorem{remark}[theorem]{Remark}

\usepackage[textsize=tiny]{todonotes}

\icmltitlerunning{Transport Clustering: Solving Low-Rank Optimal Transport via Clustering}

\begin{document}

\twocolumn[
  \icmltitle{Transport Clustering: Solving Low-Rank \\ Optimal Transport via Clustering}



  \icmlsetsymbol{equal}{*}

  \begin{icmlauthorlist}
    \icmlauthor{Henri Schmidt}{equal,xxx}
    \icmlauthor{Peter Halmos}{equal,xxx}
    \icmlauthor{Ben Raphael}{xxx}
  \end{icmlauthorlist}

  \icmlaffiliation{xxx}{Department of Computer Science, Princeton University, NJ, USA}

  \icmlcorrespondingauthor{Ben Raphael}{braphael@princeton.edu}

  \icmlkeywords{Machine Learning, ICML}

  \vskip 0.3in
]



\printAffiliationsAndNotice{\icmlEqualContribution}

\begin{abstract}
Optimal transport (OT) finds a least cost 
transport plan 
between two probability distributions using a cost 
matrix defined on pairs of points. 
Unlike standard OT, which infers unstructured pointwise mappings, low-rank optimal transport explicitly constrains the rank of the transport plan to infer latent structure.
This improves statistical stability and robustness, yields sharper parametric rates for estimating Wasserstein distances adaptive to the intrinsic rank,
and generalizes $K$-means to co-clustering.
These advantages, however, come at the cost of a 
non-convex and NP-hard optimization problem. 
We introduce \ourmethod, an algorithm to compute a low-rank OT plan that reduces low-rank OT to a clustering
problem on correspondences obtained from a full-rank \emph{transport registration} step.
We prove that this reduction yields polynomial-time, constant-factor approximation algorithms for low-rank OT: specifically, a $(1+\gamma)$ approximation for negative-type metrics and a $(1+\gamma+\sqrt{2\gamma}\,)$ approximation for kernel costs, where $\gamma \in [0,1]$ denotes the approximation 
ratio of the optimal full-rank solution relative to the low-rank optimal. Empirically, \ourmethod\ outperforms existing low-rank OT solvers 
on synthetic benchmarks and large-scale, high-dimensional 
datasets.
\end{abstract}

\section{Introduction}

Optimal transport finds a mapping between two probability distributions in a space $M$ provided an appropriate cost function $c: M \times M \to \mathbb{R}$ between pairs of points in the space. When $c$ is the squared Euclidean cost, the cost $W_{2}^{2}(\mu, \nu)$ of the optimal map 
between two probability distributions $\mu$ and $\nu$ supported on $M$ is known as the Wasserstein distance or Earth Mover's distance, and is one of the most natural and popular metrics for assessing the distance between two probability distributions.

OT has been widely used in machine learning and scientific applications because of its ability to resolve correspondences between unregistered datasets. In machine learning, optimal transport has found applications in generative modeling \citep{tong2023improving,korotin2023neural,KorotinLGSFB21}, self-attention \citep{tay20a, sander22a, geshkovski2023mathematical}, unpaired data translation \citep{KorotinLGSFB21, bortoli2024schrodinger, tong2024improving, klein2024generative}, and alignment problems in transformers and LLMs \citep{melnyk2024distributional, li2024gilot}. Moreover, OT has become an essential tool in science, with wide-ranging applications from biology \citep{schiebinger2019optimal, yang2020predicting, zeira2022PASTE, Bunne_2023, destot, klein2023moscot} to particle physics \citep{PhysRevLett.123.041801, Ba_2023, manole2024backgroundmodelingdoublehiggs}. 

In the discrete setting, where the source and target distributions are empirical measures over $n$ points, the Kantorovich OT problem \cite{kantorovich1942transfer} minimizes a linear cost over the set of non-negative matrices with fixed row and column marginals. With uniform marginals, optimal
solutions coincide with the vertices of the Birkhoff polytope \cite{birkhoff1946tres} -- i.e. the set of permutation matrices -- yielding deterministic, one-to-one mappings.
However, high-dimensional transport plans are often well-described by an interpretable, low-dimensional process. Specifically, transport
plans in high-dimensions often factor through a small number of latent factors or
anchors, reflecting a low intrinsic rank \cite{forrow19a, lin2021making}. Thus, while the ``true'' coupling often exhibits interpretable low-rank structure with a rapidly decaying spectrum,
full-rank OT is inherently incapable of finding such structure: a full-rank OT solution is a permutation matrix $\mb{P}$ with a flat spectrum of constant singular values $\sigma_{1}(\mb{P}) = \ldots = \sigma_n(\mb{P})= 1$\footnote{Any attempt to approximate the permutation matrix $\mb{P}$, e.g. using SVD or NMF, 
with a rank-$K$ doubly stochastic matrix $\mb{Q}$ will therefore
incur error 
$\lVert \mb{P}-\mb{Q} \rVert_F^2 \geq n-K$
by the Eckart-Minsky-Young theorem (see Theorem 7.4.9.1 in \citet{horn2012matrix}).
}.


Low-rank optimal transport (LR-OT) \citep{forrow19a, scetbon2020linear, lin2021making, scetbon2022lowrank, scetbon2023unbalanced, FRLC} aims to reveal low-rank structure by 
explicitly constraining the rank of the transport plan during optimization.
The low-rank constraint fundamentally alters the nature of the solution: by enforcing a rank $K \ll n$, LR-OT infers a variable spectrum $\sigma_{1}(\mb{P}) \geq \cdots \geq \sigma_{K}(\mb{P})$ that reflects the underlying latent structure \cite{scetbon2022lowrank}. This serves as a powerful regularizer,
producing estimators of Wasserstein distance that are more robust to outliers and sparse sampling, and achieves sharper statistical rates adaptive to the underlying rank \citep{forrow19a,lin2021making}. By forcing transport to factor through 
latent anchors, LR-OT simultaneously partitions and aligns
the source and target data \citep{forrow19a,
lin2021making}. In addition, the LR-OT framework strictly generalizes $K$-means 
clustering to the setting of multiple datasets \citep{scetbon2022lowrank}.



Despite the desirable features of low-rank OT, several practical and theoretical factors limit its adoption.
First, low-rank OT is a non-convex and NP-hard optimization problem, similar to NMF \citep{NMF}, and 
thus is sensitive to the choice of initialization \citep{scetbon2022lowrank} often producing different low-rank factors with different initializations. 
Second, current algorithms, which rely on local optimization through mirror-descent \citep{Scetbon2021LowRankSF, FRLC}
or Lloyd-type \citep{forrow19a, lin2021making} approaches, consist of a complex
optimization over three or more variables. 
Finally, although preliminary work has characterized theoretical properties 
of the low-rank OT problem \citep{forrow19a,scetbon2022lowrank}, existing
algorithms lack
provable guarantees beyond convergence to stationary points. This contrasts with tools for
$K$-means clustering that offer robust
$\mathcal{O}(\log K)$ \citep{kmpp} and $(1 +\epsilon)$-approximation \citep{1peps} in addition to statistical guarantees \citep{zhuang2023statistically}.

\textbf{Contributions.}
We show that low-rank OT reduces to a simple clustering 
problem on correspondences, which we call \emph{transport clustering}. Specifically, we reduce the low-rank OT problem from a co-clustering problem to a generalized $K$-means problem \citep{scetbon2022lowrank} via a 
\emph{transport registration} of the cost matrix. This registers the cost with 
the solution to a convex optimization problem: the optimal full-rank transport plan. Transport clustering eliminates the auxiliary variables used in existing low-rank solvers and converts the low-rank OT problem into a single clustering subroutine: one low-rank factor is given by solving the generalized $K$-means problem on the registered cost, and the second factor is automatically obtained from the first. We prove constant-factor guarantees for this reduction: for kernel costs the approximation factor is $(1+\gamma+\sqrt{2\gamma})$ and for negative-type metrics it is $(1+\gamma)$ where $\gamma\in[0,1]$ is the ratio of the optimal full-rank and rank $K$ OT
costs. Because the reduced problem is a (generalized) $K$-means instance, \ourmethod\ inherits the algorithmic stability and approximation guarantees of modern $K$-means 
and $K$-medians
solvers. In addition to its theoretical guarantees, transport clustering (\ourmethodshort) is a simple and practically effective algorithm 
for low-rank OT that empirically obtains lower transport cost than existing low-rank OT solvers.
\section{Background}
\label{section:background}
Suppose $X = \{ x_1, \ldots, x_n \}$ and $Y = \{ y_1,\ldots,y_m\}$ are datasets with $n$ and $m$ data points in a space $M$. Letting $\Delta_{n} = \{ \, \bm{p} \in \mathbb{R}^{n}_{+}  : \, \sum_{j} \bm{p}_{j} = 1  \}$ denote the probability simplex over $n$ elements, one may represent each dataset explicitly over the support with probability measures $\mu = \sum_{i=1}^{n} \bm{a}_{i} \delta_{x_{i}}$ and $\nu = \sum_{j=1}^{m} \bm{b}_{j} \delta_{y_{j}}$ for probability vectors $\bm{a} \in \Delta_{n}$ and $\bm{b} \in \Delta_{m}$. The
optimal transport framework \citep{peyre2019computational} aims to find the 
least-cost mapping between these datasets $\mu \mapsto \nu$ as quantified via a cost function $c : X \times Y \rightarrow\mathbb{R}$. 

\textbf{Optimal Transport.} The \emph{Monge formulation} \citep{monge1781memoire} of optimal transport finds a map $T^{\star}: M \to M$ of least-cost between the measures $\mu$ and $\nu$, $T^{\star}=\argmin_{T:\, 
T_{\sharp} \mu = 
\nu
} 
\mathbb{E}_{
\mu
}c(x, T(x))$. Here, $T_{\sharp}\mu$ denotes the pushforward measure of $\mu$ under $T$, defined by $T_{\sharp} \mu (B) := \mu(T^{-1}(B))$ for 
any measurable set $B \subset {M}$.
Define the set of couplings $\Gamma(\mu, \nu)$ to be all joint distributions $\gamma$ with marginals given by $\mu$ and $\nu$. The \emph{Kantorovich problem} \citep{kantorovich1942transfer} relaxes the Monge-problem by instead finding a coupling of least-cost $\gamma^{\star}$ between $\mu$ and $\nu$: $\gamma^{\star} \in\argmin_{
\gamma \in \Gamma(\mu, \nu)
} 
\mathbb{E}_{
\gamma
} \,c(x, y)\,$. This relaxation permits mass-splitting and guarantees the existence of a solution between any pair of measures $\mu$ and $\nu$ \citep{peyre2019computational}. 

In the discrete setting, the Kantorovich problem is equivalent to the linear optimization
\begin{equation}\label{eq:primal_ot}
    \min_{\mb{P} \in \Pi(\bm{a}, \bm{b})}\, \sum_{i=1}^{n}\sum_{j=1}^{m} \mb{P}_{ij} \, c(x_{i}, y_{j}) = \min_{\mb{P} \in \Pi_{\bm{a}, \bm{b} } }  
    \,\langle\mb{C}, \mb{P}\rangle_F,
\end{equation}
over the \emph{transportation polytope} $\Pi(\bm{a}, \bm{b}) \defeq \left\{ \mb{P} \in \mathbb{R}_+^{n \times m} : \mb{P} \mb{1}_m = \bm{a}, \mb{P}^\mathrm{T} \mb{1}_n = \bm{b} \right\}$ defined by marginals $\bm{a} \in \Delta_n$ and $\bm{b} \in \Delta_m$. $\left\langle \mb{A}, \mb{B} \right\rangle_F = \tr\mb{A}^{\top} \mb{B}$ denotes the Frobenius inner product and $[c(x_i, y_j)]=(\mb{C})_{ij}  \in \mathbb{R}^{n \times m}$ is the cost evaluated at all point pairs.

\textbf{Low-rank Optimal Transport.} 
Low-rank optimal transport (OT) constrains the non-negative rank of the transport plan $\mb{P}$ to be upper bounded by a specified constant $K$. This has computational \citep{Scetbon2021LowRankSF, scetbon2022lowrank, FRLC}, statistical \citep{forrow19a}, and interpretability 
benefits \citep{forrow19a, lin2021making, HMOT}, with the drawback that it results in a 
non-convex and NP-hard optimization problem.
For a matrix $\mb{M} \in \mathbb{R}_+^{n\times m}$, the \emph{nonnegative rank} \citep{cohen1993nonnegative} is
$\rank_+(\mb{M}) \defeq \min\{K: \mb{M} = \sum_{i=1}^K \bm{q}_i \bm{r}_i^{\top}, \bm{q}_i, \bm{r}_i\geq 0\}$,
or the minimum number of nonnegative rank-one matrices which sum to $\mb{M}$. The \emph{low-rank Kantorovich problem} \citep{scetbon2022linear, scetbon2023unbalanced} is then
\vspace{-0.2em}\begin{equation}\label{eq:primal_low_rank_ot}
    \min_{\mb{P} \in \Pi(\bm{a},\bm{b})}  \left\langle \mb{C}, \mb{P} \right\rangle_F  \, : \, \rank_+{(\mb{P})} \leq K 
\end{equation}

\vspace{-1em}which constrains the (nonnegative) rank of the \emph{transport plan} $\mb{P}$ to be at most $K$. 
Following \citep{cohen1993nonnegative, Scetbon2021LowRankSF}, the low-rank Kantorovich problem \eqref{eq:primal_low_rank_ot} 
is equivalent to
\begin{equation}
    \label{eq:primal_low_rank_ot_2}
\min_{ \, \substack{
   \mb{Q} \in \Pi(\bm{a},\bm{g}),\mb{R} \in \Pi(\bm{b},\bm{g}) \\
   \bm{g} \in \Delta_K
}} \left\langle\mb{C}, \mb{Q}\diag(\bm{g}^{-1})\mb{R}^{\top}\right\rangle_F ,
\end{equation}
which explicitly parameterizes the low-rank plan $\mb{P}$ as the product of two rank $K$ transport plans $\mb{Q}$ and $\mb{R}$ with outer marginals $\mb{Q}\mb{1}_{K}=\bm{a},\, \mb{R}\mb{1}_{K} =\bm{b}$ and a shared inner marginal $\bm{g} = \mb{Q}^{\top}\mb{1}_{n} = \mb{R}^{\top}\mb{1}_{m}$.

\textbf{$K$-Means and Generalized $K$-Means.} 
Given a dataset $X$, the $K$-means problem finds a partition $\pi = \{ \mathcal{C}_{k}\}_{k=1}^{K}$ of $X$ with $K$ clusters 
and means $\bm{\mu}_{1},\ldots,\bm{\mu}_K$
such that the total distance of each point to its nearest mean is
minimized. Letting the $k$-th cluster mean be $\bm{\mu}_k = \frac{1}{\lvert \mathcal{C}_k\rvert}\sum_{i\in\mathcal{C}_k}x_i$, the $K$-means problem minimizes the distortion
\begin{equation}\label{eq:Kmeans}
\min_{
    \pi} \sum_{\mathcal{C}_k \in \pi} \sum_{i\in \mathcal{C}_{k}} \lVert x_{i} - \bm{\mu}_{k} \rVert_{2}^{2}.
\end{equation}
Using that $|\mathcal{C}_{k}| \, \sum_{i \in \mathcal{C}_{k}}\lVert x_{i} - \bm{\mu}_{k} \rVert_{2}^{2} = \frac{1}{2}\sum_{i,j \in \mathcal{C}_{k}} \lVert x_{i} - x_{j} \rVert_{2}^{2}$ yields an  equivalent \emph{mean-free} formulation of \eqref{eq:Kmeans} in terms of pairwise distances:
\begin{equation}\label{eq:pairwise_KMeans}
    \min_{
    \substack{
   \pi }} \sum_{\mathcal{C}_k \in \pi} \frac{1}{|\mathcal{C}_{k}|} \sum_{i, j \in \mathcal{C}_{k}} \frac{1}{2}\lVert x_{i} - x_{j} \rVert_{2}^{2}.
\end{equation}
Define the assignment matrix $n\mb{Q} \in \{0, 1 \}^{n \times K}$ by $\mathbf{Q}_{ik} = 
\frac{1}{n}$ if $i\in \mathcal{C}_k$ and $0$ otherwise. Then,  \eqref{eq:pairwise_KMeans} is equivalently expressed (up to the constant factor $n$) as a sum over all assignment variables with cluster proportions given by $|\mathcal{C}_{k}|\,/\,n = \sum_{i} \mb{Q}_{ik}$:
\begin{align}
\label{eq:Kmeans_Q}
    &\langle \mb{C}_{\ell_2^2}, \mb{Q}\,\mathrm{diag}(1/\mb{Q}^\top \bm{1}_n)\,\mb{Q}^{\top} \rangle_{F} \\
    &=\sum_{i=1}^{n}\sum_{j=1}^{n}\sum_{\mathcal{C}_k \in \pi}   \frac{1}{2}\lVert x_{i} - x_{j} \rVert_{2}^{2}  \,\,\mb{Q}_{ik} \frac{n}{|\mathcal{C}_{k}|} \mb{Q}_{jk}
\end{align}
where $(\mb{C}_{\ell_2^2})_{ij}=(1/2)\lVert x_i - x_j\rVert^2_2$. In the preceding assignment form, \eqref{eq:Kmeans_Q} is the cost
of the rank $K$ transport plan $\mathbf{P}=\mb{Q}\diag(\bm{g}^{-1})\mb{R}^\top$
where $\mb{R} = \mb{Q}$ and $\bm{g}=\mb{Q}^{\top}\mb{1}_n$.
Following this observation, \cite{scetbon2022lowrank}
introduced \emph{generalized $K$-means} as the extension of \eqref{eq:Kmeans} to arbitrary cost functions $c(x_{i}, x_{j})$ by replacing $\mb{C}^{\ell_{2}^{2}}$ in the mean-free formulation \eqref{eq:Kmeans_Q} with a general cost $\mb{C}$. Let $\sqcup$ denote the disjoint set
union operator. In \emph{partition form}
this yields the following problem.
\begin{definition}
    Given a cost matrix $\,\mb{C}_{ij} = c(x_{i}, x_{j}) \in \mathbb{R}^{n\times n}$, the
    \emph{generalized $K$-means} problem is to minimize over partitions $\pi = \{\mathcal{C}_{k} \}_{k=1}^{K}$ the distortion:
    \vspace{-0.5em}\begin{equation}\label{def:GenKMeans}
    \min_{
    \pi} \left\{\sum_{k=1}^{K} \frac{1}{|\mathcal{C}_{k}|} \sum_{i, j \in \mathcal{C}_{k}} c( x_{i} , x_{j} ): \bigsqcup_{k=1}^{K}\,\mathcal{C}_{k} = [n]\right\}
    \end{equation}
\end{definition}
\vspace{-0.5em}
     Define the set of \emph{hard transport plans} to be 
$\Pi_{\bullet}(\bm{a},\bm{b}) \defeq \{ \mb{P} \in \mathbb{R}_+^{n \times K} : \mb{P} \mb{1}_K = \bm{a}, \mb{P}^\top \mb{1}_n = \bm{b}, \lVert \mb{P} \rVert_0 = n\}$, where $\lVert \mb{P}\rVert_0 = \lvert\{ (i,j): \mb{P}_{ij} >0 \}\rvert$. Then, \eqref{def:GenKMeans} is equivalent 
to the optimization
problem
\begin{equation}\label{eq:GenKMeans}
    \min_{\substack{\mb{Q} \in \Pi_{\bullet}(\bm{u}_{n},\,\cdot)}}  \,\, \langle\mb{C}, \mb{Q}\,\mathrm{diag}(1/\mb{Q}^\top \bm{1}_n)\,\mb{Q}^{\top} \rangle_{F},
\end{equation}
where $\bm{u}_n = \frac{1}{n}\mb{1}_{n}$ is the uniform marginal. Interestingly, when $X = Y$, $\bm{a} = \bm{b}=\bm{u}_n$, and $\mb{C} = \mb{C}_{\ell_2^2}$, the optimal solution $(\mb{Q}, \mb{R}, \bm{g})$ of \eqref{eq:primal_low_rank_ot_2} always has $\mb{Q} = \mb{R} \in \Pi_{\bullet}(\bm{u}_n,\bm{g})$ following Proposition 9 in 
\cite{scetbon2022lowrank}. Consequently, $K$-means strictly 
reduces to low-rank OT (see also Corollary 3 in \cite{scetbon2022lowrank}),
proving that the low-rank OT problem \eqref{eq:primal_low_rank_ot_2} is 
NP-hard.

\begin{figure*}
    \centering
    \includegraphics[width=1.0\linewidth]{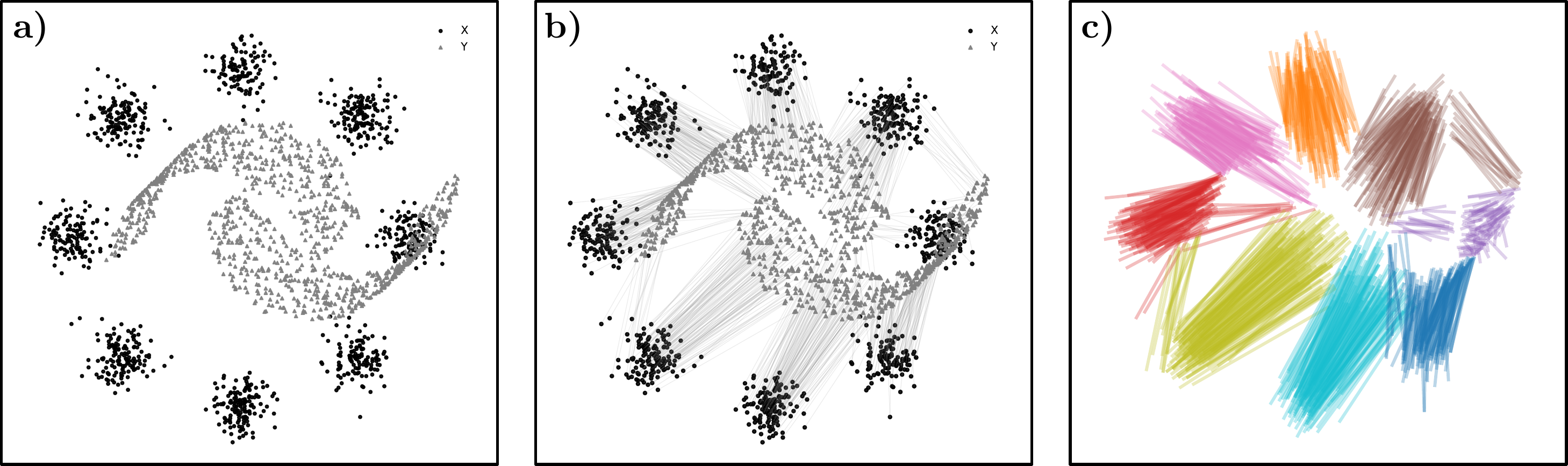}
    \caption{\label{fig:illustration} \ourmethodshort on \textbf{(a)} a synthetic 2-Moons ($X$) and 
    8-Gaussians ($Y$) dataset ($n=m = 1024$) from \cite{tong2023improving} with the \textbf{(b)} Monge map alignment of $X$ and $Y=\sigma(X)$ using \cite{halmos2025hierarchical}. \ourmethodshort reduces low-rank OT (co-clustering) to \textbf{(c)} clustering a single set of Monge registered correspondences using
    generalized $K$-means.}
\end{figure*}

\section{Transport Clustering}

We introduce a hard assignment variant of the
low-rank OT problem and argue that it naturally
generalizes $K$-means to co-clustering two datasets. 
We introduce \emph{Monge registration} of the cost matrix as a tool
for reducing low-rank OT to generalized $K$-means and discuss approximation guarantees for
the reduction. Finally, we introduce \emph{Kantorovich registration} as the analogue of Monge registration for the soft assignment low-rank OT problem 
\eqref{eq:primal_low_rank_ot_2}. As this reduction converts low-rank OT from a co-clustering problem to a clustering problem, we refer to the procedure as \emph{\ourmethod}.

Clustering methods such as $K$-Means 
output hard 
assignments of points to clusters to represent a partition. The
extension of \eqref{eq:primal_low_rank_ot_2} to co-clustering with a bipartition then 
requires the 
low rank factors to represent hard co-cluster assignments. Specifically, we require that the transport plans $\mb{Q}$ and $\mb{R}$ in 
\eqref{eq:primal_low_rank_ot_2} lie in the set of hard transport plans
$\Pi_{\bullet}(\bm{a},\bm{b})$\footnote{A well-known result on network flows (see \cite{peyre2019computational}) states that vertices of the
(soft) transportation polytope $\Pi(\bm{a}, \bm{b})$ have $\leq n + K - 1$ non-zero entries, implying that the solutions of the (soft) low-rank OT problem \eqref{eq:primal_low_rank_ot_2} are nearly hard transport plans.} instead of $\Pi(\bm{a},\bm{b})$, mirroring
the assignment version of $K$-means in Section \ref{section:background}.
\begin{definition}
    Given a cost matrix $\mb{C}_{ij} = c(x_{i} , y_{j} ) \in \mathbb{R}^{n\times n}$, the assignment form of
    the (hard) low-rank optimal transport problem is to solve:
    \begin{align}
\min_{\substack{\mb{Q}, \mb{R} \in \Pi_{\bullet}(\bm{u}_{n},\,\bm{g}) \\
    \bm{g} \in \Delta_{K}
    }}  \,\, \langle\mb{C}, \mb{Q}\,\mathrm{diag}(1/\bm{g})\,\mb{R}^{\top} \rangle_{F}.\label{eq:hardLOT}
\end{align}
\end{definition}
There is an equivalent partition-form of ~\eqref{eq:hardLOT} which parallels the partition form of $K$-means in 
\cite{zhuang2023statistically}. In particular, one finds a pair of partitions $\pi_{X}=\{\mathcal{C}_{X,k}\}, \,\pi_{Y}=\{\mathcal{C}_{Y,k}\}$ minimizing the distortion:
{\begin{align}
&\min_{\pi_X\times\pi_Y}\biggl\{  \sum_{k=1}^{K} \frac{1}{|\mathcal{C}_{k}|} \sum_{i \in \mathcal{C}_{X,k}}\sum_{j \in \mathcal{C}_{Y,k}} c( x_{i} , y_{j} ) \biggr\} \\ 
&\textrm{s.t.}\,\,\,
\  |\mathcal{C}_{X,k}|=|\mathcal{C}_{Y,k}|, \,\,\bigsqcup_{k=1}^{K}\,\mathcal{C}_{X,k} = \,\bigsqcup_{k=1}^{K}\,\mathcal{C}_{Y,k} = [n].\label{eq:hardLOT_partition} 
\end{align}}

\vspace{-0.5em}
This form \eqref{eq:hardLOT_partition} solves for a bipartition, implying~\eqref{eq:hardLOT} is a form of co-clustering (Appendix \ref{appendix:theory_1}). 
When the sets $X$ and $Y$ are distinct, \eqref{eq:hardLOT_partition} 
provides a natural generalization of $K$-means for co-clustering: (i) 
there are $K$ co-clusters, (ii) each dataset receives a distinct partition 
$\pi_{X}, \,\pi_{Y}$, (iii) co-cluster 
sizes are matched $|\mathcal{C}_{X,k}|=|\mathcal{C}_{Y,k}|$, and (iv) 
one minimizes a distortion $c( x_{i} , y_{j})$. When $X=Y$ and 
$\mathcal{C}_{X,k}=\mathcal{C}_{Y,k}$, observe that this exactly 
recovers the generalized $K$-means problem. As an example, the decomposition of \eqref{eq:hardLOT_partition} for the squared Euclidean cost can be written as
{\vspace{-0.6em}\begin{align*}
\min_{ \pi_X\times\pi_Y }\,\,
  \sum_{k=1}^{K}  
\sum_{i\in \mathcal{C}_{X,k}} &\lVert x_{i} - \bm{\mu}_{k}^{X} \rVert_{2}^{2} +
 \sum_{k=1}^{K}\sum_{j\in \mathcal{C}_{Y,k}} \lVert y_{j} - 
\bm{\mu}_{k}^{Y}
\rVert_{2}^{2}  \\
&+\sum_{k=1}^{K}|\mathcal{C}_{k}| \lVert \bm{\mu}_{k}^{X} - \bm{\mu}_{k}^{Y} \rVert_{2}^{2}
\end{align*}}\noindent

\vspace{-1.1em}where $\bm{\mu}_k^Z = \frac{1}{\lvert \mathcal{C}_{Z,k}\rvert}\sum_{i\in\mathcal{C}_{Z,k}}x_i$ for $Z = X, Y$ (see Remark~\ref{remark:LOT_as_2KMeans}). While \eqref{eq:Kmeans} finds a single centroid per cluster, this is a natural generalization for optimizing two: one minimizes two $K$-means distortions of $\bm{\mu}_{k}^{X}, \bm{\mu}_{k}^{Y}$ on $X$ and $Y$, and an additional distortion between the cluster centers $\bm{\mu}_{k}^{X}, \bm{\mu}_{k}^{Y}$. When $X = Y$ and $\bm{\mu}_{k}^{X}= \bm{\mu}_{k}^{Y}$, this collapses to $K$-means. 

To solve the low-rank OT problem \eqref{eq:hardLOT}-\eqref{eq:hardLOT_partition}, we propose a reparameterization trick
motivated by the assignment form \eqref{eq:hardLOT}.
Specifically, as the matrices $\mb{Q}, \mb{R} \in \Pi_{\bullet}(\bm{u}_{n},\,\bm{g})$ are hard assignment
matrices with matching column and row sums there exists a permutation
of the rows of $\mb{R}$ (resp. $\mb{Q}$) that takes $\mb{R}$ to $\mb{Q}$.
Formally, for any feasible $\mb{Q}, \mb{R}$, there exists a permutation matrix $\mb{P}_{\sigma} \in \mathcal{P}_n$ with $\mb{R} = \mb{P}_{\sigma}^\top \mb{Q}$. With this reparameterization, we reformulate \eqref{eq:hardLOT} as follows,
\begin{align}
    &\min_{\substack{\mb{Q},\mb{R} \in \Pi_{\bullet}(\bm{u}_{n},\,\bm{g}),\\\,\bm{g} \in \Delta_{K}}}  \,\, \langle \mb{C}, \mb{Q}\,\mathrm{diag}(\bm{g}^{-1})\,\mb{R}^{\top}\rangle_{F}\\ 
    &= \min_{\substack{\mb{Q} \in \Pi_{\bullet}(\bm{u}_{n},\,\bm{g}),\\\mb{P}_{\sigma} \in \mathcal{P}_n,\\\,\bm{g} \in \Delta_{K}}}  \,\, \langle\mb{C}, \mb{Q}\,\mathrm{diag}(\bm{g}^{-1})\,(\mb{P}_{\sigma}^{\top}\mb{Q})^{\top} \rangle_{F},\nonumber\\
    &=\min_{\substack{\mb{Q} \in \Pi_{\bullet}(\bm{u}_{n},\,\cdot),\\\mb{P}_{\sigma} \in \mathcal{P}_n}}  \,\, \langle\mb{C}\mb{P}_{\sigma}^{\top}, \mb{Q}\,\mathrm{diag}(1/\mb{Q}^T\mb{1}_n)\,\mb{Q}^{\top} \rangle_{F},
    \label{eq:RotatedLR}
\end{align}
where $\mathcal{P}_n$ is the set of permutation matrices.
This reformulation of \eqref{eq:hardLOT} might appear to offer little: 
the optimization remains over a difficult and non-convex pair of variables $(\mb{Q},\mb{P}_{\sigma})$. However, 
the reformulation \eqref{eq:RotatedLR} offers a new perspective:
for $\mb{P}_{\sigma}$ fixed, \eqref{eq:RotatedLR}
is a symmetric optimization problem over a single assignment matrix $\mb{Q}$
with respect to the \emph{registered}
cost matrix $\mb{C}\mb{P}_{\sigma}^{\top}$.

In fact, when $\mb{P}_{\sigma}$ is fixed in \eqref{eq:RotatedLR} the
result is exactly the generalized $K$-means problem \eqref{def:GenKMeans} discussed in
Section \ref{section:background}. 
Unfortunately, however, the reduction from \eqref{eq:hardLOT} to \eqref{def:GenKMeans} requires a priori knowledge of the optimal choice for this unknown permutation $\mb{P}_\sigma$. This leads
us to ask:

\vspace{-1em}\begin{quote}
    \emph{Is there an efficiently computable choice of  permutation matrix $\,\mb{P}_\sigma$ that accurately
    reduces low-rank optimal transport to the generalized $K$-means
    problem?}
\end{quote}

\vspace{-1em}We answer this question in the \emph{affirmative}.
Specifically, we show that taking the optimal Monge map $\mb{P}_{\sigma^*}$
as the choice of $\mb{P}_\sigma$ yields a constant-factor approximation 
algorithm (Algorithm \ref{alg:monge_KMeans}) for (hard) low-rank OT given an
algorithm for solving the generalized $K$-means problem (Section \ref{section:theory}). The resulting \emph{\ourmethod}
(\ourmethodshort) algorithm first finds a correspondence
between $X$ and $Y$ and then clusters the transport registered cost, effectively clustering on the correspondences (Figure \ref{fig:illustration}).

\begin{algorithm}[]
   \caption{Transport Clustering (\ourmethodshort)}
   \label{alg:monge_KMeans}
\begin{algorithmic}[1]
   \STATE \textbf{Input:} Cost matrix $\mb{C}$ and rank $K$.
   \STATE \textbf{Step 1 (Transport):} Compute the optimal full-rank plan $\mb{P}_{\sigma^{\star}}$:
   \begin{equation*}
   \mb{P}_{\sigma^{\star}} \leftarrow n \cdot \argmin_{\mb{P}\in \Pi(\bm{u}_{n},\bm{u}_{n})} \langle \mb{C}, \mb{P} \rangle_{F}
   \end{equation*}
   \STATE \textbf{Step 2 (Clustering):} Register the cost $\tilde{\mb{C}} \leftarrow \mb{C}\mb{P}_{\sigma^{\star}}^\top$ and solve generalized $K$-means for $\mb{Q}$:
   \begin{equation*}
   \mb{Q} \leftarrow \argmin_{\mb{Q} \in \Pi_{\bullet}(\bm{u}_{n},\cdot)} \,\, \langle\tilde{\mb{C}}, \mb{Q}\,\mathrm{diag}(1/\mb{Q}^\top\bm{1}_n)\,\mb{Q}^{\top}\rangle_{F}
   \end{equation*}
   \STATE \textbf{Output:} The low-rank factors $(\mb{Q},  \mb{P}_{\sigma^{\star}}^\top\mb{Q})$.
\end{algorithmic}
\end{algorithm}

Using standard algorithms for the Monge problem such as the Hungarian algorithm 
\citep{Kuhn1955Hungarian} or the Sinkhorn algorithm \citep{sinkhorn}, ones 
easily implements step 1 in polynomial time. For step 2, we propose
two algorithms for generalized 
$K$-means problem based upon (1) mirror descent and (2) semidefinite programming based 
algorithms for $K$-means 
\citep{peng2007approximating, fei2018hidden, zhuang2023statistically}. Given a $(1 + \epsilon)$-approximation algorithm $\mathcal{A}$ for $K$-means, 
an appropriate initialization for step 2 of 
Algorithm \ref{alg:monge_KMeans} maintains the constant factor
approximation guarantee with an 
additional $(1+\epsilon)$
factor. An analogous statement holds for metric costs where the $K$-means solver $\mathcal{A}$
is replaced with a $K$-medians solver, yielding polynomial-time constant-factor approximations for
low-rank OT with metric and kernel costs independent of
an algorithm for generalized $K$-means (Section \ref{section:theory}). 

We note that an analogous notion of 
\emph{Kantorovich registration} exists for the soft assignment variant of the low-rank OT problem \eqref{eq:primal_low_rank_ot_2} with arbitrary marginals $\bm{a},\bm{b}$ supported on $X$ and $Y$ with $n \neq m$.
In this setting, rather than register via the Monge permutation, one registers by the optimal Kantorovich plan $\mb{P}^*$ using either $\mb{Q} = \mb{P}^* \diag(1/\bm{b}) \mb{R}$ or $\mb{R} = (\mb{P}^*)^{\top} \diag(1/\bm{a}) \mb{Q}$. When solving with respect to $\mb{Q}$, this results in a (soft)
generalized $K$-means problem:
\[
\min_{\substack{\mb{Q} \in \Pi(\bm{a},\bm{\cdot})}} \,\, \left\langle \mb{C} \mb{P}^{*,\top} \diag(1/\bm{a}), \mb{Q} \diag(1/\mb{Q}^{\top}\bm{1}_n)\mb{Q}^{\top}    
    \right\rangle_{F}.
\]
To obtain $\mb{R}$, for the resultant $\mb{Q}^{\top}\mb{1}_{n} = \bm{g} \in \Delta_{K}$ one applies the conjugation $\mb{R} = \mb{P}^{*,\top} \diag(1/\bm{a}) \mb{Q}$, which ensures $\mb{R}\mb{1}_{K} = \mb{P}^{*,\top} \diag(1/\bm{a}) \mb{Q} \mb{1}_{K} = \bm{b}$ and $\mb{R}^{\top}\mb{1}_{n} = \mb{Q}^{\top}\diag(1/\bm{a})\mb{P}^*\mb{1}_{n} = \bm{g}$, so that $(\mb{Q}, \mb{R}, \bm{g})$ is feasible, and $\mb{Q}\diag(\bm{g}^{-1})\mb{R}^{\top} \in \Pi_{\bm{a},\bm{b}}$.

\section{Theoretical Results}
\label{section:theory}
\textbf{Approximation of low-rank optimal transport by generalized $K$-means.}
\label{subsec:theory1}
In this section, we justify the reduction
from the low-rank optimal transport 
problem \eqref{eq:primal_low_rank_ot_2} to the 
generalized $K$-means problem \eqref{def:GenKMeans} by proving
that solving the proxy problem \eqref{def:GenKMeans}
incurs at most a constant factor in cost. All proofs
are found in Appendix \ref{appendix:theory_1}.

In detail, we derive a $(2 + \gamma)$ approximation ratio for any 
cost $c(\cdot, \cdot)$ satisfying the triangle inequality and a $(1 + \gamma + \sqrt{2\gamma})$ approximation ratio for any cost induced by a kernel,
which includes the squared Euclidean cost. 
For metrics of negative type, we provide an improved approximation
ratio of $(1 + \gamma)$. 
Examples of negative type metrics include all $\ell_p$ metrics for $p \in [1, 2]$ and weighted linear transformations thereof (see Theorem 3.6 \cite{meckes2013positive}). Any metric embeddable in $\ell_p$, $p \in [1, 2]$, is also
of negative type. For example, tree metrics are exactly embeddable in 
$\ell_p$ while shortest path metrics are approximately embeddable in $\ell_p$
with small distortion \citep{abraham2005metric}. 


To state our results, we write that a cost matrix $\mb{C}$ is
\emph{induced} by a cost $c(\cdot, \cdot)$ if there exists points ${X}=\{x_1\ldots,x_n\}$ and ${Y} = \{y_1, \ldots, y_n\}$ such that
$\mb{C}_{ij}=c(x_i,y_j)$. A cost $c(\cdot,\cdot)$ is a \emph{kernel cost}
if $c(x, y) = \lVert\phi(x)-\phi(y)\lVert^2_2$ for some feature-map $\phi: X \to \mathbb{R}^d$. 
A cost function 
$c(\cdot, \cdot)$ is \emph{conditionally negative semidefinite} if $\sum_{i=1}^n\sum_{j=1}^n\alpha_i\alpha_j \,c(x_i,x_j) \leq 0$ for all
$x_1, \ldots, x_n$ and $\alpha_1, \ldots, \alpha_n$ such that $\sum_{i=1}^n\alpha_i = 0$. Equivalently, this requires all cost matrices $\mb{C}$ induced by $c(\cdot, \cdot)$ to be negative semidefinite $\mb{C} \preceq 0$ over $\mb{1}_n^\perp = \{ \xi \in \mathbb{R}^{n}: \langle \xi , \mb{1}_{n} \rangle =0\}$.
A cost function $c(\cdot,\cdot)$ is said to be of \emph{negative type} if it is a
metric and conditionally negative semidefinite.

\begin{theorem}
\label{thm:reduction_to_kcut}
    Let $\mb{C} \in \mathbb{R}^{n \times n}$ be a cost matrix either induced by i) a metric of negative type, ii) a kernel
    cost, or iii) a cost satisfying the triangle inequality. 
    If $\mb{P}_{\sigma^{\star}}$ denotes
    the full-rank optimal transport plan for $\mb{C}$ and $\tilde{\mb{C}} = \mb{C} \mb{P}_{\sigma^{\star}}^\top$ is the Monge registered cost, then
    \begin{align*} &\min_{\substack{\mb{Q} \in \Pi_{\bullet}(\bm{u}_{n},\,\cdot)}}  \,\, \langle \tilde{\mb{C}},  \mb{Q}\,\mathrm{diag}(1/\mb{Q}^{\top}\mb{1}_n)\,\mb{Q}^{\top} \rangle_{F} \,\, \\
    & \leq (1+\gamma) \,\, \cdot \min_{\substack{\mb{Q},\mb{R} \in \Pi_{\bullet}(\bm{u}_{n},\,\bm{g}),\\\bm{g} \in \Delta_{K}}}  \,\, \langle \mb{C}, \mb{Q}\,\mathrm{diag}(\bm{g}^{-1})\,
        \mb{R}^{\top} \rangle_{F} ,\,\, \tag*{(Metrics of Negative Type)} \\
        &\leq \,\, (1+\gamma + \sqrt{2\gamma}) \,\, \cdot \min_{\substack{\mb{Q},\mb{R} \in \Pi_{\bullet}(\bm{u}_{n},\,\bm{g}),\\\bm{g} \in \Delta_{K}}}  \,\, \langle \mb{C}, \mb{Q}\,\mathrm{diag}(\bm{g}^{-1})\,
        \mb{R}^{\top} \rangle_{F}, \,\, \tag*{(Kernel Costs)} \\
        &\leq \,\, (1+\gamma + \rho) \,\, \cdot \min_{\substack{\mb{Q},\mb{R} \in \Pi_{\bullet}(\bm{u}_{n},\,\bm{g}),\\\bm{g} \in \Delta_{K}}}  \,\, \langle \mb{C}, \mb{Q}\,\mathrm{diag}(\bm{g}^{-1})\,
        \mb{R}^{\top} \rangle_{F}, \,\, \tag*{(General Metrics)} 
    \end{align*}
    where $\gamma \in [0,1]$ is the ratio of the  cost of the optimal rank $n$ and $K$ 
    solutions and $\rho \in [0,1]$ is the asymmetry coefficient of the cluster-variances (defined formally in Lemma~\ref{lemma:folklore_metric_bounds}). 
\end{theorem}


Note that the approximation ratio $\gamma \leq 1$  as the optimal cost decreases monotonically 
with the rank. Consequently,
the  upper bound in Theorem  \ref{thm:reduction_to_kcut} is at worst a $2$-approximation for negative type metric costs and at worst a $(2+\sqrt{2)}$-approximation for kernel costs. Further, following the argument in \citep{scetbon2022lowrank}, 
$\gamma$ is typically much smaller than one, especially
for small $r \ll n$. Finally, we
note that the statements of Theorem 
\ref{thm:reduction_to_kcut} holds even
when $\bm{g}$ is held fixed in both the upper and lower bounds. This follows
from analyzing the proof of the theorem.

Next, we show that the derived approximation ratios are essentially tight
and cannot be further improved without additional assumptions. Specifically,
we show that when $\bm{g}$ is fixed, the upper bound in Theorem \ref{thm:reduction_to_kcut} is realized by explicit examples
in $\mathbb{R}^2$. We provide separate examples for the Euclidean and squared Euclidean distances (Appendix \ref{appendix:lowerbounds}). Formally, we have the following result.
\begin{proposition}
    \label{prop:lowerbound}
    For all $\epsilon > 0$, there exists an integer $n$ and $X,Y$ of size $n$ such that for the cost matrix $\mb{C} \in \mathbb{R}_{+}^{n \times n}$ induced on these points by either the Euclidean or squared Euclidean cost,
    \begin{align*}
    &\min_{\substack{\mb{Q} \in \Pi_{\bullet}(\bm{u}_{n},\,\bm{g})}}  \,\, \langle\tilde{\mb{C}}, \mb{Q}\,\mathrm{diag}(\bm{g}^{-1})\,\mb{Q}^{\top}\rangle_{F} \\
    &\geq (2-\epsilon) \tag*{(Euclidean Metric)}  \,\, \cdot \min_{\substack{\mb{Q},\mb{R} \in \Pi_{\bullet}(\bm{u}_{n},\,\bm{g})}}  \,\, \langle \mb{C},\mb{Q}\,\mathrm{diag}(\bm{g}^{-1})\,
        \mb{R}^{\top} \rangle_{F}, \\
        &\geq (3-\epsilon) \,\, \cdot \min_{\substack{\mb{Q},\mb{R} \in \Pi_{\bullet}(\bm{u}_{n},\,\bm{g})}}  \,\, \langle \mb{C},\mb{Q}\,\mathrm{diag}(\bm{g}^{-1})\,
        \mb{R}^{\top} \rangle_{F}, \tag*{(Squared-Euclidean Distance)}
    \end{align*}
    for some $\bm{g} \in \Delta_{K}$ and Monge registered cost $\tilde{\mb{C}} = \mb{C} \mb{P}_{\sigma^{\star}}^\top$.
\end{proposition}
The preceding lower
bounds rely on (1) \emph{unstable} arrangements of points,
where the Monge map changes dramatically upon an 
$\epsilon$-perturbation, and (2) a limit where the size of a the sets of the points ${X}, Y$ is taken to $\infty$ as $|{X}|\uparrow \infty ,\,\, |{Y}| \uparrow \infty$. In finite settings with stable Monge maps, the approximation ratios in 
Theorem \ref{thm:reduction_to_kcut} may be greatly improved.

\textbf{Transport Registered Initialization with $K$-means and
$K$-medians.} \label{subsec:theory2} In the preceding section, we derived constant factor approximation guarantees
by reducing low-rank OT to generalized $K$-means via Algorithm
\ref{alg:monge_KMeans}. However, $\Tilde{\mb{C}}$ does not necessarily express a matrix of intra-dataset distances, so that even for kernel costs and metrics one cannot directly solve generalized $K$-means using $K$-means or $K$-medians. In Theorem 2, we show that by solving $K$-means or $K$-medians clustering optimally on $X,\,Y$ separately to yield $\mb{Q}_{X},\mb{R}_{Y}$ and taking the \emph{minimum} of the Monge-registered solutions $(\mb{Q}_{X}, \mb{P}_{\sigma^{*}}^{\top}\mb{Q}_{X} )$ and $(\mb{P}_{\sigma^{*}}\mb{R}_{Y}, \mb{R}_{Y} )$ in cost $(\mb{Q}, \mb{R}) \to \langle \mb{C}, \mb{Q}\diag(1/\mb{Q}^{\top}\mb{1}_{n}) \mb{R}^{\top} \rangle$, the constant factor approximation guarantees are preserved. In other words, using the best initialization in generalized $K$-means between $\mb{Q}^{(0)} =\mb{Q}_{X}$ and $\mb{Q}^{(0)} =\mb{P}_{\sigma^{*}}\mb{R}_{Y}$ by solving $K$-means or $K$-medians already ensures a constant-factor approximation to low-rank OT on initialization, which only requires an algorithm for generalized $K$-means with a local descent guarantee to maintain the approximation.
 
Let $\mathcal{A}_1 ,  \mathcal{A}_2$ denote blackbox $(1 + \epsilon)$-approximation algorithms for $K$-means and $K$-medians. For example, such polynomial time approximation algorithms exist 
when the dimension is fixed for $K$-means \citep{1peps}, and $(1 + \epsilon)$ approximation algorithms exist for $K$-medians \citep{Kolliopoulos2007}.
Then, we have the following guarantee for Algorithm \ref{alg:monge_KMeans} when using
Algorithm \ref{alg:gen_kmeans_initialization} to implement step 2 of the procedure. 
\begin{theorem}\label{thrm:K_means_init_guarantee}
    Let $\mb{C}$ be a $n$-by-$n$ cost matrix either induced by i) a metric of negative type, ii) a kernel
    cost, or iii) a cost satisfying the triangle inequality. 
    Let $(\mb{Q}^*, \mb{R}^*)$ be the solution output by using Algorithm \ref{alg:gen_kmeans_initialization} for step (ii) of Algorithm \ref{alg:monge_KMeans} with oracles $\mathcal{A}_1$ and $\mathcal{A}_2$. Then, 
    \begin{align*} &(1+\epsilon)^{-1}\cdot\langle \mb{C},  \mb{Q}^*\,\mathrm{diag}(1/(\mb{Q}^*)^\top\mb{1}_n)\,(\mb{R}^*)^{\top} \rangle_{F} \,\, \\
    & \leq (2+2\gamma) \,\, \cdot \min_{\substack{\mb{Q},\mb{R} \in \Pi_{\bullet}(\bm{u}_{n},\,\bm{g}),\\\bm{g} \in \Delta_{K}}}  \,\, \langle \mb{C}, \mb{Q}\,\mathrm{diag}(\bm{g}^{-1})\,
        \mb{R}^{\top} \rangle_{F} \,\, \tag*{(Metrics of Negative Type)} \\
        &\leq \,\, (1+\gamma + \sqrt{2\gamma}) \,\, \cdot \min_{\substack{\mb{Q},\mb{R} \in \Pi_{\bullet}(\bm{u}_{n},\,\bm{g}),\\\bm{g} \in \Delta_{K}}}  \,\, \langle \mb{C}, \mb{Q}\,\mathrm{diag}(\bm{g}^{-1})\,
        \mb{R}^{\top} \rangle_{F} \,\, \tag*{(Kernel Costs)} \\
        &\leq \,\, (2+2\gamma + 2\rho) \,\, \cdot \min_{\substack{\mb{Q},\mb{R} \in \Pi_{\bullet}(\bm{u}_{n},\,\bm{g}),\\\bm{g} \in \Delta_{K}}}  \,\, \langle \mb{C}, \mb{Q}\,\mathrm{diag}(\bm{g}^{-1})\,
        \mb{R}^{\top} \rangle_{F} \,\, \tag*{(General Metrics)} 
    \end{align*}
    where $\gamma, \rho \in [0,1]$ are defined as in Theorem~\ref{thm:reduction_to_kcut}.
\end{theorem}

\begin{figure*}
    \centering\includegraphics[width=1.0\linewidth]{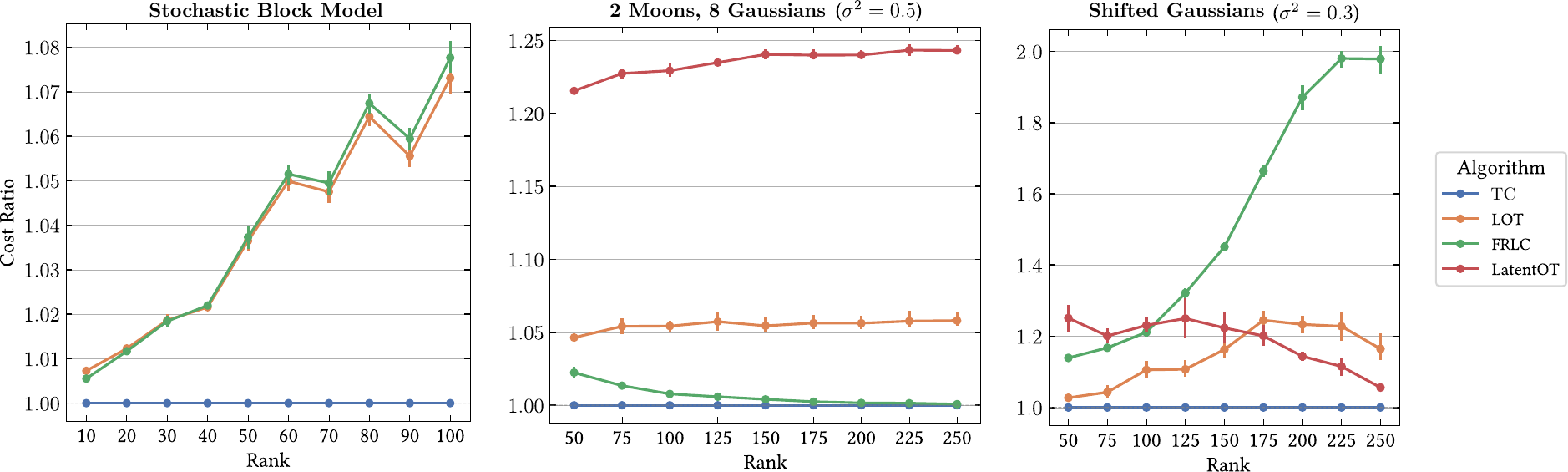}\caption{\label{fig:synthetic_results} The relative cost of the rank 
    $K \in \{50, 75, \ldots, 250\}$ transport plan inferred \LOT, \FRLC, and \LIN compared to the 
    cost of the transport plan inferred by \ourmethodshort across $315$ synthetic instances (lower is 
    better). Each dataset contains $n = m = 5000$ data points. LatentOT is excluded from the 
    stochastic block model evaluation as it
    takes as input a squared Euclidean cost matrix.}
\end{figure*}

\textbf{Generalized $K$-means Solver.}
\label{section:algo_genkmeans}
To solve the generalized $K$-means problem we propose (1) a mirror-descent algorithm called \texttt{GKMS} that solves generalized $K$-means locally, like Lloyd's algorithm, and (2) a semidefinite programming based approach. \texttt{GKMS} solves a sequence of diagonal, one-sided Sinkhorn projections \citep{sinkhorn} of a classical exponentiated gradient update. Suppose $(\gamma_k)_{k =1}^{\infty}$ is a positive sequence of step sizes for a mirror-descent with respect to the $\mathrm{KL}$ divergence. Then, the update for $\mb{Q}^{(k)}$ is given by:
\vspace{-0.5em}\begin{align}\label{alg:gen_KMeans}
\mb{Q}^{(k+1)}= P_{\bm{u}_{n}, \cdot}\left(
\mb{Q}^{(k)} 
\odot  \exp\left(-
\gamma_{k} \nabla_{\mb{Q}}\mathcal{F}\mid_{\mb{Q}^{(k)}}
\right)
\,\right),
\end{align}

\vspace{-1em}where $\mathcal{F}(\mb{Q})$ is the cost $ \langle \tilde{\mb{C}}, \mb{Q} \diag(1/\mb{Q}^{\top} \mb{1}_{n}) \mb{Q}^{\top} \rangle_{F}$ and $P_{\bm{u}_{n}, \cdot}(\mb{X}) = \diag(\bm{u}_{n}/ \mb{X
}\mb{1}_{K})\mb{X}$ is a Sinkhorn projection onto the set of positive matrices with marginal $\bm{u}_{n}$, $\Pi(\bm{u}_{n}, \cdot)=\{ \mb{X} \in \mathbb{R}_{+}^{n\times K} :\,\mb{X}\mb{1}_K = \bm{u}_{n} \}$. Observe that solving low-rank OT \eqref{eq:hardLOT} with the constant-factor guarantees of Theorem~\ref{thrm:K_means_init_guarantee} only requires \eqref{alg:gen_KMeans} to decrease the cost from the initialization of Algorithm~\ref{alg:gen_kmeans_initialization}. We show in Proposition~\ref{prop:descent_GKMS} that assuming a $\delta$ lower bound on $\mb{Q}^{\top}\mb{1}_{n}$ (similar to \cite{Scetbon2021LowRankSF, FRLC}), relative-smoothness to the entropy mirror-map $\psi$ holds $\lVert \nabla \mathcal{F}^{(k+1)}-\nabla\mathcal{F}^{(k)} \rVert_{F}
 \leq \beta\,\, \lVert \nabla \psi^{(k+1)}-\nabla \psi^{(k)} \rVert_{F}$ for $\beta= \mathrm{poly}(n, \lVert \mb{C} \rVert_{F}, \delta)$. By the descent lemma \citep{Lu2018}, this implies that Theorem~\ref{thrm:K_means_init_guarantee} provides upper bounds on the quality of the final solution of \texttt{GKMS}. See Appendix~\ref{sec:GKMS} for more details on the \texttt{GKMS} algorithm and Appendix~\ref{subsec:alg_sdp} for a semidefinite programming approach

\textbf{Complexity Analysis.}\label{sec:Complexity} The time and space complexity of Algorithm~\ref{alg:monge_KMeans} depends on the complexity of optimal transport and generalized K-means. Procedures such as \cite{Agarwal2024} and \cite{halmos2025hierarchical} yield approximate OT solutions with $\Tilde{\mathcal{O}}(n)$ time and $\mathcal{O}(n)$ space complexity for constant dimension $d$. \texttt{GKMS} requires $\mathcal{O}(ndr)$ iteration complexity if the cost is factorized $\mb{C}=\mb{U}_{d}\mb{V}_{d}^{\top}$ and $\mathcal{O}(nr)$ space to store $\mb{Q}$.  In addition, recent SDP approaches for $K$-means using the Burer-Monteiro factorization \cite{zhuang2023statistically} likewise provide linear time and space complexity for generalized $K$-means. 

\section{Numerical Experiments}


\textbf{Synthetic Validation.}

We evaluated \ourmethod\ against \LOT \citep{Scetbon2021LowRankSF}, \FRLC \citep{FRLC}, and \LIN \citep{lin2021making} on three synthetic benchmarks ($n=m=5000$): (1) 2-Moons to 8-Gaussians \citep{tong2023improving}, (2) shifted Gaussians (SG), and (3) the stochastic block model (SBM). We varied noise levels $\sigma^2$ and ranks $K \in [10,250]$ across five random seeds. See Appendix \ref{appendix:synthetic} for full simulation details.

To evaluate the low-rank OT methods, we computed the relative cost of the low-rank OT plans
output by existing methods compared to the cost of the low-rank OT plan output by \ourmethodshort.
Across all synthetic datasets, \ourmethodshort was consistently the best performing method 
in terms of minimizing the low-rank OT cost (Figure \ref{fig:synthetic_results}). On the 2M-8G dataset, \ourmethodshort outperformed all methods in the highest noise setting
(Figure \ref{fig:synthetic_results}, \ref{fig:twomoons}) and was slightly
 ($\leq1\%$ difference) outperformed by \FRLC
in the low noise, high rank setting. On the SG dataset, \ourmethodshort was 
the top performing method and obtained an average relative improvement of $23\%$ compared
to the next best performing method \LOT (Figure \ref{fig:synthetic_results}, \ref{fig:shifted_gaussians}). On the SBM dataset, \ourmethodshort outperformed all methods and
obtained an average relative improvement of $4\%$ compared to the next best performing method
\LOT (Figure \ref{fig:synthetic_results}, \ref{fig:sbm_cost}).

To evaluate co-cluster recovery, we computed the ARI/AMI with reference to the ground truth
clusters when the rank $K$ matched the true number of clusters 
($K = 250$ for SG, $K = 100$ for SBM). On the
SG dataset, \ourmethodshort was the second best performing method (Figure \ref{fig:shifted_gaussians})
with a slightly worse average ARI/AMI than \LIN (\ourmethodshort 0.97/0.99; \LOT 0.94/0.98; \FRLC
0.60/0.88; \LIN: 1.00/1.00). On the SBM dataset,
\ourmethodshort was the best performing method (Figure \ref{fig:sbm_cost}) and obtained the highest average ARI/AMI (\ourmethodshort 0.09/0.20; \LOT 0.02/0.02; \FRLC 0.02/0.01). 

\begin{table}[]
\centering
\caption{Comparison of low-rank OT methods across three datasets: CIFAR-10 ($n=60{,}000$), smallest mouse embryo split ($n=18{,}819$), and largest mouse embryo split ($n=131{,}040$).}
\label{tab:lr-ot-combined}
\resizebox{\linewidth}{!}{%
\begin{tabular}{llc cccccc}
\toprule
Dataset & Method & Rank & OT Cost $\downarrow$ & AMI (A/B) $\uparrow$ & ARI (A/B) $\uparrow$ & CTA $\uparrow$ \\
\midrule
\multirow{3}{*}{\shortstack{CIFAR-10 \\ ($60{,}000$)}} 
    & \ourmethodshort & 10 & \textbf{231.200} & \textbf{0.478} / \textbf{0.476} & \textbf{0.358} / \textbf{0.356} & \textbf{0.412} \\
    & \FRLC       & 10 & 235.950 & 0.411 / 0.407 & 0.281 / 0.277 & 0.351 \\
  & \LOT        & 10 & 234.733 & 0.430 / 0.427 & 0.306 / 0.303 & 0.358 \\
\midrule
\multirow{3}{*}{\shortstack{Mouse embryo \\ E8.5 $\to$ E8.75 \\ ($18{,}819$)}} 
  & \ourmethodshort & 43 & \textbf{0.506} & \textbf{0.639} / \textbf{0.617} & \textbf{0.329} / \textbf{0.307} & \textbf{0.722} \\
  & \FRLC       & 43 & 0.553 & 0.556 / 0.531 & 0.217 / 0.199 & 0.525 \\
  & \LOT        & 43 & 0.520 & 0.605 / 0.592 & 0.283 / 0.272 & 0.611 \\
\midrule
\multirow{3}{*}{\shortstack{Mouse embryo \\ E9.5 $\to$ E9.75 \\ ($131{,}040$)}} 
  & \ourmethodshort & 80 & \textbf{0.389} & \textbf{0.554} / \textbf{0.551} & \textbf{0.172} / \textbf{0.169} & \textbf{0.564} \\
  & \FRLC       & 80 & 0.399 & 0.491 / 0.487 & 0.116 / 0.115 & 0.447 \\
  & \LOT        & 80 & --    & -- / -- & -- / -- & -- \\
\bottomrule
\end{tabular}}
\end{table}

\textbf{Co-Clustering on CIFAR10.} Following \cite{zhuang2023statistically} we applied low-rank OT methods to the CIFAR-10 dataset, which contains 60{,}000 images of size $32 \times 32 \times 3$ across 10 classes. We use a ResNet to embed the images to $d=512$ \citep{he2016deep} and apply a PCA to $d=50$, following the procedure of \cite{zhuang2023statistically}. We perform a stratified 50:50 split of the images into two datasets of 30{,}000 images with class-label distributions matched. We co-cluster these two datasets using the methods which scale to it: \ourmethodshort, \FRLC \citep{FRLC}, and \LOT \cite{Scetbon2021LowRankSF}. We set the rank $K=10$ to match the number of labels. For \ourmethodshort, we solve for $\mb{P}_{\sigma^{\star}}$ with \cite{halmos2025hierarchical} and solve generalized $K$-means with \texttt{GKMS}. On this 60k point alignment, \ourmethodshort\ attains the lowest OT cost of $231.20$ vs.\ \LOT ($234.73$) and \FRLC ($235.95$). To evaluate the co-clustering performance of \ourmethodshort\, we evaluate the AMI and ARI of the labels derived from the asymmetric factors against the ground-truth class label assignments (Table~\ref{tab:lr-ot-combined}). \ourmethodshort\ shows stronger agreement on both marginals (AMI/ARI: split~A $0.478/0.358$, split~B $0.476/0.356$) than \LOT ($0.430/0.306$, $0.427/0.303$) or \FRLC ($0.411/0.281$, $0.407/0.277$). To assess the accuracy of co-clustering across distinct domains, we computing the class-transfer accuracy (CTA): the fraction of mass aligned between ground-truth classes \emph{across datasets} over the total (for $\rho$ the class-class transport matrix, this is $\tr \rho \,/ \sum \rho_{k,k'}$).
\ourmethodshort\ attains a CTA of $0.412$, compared to \LOT ($0.358$) and \FRLC ($0.351$), indicating more accurate cross-domain label transfer. See Section~\ref{sec:cifar} for details.

\textbf{Large-Scale Single-Cell Transcriptomics.} 
Recent single-cell datasets have sequenced millions of nuclei from model organisms such as the mouse \citep{shendure2024, shendure2022} and zebrafish \citep{Liu2022-zz} across time to characterize cell-differentiation and stem-cell reprogramming. Optimal transport has emerged as the canonical tool for aligning single-cell datasets \citep{schiebinger2019optimal, zeira2022PASTE, PASTE2, destot, klein2023moscot}, and low-rank optimal transport has recently emerged as a tool to co-cluster or link cell-types across time, allowing one to infer a map of cell-type differentiation \citep{HMOT, klein2023moscot}. We benchmark the co-clustering and alignment performance of \ourmethodshort, \LOT, and \FRLC on a recent, massive-scale dataset of single-cell mouse embryogenesis \cite{shendure2024} measured across 45 timepoint bins with combinatorial indexing (sci-RNA-seq3). We align 7 time-points with $n=18819$-$131040$ cells (stages E8.5-E10.0)  for a total of 6 pairwise alignments (Table~\ref{tab:lr-ot-combined}, Supplementary Table~\ref{tab:sc-lr-ot}). We set the rank $K \in \{ 43, 53, 57, 67, 80, 77 \}$ to be the number of ground-truth cell-types. While \LOT runs up to E8.75-9.0 ($30240$ cells) and fails to compute an alignment past E9.0-E9.25 ($45360$ cells), we find \ourmethodshort\ and \FRLC scale to all pairs. Transport clustering yields lower OT cost, higher AMI, and higher ARI than both \LOT and \FRLC on all dataset pairs (Supplementary Table~\ref{tab:sc-lr-ot}). Notably, the co-clustering performance is also improved for all timepoints: as an example, on E8.5-8.75 \ourmethodshort\ achieves a CTA of $0.722$ and correctly maps the majority of mass between recurring cell-types across the datasets, compared to \LOT ($0.611$) and \FRLC ($0.525$). See Section~\ref{sec:single_cell} for further details.

\textbf{Estimation of Wasserstein Distances.} Lastly, we demonstrate the efficacy of \ourmethod\ as an estimator of Wasserstein distances. For two measures $\hat \mu_{n}, \hat \nu_{n}$ given by empirical $n$-sample estimates of $\mu, \nu$, \cite{forrow19a} introduced a robust estimator for Wasserstein distances derived from a low-rank factored coupling:
\vspace{-1em}\begin{equation}\label{eq:Forrow_estimator}
\hat W_{2}^{2}( \hat \mu_{n}, \hat \nu_{n}) = \sum_{j=1}^{K} \lambda_{j} \bigl\lVert \bm\mu_{j}^{(1)} - \bm \mu_{j}^{(2)}
\bigr\rVert_{2}^{2}
\end{equation}

\vspace{-1.5em}where $\lambda \in \Delta_{K}$ ($\bm{g}$) represents the barycentric measure, and $\bm \mu_{j}^{(1)},\, \bm \mu_{j}^{(2)}$ are the centroids of the first dataset $X$ ($\hat \mu_{n}$) and second dataset $Y$ ($\hat \nu_{n}$) induced by the couplings $\mb{Q}$ and $\mb{R}$. We evaluate this estimator on the fragmented hypercube benchmark of \cite{forrow19a} (Section 6.1), which admits a ground-truth distance of $W_{2}^{2}(\mu,\nu) = 8$.

\begin{figure}
    \centering
    \includegraphics[width=0.95\linewidth]{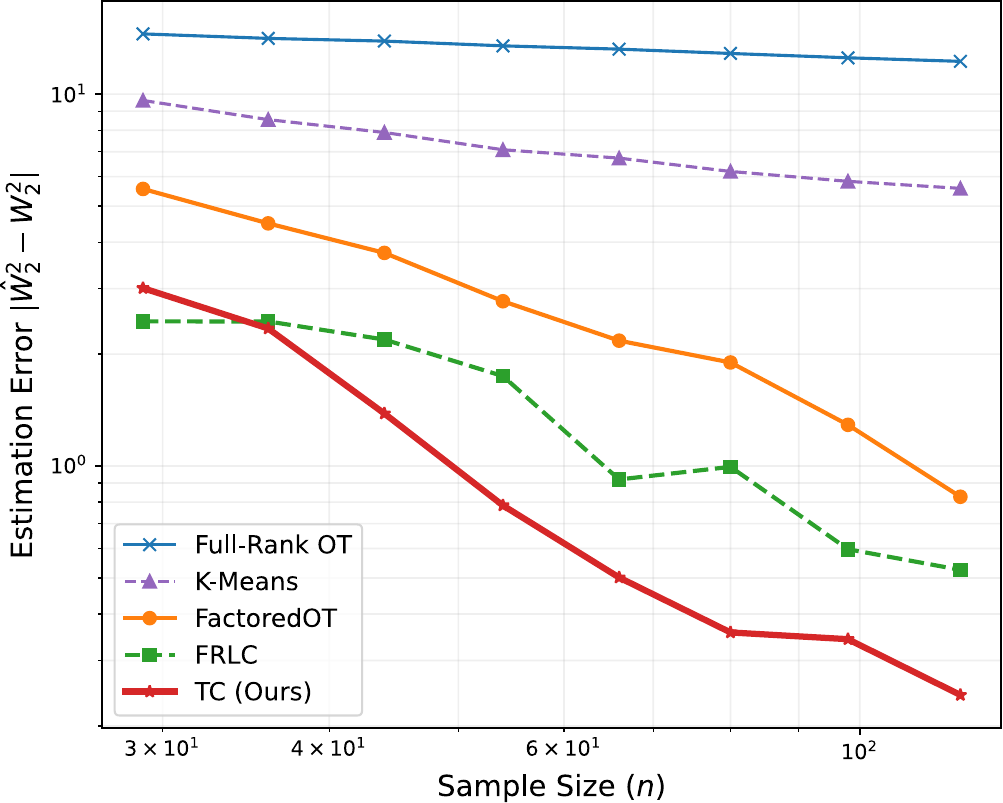}
    \caption{ 
    Estimation of squared Wasserstein distance on the fractured hypercube of \cite{forrow19a}. Convergence shown for fixed $d=30$, $K=10$, and averaged over $10$ runs.
    }\label{fig:w2_estimation}
\end{figure}

We compare the standard plug-in estimator of full-rank OT, 
known to suffer a curse of dimensionality with a $n^{-1/d}$ convergence rate, 
against the estimator \eqref{eq:Forrow_estimator} evaluated on low-rank couplings derived from: (1) \FC \cite{forrow19a}, (2) a naive baseline of independent $K$-means alignment, (3) the low-rank solver \FRLC, and (4) \ourmethodshort. In Figure~\ref{fig:w2_estimation} we find, that full-rank OT converges slowly to the true estimate, while low-rank OT algorithms exploit the intrinsic low-rank to achieve a faster rate. Notably, \ourmethodshort\ achieves the most accurate estimate of $W_{2}^{2}$ across the range of $n$ (Figure~\ref{fig:w2_estimation}, Table~\ref{tab:results_n}).

\textbf{Additional Experiments and Ablations.} We provide additional experiments and ablations to understand the impact of entropy regularization, initialization, and the asymmetry of datasets (Kantorovich registration) in Appendix~\ref{sec:additional_ablations}.



\section*{Impact Statement}

This paper presents work whose goal is to advance the field of Machine
Learning. There are many potential societal consequences of our work, none
which we feel must be specifically highlighted here.

\bibliography{references}
\bibliographystyle{icml2026}

\newpage
\appendix
\onecolumn

\section{Appendix}
\subsection{Approximation guarantees for low-rank optimal transport}
\label{appendix:theory_1}
To prove the approximation guarantees stated in Theorem \ref{thm:reduction_to_kcut}, we start by 
proving the equivalence between the partition formulation \eqref{eq:hardLOT_partition} and the assignment formulation \eqref{eq:hardLOT} of the (hard) low-rank OT problem.

Throughout, we assume that $\mb{C}$ is induced by a cost matrix $c(\cdot, \cdot)$ on 
${X} = \{x_1, \ldots, x_n\}$ and 
${Y} = \{y_1, \ldots, y_n\}$, matching the assumptions in Theorem \ref{thm:reduction_to_kcut}.
Denote the set of partitions of $\{1, \ldots, n\}$ as $\mathcal{P}_n$ and the set
of partitions of size $K$ as $\mathcal{P}_n^K$. 
Define the cost $\mathcal{J}(\mathcal{X}, \mathcal{Y})$ of two partitions $\mathcal{X}, \mathcal{Y} \in \mathcal{P}_n^K$
as
\begin{equation*}
    \mathcal{J}(\mathcal{X}, \mathcal{Y}) \defeq \sum_{k=1}^K\frac{1}{\lvert X_k \rvert}\sum_{i\in X_k}\sum_{j \in Y_k}c(x_i,y_j).
\end{equation*}
Then,
the assignment formulation \eqref{eq:hardLOT} is equivalent to the following partition
formulation over the datasets ${X}$ and ${Y}$:
\begin{equation}
\label{eq:comb_reformulation_monge_sep}
    \min_{\substack{\mathcal{X}=\{X_k\}_{k=1}^K\\\mathcal{Y}=\{Y_k\}_{k=1}^K}}\{\mathcal{J}(\mathcal{X}, \mathcal{Y}) : \lvert X_k\rvert=\lvert Y_k \rvert, \mathcal{X},\mathcal{Y}\in\mathcal{P}_n^K\}.
\end{equation}
The form \eqref{eq:comb_reformulation_monge_sep} is a concise form of 
\eqref{eq:hardLOT_partition} that is used in the proofs.
To see the equivalence, note that the cost of a solution 
$(\mb{Q}, \mb{R}, \bm{g})$ equals
\begin{align*}
    &\langle \mb{C}, \mb{Q}\diag(\bm{g}^{-1})\mb{R}^\top\rangle \\
    &= \sum_{i=1}^n\sum_{j=1}^n \mb{C}_{ij}[\mb{Q}\diag(\bm{g}^{-1})\mb{R}^\top]_{ij}= \sum_{i=1}^n\sum_{j=1}^n \mb{C}_{ij}\sum_{k=1}^K  \frac{\mb{Q}_{ik} \mb{R}_{jk}}{g_{k}}\\
    &= \sum_{k=1}^K \frac{1}{g_k} \sum_{i=1}^n\sum_{j=1}^n \mb{C}_{ij} \mb{Q}_{ik} \mb{R}_{jk}= \sum_{k=1}^K \frac{1}{g_k }\sum_{i\in X_k}\sum_{j \in Y_k}c(x_i, y_j)
\end{align*}
where $X_k = \{i : \mb{Q}_{ik} > 0\}, Y_k = \{i : \mb{R}_{ik} > 0\}$ are partitions
in $\mathcal{P}_n$ due to the constraints on $\mb{Q}$ and $\mb{R}$.
Rescaling the objective by $n$, we have that $ng_k^{-1} = \lvert X_k\rvert = \lvert Y_k \rvert$. Thus, every feasible solution $(\mb{Q}, \mb{R}, \bm{g})$ of \eqref{eq:hardLOT} induces a
solution of \eqref{eq:comb_reformulation_monge_sep} with equivalent cost, up to
a constant factor $n$. For the other direction, observe that any solution
of \eqref{eq:comb_reformulation_monge_sep} induces a solution of \eqref{eq:hardLOT}
with equal cost, again up to the factor of $n$, by following the equalities in the
opposite order.

When $\mb{R} = \mb{P}_\sigma^\top \mb{Q}$ for a permutation matrix $\mb{\sigma}$, 
it follows that $Y_k = \sigma(X_k)$. Thus,
fixing $\mb{P}_\sigma$ in the low-rank OT problem \eqref{eq:RotatedLR}
is equivalent to requiring that $Y_k = \sigma(X_k)$. 
Consequently, any approximation guarantee for the partition formulation
\eqref{eq:comb_reformulation_monge_sep} carries directly over to \eqref{eq:hardLOT}. 
Formally, we have the following statement.
\begin{lemma}
\label{lemma:comb_reformulation_equivalence}
    For any $\alpha > 0$ and permutation $\sigma$, the inequality 
\begin{equation*}
    \min_{\substack{\mb{Q} \in \Pi_{\bullet}(\bm{u}_{n},\,\bm{g}),\\\mb{P}_{\sigma} \in \mathcal{P}_n,\\\,\bm{g} \in \Delta_{K}}}  \,\, \langle \mb{Q}\,\mathrm{diag}(\bm{g}^{-1})\,\mb{Q}^{\top}, \tilde{\mb{C}} \rangle_{F} \leq \alpha\cdot \min_{\substack{\mb{Q},\mb{R} \in \Pi_{\bullet}(\bm{u}_{n},\,\bm{g}),\\\,\bm{g} \in \Delta_{K}}}  \,\, \langle \mb{Q}\,\mathrm{diag}(\bm{g}^{-1})\,\mb{R}^{\top}, \mb{C} \rangle_{F},
\end{equation*}
where $\tilde{\mb{C}} = \mb{C}\mb{P}^\top_\sigma$
holds if and only if
\begin{equation*}
   \min_{\mathcal{X} \in \mathcal{P}_n^K} \mathcal{J}(\mathcal{X}, \sigma(\mathcal{X}))\leq \alpha\cdot  \min_{\substack{\mathcal{X}=\{X_k\}_{k=1}^K\\\mathcal{Y}=\{Y_k\}_{k=1}^K}}\{\mathcal{J}(\mathcal{X}, \mathcal{Y}) : \lvert X_k\rvert=\lvert Y_k \rvert, \mathcal{X},\mathcal{Y}\in\mathcal{P}_n^K\}.
\end{equation*}
\end{lemma}
This states that in order to prove Theorem \ref{thm:reduction_to_kcut} it suffices
to prove the analogous inequality for the partition formulation
\eqref{eq:comb_reformulation_monge_sep}.

We now start the proof of Theorem \ref{thm:reduction_to_kcut}. 
In the case where $c(\cdot, \cdot)$
is a metric we prove both of the results together, as many of the components
are shared. The case where $c(\cdot, \cdot)$ is induced by a kernel
is handled separately, as the triangle inequality is lost, and
na\"ive application of the doubled triangle inequality results in a 
worse  guarantee.

We start by proving the following upper bound on twice
$\min_{\mathcal{X} \in \mathcal{P}_n^r} \mathcal{J}(\mathcal{X}, \sigma(\mathcal{X}))$, which holds for arbitrary metrics. 

\begin{lemma}
\label{lemma:metric_intra_ub}
    Let $\mathcal{X} =\{X_1,\ldots, X_K\}, \mathcal{Y} = \{Y_1, \ldots, Y_K\}$ be a feasible solution
    to the optimization problem \eqref{eq:comb_reformulation_monge_sep}
    and suppose that $c(\cdot, \cdot)$ is a metric. Then, for any
    permutation $\sigma$, 
    \begin{equation}
         \mathcal{J}_1+\mathcal{J}_2\leq 2M_\sigma+\sum_{k=1}^K\frac{1}{\lvert X_k \rvert}\sum_{i,j\in X_k} c(x_i, x_{j})+\sum_{k=1}^K\frac{1}{\lvert Y_k \rvert}\sum_{i,j\in Y_k} c(y_i, y_{j}),
    \end{equation}
    where $M_\sigma = \sum_{i=1}^nc(x_i,y_{\sigma(i)})$, $\mathcal{J}_1=\mathcal{J}(\sigma^{-1}(\mathcal{Y}),\mathcal{Y})$, and $\mathcal{J}_2 = \mathcal{J}(\mathcal{X},\sigma(\mathcal{X}))$.
\end{lemma}
\begin{proof}
    Consider the solution $\sigma^{-1}(\mathcal{Y}) =\{\sigma^{-1}
    (Y_k)\}_{k=1}^K, \mathcal{Y} = \{Y_k\}_{k=1}^K$ to the 
    optimization problem
    \eqref{eq:comb_reformulation_monge_sep}: this is a feasible solution
    as $\sigma^{-1}$ preserves the size of sets. Using the triangle inequality, we have that $c(x_i, y_j) \leq c(x_i, z_i) + c(z_i, y_{j})$, so that taking $z_i := y_{\sigma(i)}$ we can bound the cost of $\mathcal{J}_1$ as:
\begin{align}
\mathcal{J}_1 &=\sum_{k=1}^K\frac{1}{\lvert Y_k \rvert}\sum_{i\in \sigma^{-1}(Y_k)}\sum_{j \in Y_k}c( x_i, y_j) \nonumber
\\
&\leq \sum_{k=1}^K\frac{1}{\lvert Y_k \rvert}\sum_{i\in \sigma^{-1}(Y_k)}\sum_{j \in Y_k}[c( x_i, y_{\sigma(i)}) + c(y_{\sigma(i)}, y_{j})] \nonumber \\
&= \sum_{k=1}^K\frac{\lvert Y_k \rvert}{\lvert Y_k \rvert}
\sum_{i\in \sigma^{-1}(Y_k)}c( x_i, y_{\sigma(i)}) + \sum_{k=1}^K\frac{1}{\lvert Y_k \rvert}\sum_{i\in \sigma^{-1}(Y_k)}\sum_{j \in Y_k}c(y_{\sigma(i)}, y_{j}).
\label{eq:metric_upper_bound}
\end{align}
Using the fact that $\sigma^{-1}(\mathcal{Y})$ partitions $\{1, \ldots, n\}$ and
performing a change of variables with $\sigma$, the upper bound 
\eqref{eq:metric_upper_bound} becomes
\begin{align}\label{eq:J1_bound_Y}
    \mathcal{J}_1 \leq \sum_{i=1}^n c( x_i, y_{\sigma(i)}) + \sum_{k=1}^K\frac{1}{\lvert Y_k \rvert}\sum_{i, j \in Y_k}c(y_{i}, y_{j}).
\end{align}
We then apply a symmetric argument to the feasible solution $(\mathcal{X}, \sigma(\mathcal{X}))$ of \eqref{eq:comb_reformulation_monge_sep} by using the bound $c(x_i, y_j) \leq c(x_i, x_{\sigma^{-1}(j)}) + c(x_{\sigma^{-1}(j)}, y_j)$. This yields
\begin{align}\label{eq:J2_bound_Y}
    \mathcal{J}_2 \leq \sum_{i=1}^n c( x_i, y_{\sigma(i)}) + \sum_{k=1}^K\frac{1}{\lvert X_k \rvert}\sum_{i, j \in X_k}c(x_{i}, x_{j}),
\end{align}
and adding the bounds together completes the proof.
\end{proof}
The preceding result yields the aforementioned upper bound as $\min_{\mathcal{X} \in \mathcal{P}_n^K} \mathcal{J}(\mathcal{X}, \sigma(\mathcal{X})) \leq \min\{\mathcal{J}_1, \mathcal{J}_2\}$.
We now state two well-known folklore results relating the sum of intra- and 
inter-dataset distances. For completeness, we provide proofs of both 
statements.

Metrics of negative type form an interesting
class of metrics as they satisfy the
following relationship between the intra-cluster and
inter-cluster variances.
\begin{restatable}{lemma}{cndbound}
\label{lemma:intra_cndbound}
    Suppose $c(\cdot,\cdot)$ is conditionally negative semidefinite. Then, for all 
    sets of points ${X}=\{x_1\ldots,x_n\}$ and 
    ${Y} = \{y_1, \ldots, y_n\}$,
    \begin{equation}\sum_{i=1}^n\sum_{j=1}^nc(x_i,x_j) + \sum_{i=1}^n\sum_{j=1}^nc(y_i,y_j) \leq 2\cdot\sum_{i=1}^n\sum_{j=1}^nc(x_i,y_j).
    \end{equation}
\end{restatable}
\begin{proof}
    Let ${Z} = {X} \cup {Y}$. Define the matrix $D \in \mathbb{R}^{2n \times 2n}$ by $D_{z, z'} = c(z,z')$. Then, since $c(\cdot, \cdot)$ is 
    conditionally negative semidefinite, $\alpha^\top D\alpha\leq 0$ for
    all $\alpha\,\,\text{s.t.}\,\,\alpha^\top \mb{1}_{2n}=0$. Set $\bar{\alpha}_{z} = \mathbbm{1}(z \in {X}) - \mathbbm{1}(z \in {Y})$. Then, since $\lvert{X}\rvert = \lvert{Y}\rvert$ we have $\bar{\alpha}^\top \mb{1}_{2n} = 0$ and therefore
    \begin{equation*}
        \bar{\alpha}^\top D \bar{\alpha} = \sum_{i=1}^n\sum_{j=1}^nc(x_i,x_j)+\sum_{i=1}^n\sum_{j=1}^nc(y_i,y_j)-2\cdot\sum_{i=1}^n\sum_{j=1}^nc(x_i,y_j)\leq 0.
    \end{equation*}
    This completes the proof.
\end{proof}

For arbitrary metrics, the preceding bound
holds with an extra factor of $2$.
\begin{lemma}\label{lemma:folklore_metric_bounds}
    Suppose $c(\cdot, \cdot)$ is a metric. Then, for all sets of points
    ${X} = \{x_1,\ldots,x_n\}$ and ${Y} = \{y_1,\ldots, y_n\}$,
    \begin{equation}\label{eq:intra_UB_cross_metrics}
        \sum_{i=1}^n\sum_{j=1}^nc(x_i,x_j) + \sum_{i=1}^n\sum_{j=1}^nc(y_i,y_j) \leq 2(1+\rho)\cdot\sum_{i=1}^n\sum_{j=1}^nc(x_i,y_j),
    \end{equation}
    for $\rho \in [0, 1]$ defined to be the asymmetry coefficient of 
    intra-dataset distances:
    \begin{equation}
        \rho \defeq \min\left\{
        \frac{\sum_{i=1}^n\sum_{j=1}^nc(x_i,x_j)}
        {\sum_{i=1}^n\sum_{j=1}^nc(y_i,y_j)}, \,\, 
        \frac{\sum_{i=1}^n\sum_{j=1}^nc(y_i,y_j)}
        {\sum_{i=1}^n\sum_{j=1}^nc(x_i,x_j)}
    \right\}.
    \end{equation}
\end{lemma}
\begin{proof}
    Applying the triangle inequality gives the two inequalities
    \begin{align}
        \label{eq:triangle_inequalities_metric}
        c(x_i,x_j) \leq c(x_i,y_k) + c(y_k,x_j) \quad\text{and}\quad c(y_i,y_j) \leq c(y_i,x_k) + c(x_k,y_j)
    \end{align}
    for all $i, j, k \in \{1,\ldots, n\}$. 
    Taking the sum over $i, j, k$ from $1$ to $n$ and applying symmetry of the cost $c(\cdot, \cdot)$ to the first inequality in \eqref{eq:triangle_inequalities_metric} yields
    \begin{equation*}
        n\cdot\sum_{i=1}^n\sum_{j=1}^nc(x_i,x_j) \leq  n\cdot\sum_{i=1}^n\sum_{k=1}^n c(x_i,y_k) + n\cdot\sum_{k=1}^n\sum_{j=1}^n c(y_k,x_j)
    \end{equation*}
    which holds if and only if
    \begin{equation}
        \sum_{i=1}^n\sum_{j=1}^n c(x_i,x_j) \leq 2\cdot\sum_{i=1}^n\sum_{j=1}^n c(x_i, y_j).
    \end{equation}
    Applying a symmetric argument to the second inequality in 
    \eqref{eq:triangle_inequalities_metric} and adding the
    two inequalities shows
    \begin{align*}
        \sum_{i=1}^n\sum_{j=1}^n c(x_i,x_j) + \sum_{i=1}^n\sum_{j=1}^n c(y_i,y_j) \leq 4\cdot \sum_{i=1}^n\sum_{j=1}^n c(x_i, y_j).
    \end{align*}
    When asymmetry $\rho \neq 1$ is present, the preceding
    bound is tightened by refining the bound
    on the smaller term in the left hand side of the preceding
    equation to $2\rho\cdot\sum_{i=1}^n\sum_{j=1}^n c(x_i, y_j)$. 
    This completes the proof.
\end{proof}

The proof of Theorem \ref{thm:reduction_to_kcut} in the metric case
is then a corollary of Lemma \ref{lemma:comb_reformulation_equivalence}, \ref{lemma:metric_intra_ub}, \ref{lemma:intra_cndbound}, and
\ref{lemma:folklore_metric_bounds}. The details are described
in the subsequent proof.

\begin{proof}[Proof of Theorem \ref{thm:reduction_to_kcut} (Metric Costs)]\label{proof:metric}
    Applying the fact that $(\sigma^{-1}(\mathcal{Y}^*), \mathcal{Y}^*)$ and $(\mathcal{X}^*, \sigma(\mathcal{X}^*))$
    are feasible solutions together with the inequality $2\cdot\min\{a, b\} \leq a + b$, we have
    \begin{equation*}
    \min_{\mathcal{X} \in \mathcal{P}_n^K} \mathcal{J}(\mathcal{X}, \sigma(\mathcal{X})) \leq \min\{ \mathcal{J}_1, \mathcal{J}_2\} \leq (1/2) (\mathcal{J}_1 + \mathcal{J}_2).
    \end{equation*}
    Applying Lemma~\ref{lemma:metric_intra_ub} with the optimal solution of $\mathcal{X}^{*} = \{X_1^*,\ldots, X_K^*\}$ and $ \mathcal{Y}^{*} = \{Y_1^*,\ldots,Y_K^*\}$ to the right
    hand side of the preceding inequality then yields the bound
    \begin{equation}
        \min_{\mathcal{X} \in \mathcal{P}_n^K} \mathcal{J}(\mathcal{X}, \sigma(\mathcal{X})) \leq M_{\sigma} + \sum_{k=1}^K\frac{1}{2\lvert X_k^* \rvert}\sum_{i, j \in X_k^*}c(x_{i}, x_{j}) + \sum_{k=1}^K\frac{1}{2\lvert Y_k^* \rvert}\sum_{i, j \in Y_k^*}c(y_{i}, y_{j}), \label{eq:metric_proof_intra_ub}
    \end{equation}
    which upper bounds the cost in terms of the intra-dataset
    costs of $\mathcal{X}^*$ and $\mathcal{Y}^*$. 

    When $c(\cdot,\cdot)$ is negative semidefinite, applying
    Lemma \ref{lemma:intra_cndbound} to \eqref{eq:metric_proof_intra_ub} shows that
    \begin{equation}
        \min_{\mathcal{X} \in \mathcal{P}_n^K} \mathcal{J}(\mathcal{X}, \sigma(\mathcal{X})) \leq M_{\sigma} + \sum_{k=1}^K\frac{1}{\lvert X_k^*\rvert}\sum_{i \in X_k^*}\sum_{j\in Y_k^*}c(x_{i}, y_{j}) = M_\sigma+\mathcal{J}(\mathcal{X}^*,\mathcal{Y}^*),
    \end{equation}
    proving the claim. When $c(\cdot, \cdot)$ is a metric,
    applying Lemma \ref{lemma:folklore_metric_bounds} to \eqref{eq:metric_proof_intra_ub} yields the weaker bound
    \begin{equation}
        \min_{\mathcal{X} \in \mathcal{P}_n^K} \mathcal{J}(\mathcal{X}, \sigma(\mathcal{X})) \leq M_\sigma+(1 + \rho)\cdot\mathcal{J}(\mathcal{X}^*,\mathcal{Y}^*).
    \end{equation}
    Combining these two bounds with the equivalence in Lemma 
    \ref{lemma:comb_reformulation_equivalence} completes the
    proof.
\end{proof}

For kernel costs of the form $c(x_i, y_j) = \lVert \phi(x_i)-\phi(y_j)\lVert^2_2,$ the squared norm $\lVert \cdot \rVert^2_2$ is not 
a metric and the preceding argument no longer applies. While the proof of Theorem \ref{thm:reduction_to_kcut} does go through upon replacing applications of the triangle inequality with applications of the doubled triangle inequality $\lVert x - y\rVert^2_2 \leq 2(\lVert x - z\rVert^2_2 + \lVert z - y\rVert^2_2)$, it reduces the quality of the bound
to $2 + 2\gamma$.
To slightly improve this bound, we derive an analog of Lemma \ref{lemma:metric_intra_ub} for
kernel costs by applying Young's inequality at a different point in the argument and optimizing the bound over the introduced parameter $t$.


As a preliminary,
we will make use of the following relationship between the cross cluster
cost between ${X}$ and ${Y}$ to the intra-cluster cost of 
${Y}$.
\begin{lemma}
\label{lemma:sqeuc_decomposition}
    Let ${X} = \{x_1,\ldots, x_n\}$ and ${Y} = \{y_1,\ldots, y_n\}$. Then,
    \begin{equation}
        \frac{1}{n}\sum_{i=1}^n\sum_{j=1}^n\lVert x_i-y_j\rVert^2_2=\sum_{i=1}^n\lVert x_i-\bm{\mu}({Y})\rVert^2_2+\frac{1}{2n}\sum_{i=1}^n\sum_{j=1}^n\lVert y_i-y_j\rVert^2_2,
    \end{equation}
    where $\bm{\mu}({Y}) = \frac{1}{n}\sum_{i=1}^ny_i$.
\end{lemma}
\begin{proof}
    Inserting $\bm{\mu}({Y}) - \bm{\mu}({Y})$ into the left hand side summation and
    expanding the result yields:
    \begin{align*}
        \frac{1}{n}\sum_{i=1}^n\sum_{j=1}^n\lVert x_i-y_j\rVert^2_2 &=  \frac{1}{n}\sum_{i=1}^n\sum_{j=1}^n\lVert x_i-\bm{\mu}({Y})+\bm{\mu}({Y})-y_j\rVert^2_2\\
        &= \frac{1}{n}\sum_{i=1}^n\sum_{j=1}^n\lVert x_i-\bm{\mu}({Y})\rVert^2_2+\frac{1}{n}\sum_{i=1}^n\sum_{j=1}^n\lVert \bm{\mu}({Y})-y_j\rVert^2_2
        \\
        &+ \frac{1}{n}\sum_{i=1}^n\sum_{j=1}^n\left\langle x_i-\bm{\mu}({Y}), \bm{\mu}({Y})-y_j\right\rangle
        \\
        &= \sum_{i=1}^n\lVert x_i-\bm{\mu}({Y})\rVert^2_2+\frac{1}{2n}\sum_{i=1}^n\sum_{j=1}^n\lVert y_i-y_j\rVert^2_2.
    \end{align*}
    The second equality follows from the identity $\sum_{i=1}^n\lVert y_i - \bm{\mu}({Y})\rVert^2_2 = \frac{1}{2n}\sum_{i=1}^n\sum_{j=1}^n\lVert y_i - y_j\rVert_2^2$
    and the identity:
    \begin{equation*}
        \frac{1}{n}\sum_{i=1}^n\sum_{j=1}^n\left\langle x_i-\bm{\mu}({Y}), \bm{\mu}({Y})-y_j\right\rangle=\sum_{i=1}^n\left\langle x_i-\bm{\mu}({Y}), \frac{1}{n}\sum_{j=1}^n(\bm{\mu}({Y})-y_j)\right\rangle = 0.
    \end{equation*}
    This completes the proof.
\end{proof}
Next, we have the following analog of Lemma \ref{lemma:metric_intra_ub} specialized to the
squared Euclidean cost. The key trick is to expand one of the inner-products, and to apply both Cauchy-Schwarz and Young's inequality term-wise only
after performing the decomposition from Lemma \ref{lemma:sqeuc_decomposition} to obtain
an improved constant.
\begin{lemma}
\label{lemma:euclidean_intra_ub}
    Let $\mathcal{X} =\{X_1,\ldots, X_K\}, \mathcal{Y} = \{Y_1, \ldots, Y_K\}$ be a feasible solution
    to the optimization problem \eqref{eq:comb_reformulation_monge_sep}
    and suppose that $c(x_i, y_j) = \lVert x_i - y_j\rVert^2_2$. 
    Then, for all $t >0$ 
    \begin{align}
         \mathcal{J}_1+\mathcal{J}_2& \leq 
         2 \cdot \left(1 + 1/t \right) M_{\sigma} +
         \left( 2 + t \right) \left[ \sum_{k=1}^{K}\frac{1}{2\lvert X_k\rvert }\sum_{i, j \in X_k}\lVert x_i - x_j\rVert_2^2 
            +\sum_{k=1}^{K}\frac{1}{2\lvert Y_k\rvert }\sum_{i, j \in Y_k}\lVert y_i - y_j\rVert_2^2
            \right]
    \end{align}
    In the preceding, $M_\sigma = \sum_{i=1}^nc(x_i,y_{\sigma(i)})$, $\mathcal{J}_1=\mathcal{J}(\sigma^{-1}(\mathcal{Y}),\mathcal{Y})$, and $\mathcal{J}_2 = \mathcal{J}(\mathcal{X},\sigma(\mathcal{X}))$.
\end{lemma}
\begin{proof}
    By Lemma \ref{lemma:sqeuc_decomposition}, the cost $\mathcal{J}_1$ is equal to
    \begin{align}
        \mathcal{J}_1 &= \sum_{k=1}^K\frac{1}{\lvert Y_k\rvert}\sum_{i\in\sigma^{-1}(Y_k)}\sum_{j\in Y_k}\lVert x_i-y_j\rVert^2\nonumber\\&= \sum_{k=1}^K\left[\sum_{i\in\sigma^{-1}(Y_k)}\lVert x_i-\bm{\mu}(Y_k)\rVert^2 +  \sum_{i \in Y_k}\lVert y_i - \bm{\mu}(Y_k)\rVert^2_2\right] ,\label{eq:euclidean_intra_ub_1}
    \end{align}
    where $\bm{\mu}(Y_k)\defeq \frac{1}{\lvert Y_k\rvert}\sum_{i\in Y_k}y_i$. Next, we expand the inner-product of the left hand side term in \eqref{eq:euclidean_intra_ub_1}. This yields
    \begin{align*}&\sum_{k=1}^K\sum_{i\in\sigma^{-1}(Y_k)}\lVert x_i-\bm{\mu}(Y_k)\rVert^2 =
        \sum_{k=1}^K\sum_{i\in\sigma^{-1}(Y_k)}\lVert x_i-y_{\sigma(i)}+y_{\sigma(i)}-\bm{\mu}(Y_k)\rVert^2\\
        &=
        \sum_{k=1}^K\sum_{i\in\sigma^{-1}(Y_k)}(\lVert x_i-y_{\sigma(i)} \rVert_{2}^{2} +\lVert y_{\sigma(i)}-\bm{\mu}(Y_k)\rVert^2 + 2 \left\langle x_i-y_{\sigma(i)}, y_{\sigma(i)}-\bm{\mu}(Y_k)
        \right\rangle) \\
        &= M_{\sigma} + \sum_{k=1}^K\sum_{i\in\sigma^{-1}(Y_k)}\lVert y_{\sigma(i)}-\bm{\mu}(Y_k)\rVert^2 + \sum_{k=1}^K\sum_{i\in\sigma^{-1}(Y_k)} 2 \left\langle x_i-y_{\sigma(i)}, y_{\sigma(i)}-\bm{\mu}(Y_k)
        \right\rangle
        \end{align*}
        By an application of the Cauchy-Schwarz inequality to each inner product term followed by an application of Young's inequality, we obtain
        \begin{align*}
    \sum_{k=1}^K\sum_{i\in\sigma^{-1}(Y_k)} 2 \left\langle x_i-y_{\sigma(i)}, y_{\sigma(i)}-\bm{\mu}(Y_k)
        \right\rangle &\leq \sum_{k=1}^K\sum_{i\in\sigma^{-1}(Y_k)} 2 \left\lVert x_i-y_{\sigma(i)} \right\rVert_{2} \left\lVert y_{\sigma(i)}-\bm{\mu}(Y_k)
        \right\rVert_{2} \\
        & \leq\sum_{k=1}^K\sum_{i\in\sigma^{-1}(Y_k)} \left(\frac{1}{t} \left\lVert x_i-y_{\sigma(i)} \right\rVert^{2} + t \left\lVert y_{\sigma(i)}-\bm{\mu}(Y_k)
        \right\rVert^{2} \right) \\&= \frac{M_{\sigma}}{t} + t \cdot \sum_{k=1}^K\sum_{i\in\sigma^{-1}(Y_k)} \left\lVert y_{\sigma(i)}-\bm{\mu}(Y_k)
        \right\rVert^{2},
        \end{align*}
        for any parameter $t > 0$.
        Combining with \eqref{eq:euclidean_intra_ub_1}, $\mathcal{J}_1$
        is upper bounded as
        \begin{align*}
            \mathcal{J}_{1} &\leq \left( 2 + t \right)\cdot\sum_{k=1}^{K}\sum_{i \in Y_k}\lVert y_i - \bm{\mu}(Y_k)\rVert^2_2 
            +
           (1 + 1/t ) M_{\sigma} \\
           &= \left( 2 + t \right)\cdot \sum_{k=1}^{K}\frac{1}{2\lvert Y_k\rvert }\sum_{i, j \in Y_k}\lVert y_i - y_j\rVert_2^2 
            +
           (1 + 1/t ) M_{\sigma}
        \end{align*}
The latter part of the equality follows from the identity: 
    \begin{equation*}
        \sum_{i \in Y_k}\lVert y_i - \bm{\mu}(Y_k)\rVert^2_2 = \frac{1}{2\lvert Y_k\rvert }\sum_{i, j \in Y_k}\lVert y_i - y_j\rVert_2^2.
    \end{equation*}
Applying a symmetric argument to derive an upper bound on $\mathcal{J}_2$ yields
\begin{equation*}
\mathcal{J}_{2} \leq \left( 2 + t \right)\cdot \sum_{k=1}^{K}\frac{1}{2\lvert X_k\rvert }\sum_{i, j \in X_k}\lVert x_i - x_j\rVert_2^2 
            +
           (1 + 1/t ) M_{\sigma}
    \end{equation*}
Summing the two bounds then completes the proof.
\end{proof}

\begin{proof}[Proof of Theorem \ref{thm:reduction_to_kcut} (Squared Euclidean and Kernel Costs)]\label{proof:sq_Euclidean}
Replicating the argument of the metric case of Theorem \ref{thm:reduction_to_kcut}, we evaluate $\mathcal{J}_{1}$ and $\mathcal{J}_{2}$ on the sets $(\sigma^{-1}(\mathcal{Y}^*), \mathcal{Y}^*)$ and $(\mathcal{X}^*, \sigma(\mathcal{X}^*))$ where the optimal solution to \eqref{eq:comb_reformulation_monge_sep} is $\mathcal{X}^{*} = \{X_1^*,\ldots, X_K^*\}$ and $ \mathcal{Y}^{*} = \{Y_1^*,\ldots,Y_K^*\}$. This yields
\begin{equation*}
    \min_{\mathcal{X} \in \mathcal{P}_n^K} \mathcal{J}(\mathcal{X}, \sigma(\mathcal{X})) \leq \frac{1}{2}\cdot(\mathcal{J}_1(\sigma^{-1}(\mathcal{Y}^*), \mathcal{Y}^*) + \mathcal{J}_2(\mathcal{X}^*, \sigma(\mathcal{X}^*))).
    \end{equation*}
By Lemma~\ref{lemma:euclidean_intra_ub} and Lemma~\ref{lemma:intra_cndbound}, the right-hand side is upper-bounded by
\begin{align}
    &\phantom{\leq} \left(1 + 1/t \right) \cdot M_{\sigma} +
         \left( 1 + t/2 \right)\cdot \left[ \sum_{k=1}^{K}\frac{1}{2\lvert X_k^*\rvert }\sum_{i, j \in X_k^*}\lVert x_i - x_j\rVert_2^2 
            +\sum_{k=1}^{K}\frac{1}{2\lvert Y_k^*\rvert }\sum_{i, j \in Y_k^*}\lVert y_i - y_j\rVert_2^2
            \right] \nonumber\\
    &\leq \left(1 + 1/t \right) \cdot  M_{\sigma} +
         \left( 1 + t/2 \right) \cdot \sum_{k=1}^{K}\frac{1}{\lvert X_k^*\rvert } \sum_{i\in X_k^*, j\in Y_{k^*}}\lVert x_i - y_j\rVert_2^2 \nonumber\\
         &=\left(1 + 1/t \right) \cdot M_{\sigma} + (1+t/2)\cdot \mathcal{J}(\mathcal{X}^{*}, \mathcal{Y}^{*}).\nonumber
\end{align}
Recalling that $\gamma = \mathrm{OPT}_{n}/\mathrm{OPT}_{r}$, where $\mathrm{OPT}_{n} = M_{\sigma}$ corresponds to the cost of the optimal rank $n$ coupling and $\mathcal{J}(\mathcal{X}^{*}, \mathcal{Y}^{*})$ the optimal rank $r$ cost. Since $\gamma \mathcal{J}(\mathcal{X}^{*}, \mathcal{Y}^{*}) =  \gamma \mathrm{OPT}_{r} = \mathrm{OPT}_{n} = M_{\sigma}$, substituting yields the upper-bound coefficient
\begin{align}
     \left(1 + \frac{1}{t}\right) \cdot \gamma + \left(1+\frac{t}{2} \right) = 1 + \gamma + \frac{\gamma}{t} + \frac{t}{2}
\end{align}
As this holds for arbitrary $t>0$, we optimize the bound with respect to $t$ to find the optimal value $t = \sqrt{2 \gamma}$. Thus, the tightest bound has coefficient
\begin{align*}
    \left(1 + 1/\sqrt{2 \gamma} \right) \gamma + (1+\sqrt{2 \gamma}/2) =1+\gamma + \sqrt{2\gamma},
\end{align*}
concluding the proof.
\end{proof}

Next, we show that the approximation guarantees of Theorem \ref{thm:reduction_to_kcut} hold with an additional $(1 + \epsilon)$
factor for kernel costs (resp. metric costs) when using Algorithm \ref{alg:gen_kmeans_initialization} to solve step (ii) 
in Algorithm \ref{alg:monge_KMeans}. For both metric and kernel costs,
the proof of the approximation guarantees follows
from analyzing the proof of Theorem \ref{thm:reduction_to_kcut}.
\vspace{0.5em}
\begin{algorithm}[]
   \caption{Transport Registered Initialization}
  \label{alg:gen_kmeans_initialization}
\begin{algorithmic}[1]
   \STATE \textbf{input:} Datasets $X$ and $Y$, cost matrix $\mb{C}$, rank $K$, and oracle $\mathcal{A}_1$ (resp. $\mathcal{A}_2$).
   \STATE $\mb{Q} \leftarrow \mathcal{A}_1(X)$
   \STATE $\mb{R} \leftarrow \mathcal{A}_1(Y)$
   \IF{$\langle \mb{C}, \mb{Q}\diag(1/\mb{Q}^{\top}\mb{1}_n)(\mb{P}_{\sigma^*}\mb{Q})\rangle_F\leq \langle \mb{C}, (\mb{P}_{\sigma^*}\mb{R})\diag(1/\mb{R}^{\top}\mb{1}_n)\mb{R}^{\top}\rangle_F$}
   \STATE $\mb{Q}^{(0)} \leftarrow \mb{Q}$
   \ELSE
   \STATE $\mb{Q}^{(0)} \leftarrow \mb{P}_{\sigma^*}\mb{R}$
   \ENDIF
   \STATE \textbf{output} $(\mb{Q}^{(0)},  \mb{P}_{\sigma^{\star}}^\top\mb{Q}^{(0)})$
\end{algorithmic}
\end{algorithm}

We start with the proof for squared Euclidean and kernel cost functions.
\begin{proof}[Proof of Theorem \ref{thrm:K_means_init_guarantee} (Squared Euclidean and Kernel Costs)]
    Let $\mathcal{X}=\{X_1,\ldots, X_K\}$ and $\mathcal{Y}=\{Y_1,\ldots,Y_K\}$ be $(1 + \epsilon)$ 
    optimal solutions to the
    $K$-means problem emitted by $\mathcal{A}_1$ for $X$ and $Y$ respectively. Let $\mathcal{X}^{*} = \{X_1^*,\ldots, X_K^*\}$ and $ \mathcal{Y}^{*} = \{Y_1^*,\ldots,Y_K^*\}$ be the optimal solution to \eqref{eq:comb_reformulation_monge_sep}.
    Then, by the $(1+\epsilon)$ optimality of $\mathcal{X}$ and $\mathcal{Y}$ and Lemma \ref{lemma:intra_cndbound}, we have
    \begin{align*}
        &\sum_{k=1}^{K}\frac{1}{2\lvert Y_k\rvert }\sum_{i, j \in X_k}\lVert x_i - x_j\rVert_2^2 
            +\sum_{k=1}^{K}\frac{1}{2\lvert Y_k\rvert }\sum_{i, j \in Y_k}\lVert y_i - y_j\rVert_2^2 \\\leq\quad 
            &(1+\epsilon)\cdot(\sum_{k=1}^{K}\frac{1}{2\lvert X^*_k\rvert }\sum_{i, j \in X_k^*}\lVert x_i - x_j\rVert_2^2 
            +\sum_{k=1}^{K}\frac{1}{2\lvert Y_k^*\rvert }\sum_{i, j \in Y_k^*}\lVert y_i - y_j\rVert_2^2)\\
            \leq\quad&(1+\epsilon)\cdot\mathcal{J}(\mathcal{X}^*,\mathcal{Y}^*).
    \end{align*}
    
    By Lemma \ref{lemma:euclidean_intra_ub}, and the preceding 
    inequality it then follows that for all $t > 0$
    \begin{equation*}
        \frac{1}{2}(\mathcal{J}(\mathcal{X}, \sigma(\mathcal{X})) + \mathcal{J}(\sigma^{-1}(\mathcal{Y}),\mathcal{Y}))\leq(1+1/t)\cdot M_{\sigma}+(1+\epsilon)(1+t/2)\cdot\mathcal{J}(\mathcal{X}^*,\mathcal{Y}^*).
    \end{equation*}
    Optimizing the parameter $t$ as in the proof of Theorem \ref{thm:reduction_to_kcut} and taking the minimum
    over the two solutions completes the proof.
\end{proof}

An analogous proof technique to Theorem 2 applies to the 
metric case with the additional
application of the basic inequality:
\begin{equation*}
    \min_{j}\sum_{i=1}^nc(x_i,x_{j})\leq\frac{1}{n}\sum_{i=1}^n\sum_{j=1}^nc(x_i,x_j)\leq 2\cdot\min_{j}\sum_{i=1}^nc(x_i,x_{j}).
\end{equation*}
This inequality relates the $K$-medians objective to the 
intra-cluster cost $\sum_{i=1}^n\sum_{j=1}^nc(x_i,x_j)$ 
used in Lemma \ref{lemma:metric_intra_ub} and picks
up an additional factor of $2$. \qed

\subsection{Lower bounds for Theorem \ref{thm:reduction_to_kcut}}
\label{appendix:lowerbounds}

\begin{figure}[t]
  \centering
  \begin{subfigure}{0.65\linewidth}
    \centering
    \raisebox{7mm}{
    \resizebox{\linewidth}{!}{%
      \begin{tikzpicture}[x=3cm,y=3cm,>=stealth]
          \def\eps{0.35}
          \newcommand{\bball}[2]{\draw[thick,fill=white] (#1,#2) circle (\rad);}
          \newcommand{\wball}[2]{\draw[thick,fill=black] (#1,#2) circle (\rad);}
          \tikzset{ticklabel/.style={above=-1pt, fill=white, inner sep=1pt, outer sep=0pt, rounded corners=1pt}}
        
          \def\rad{0.04}
          \def\dx{0.28}    
          \def\dy{0.3}  
          \newcommand{\pokeball}[2]{%
            \draw[thick,fill=white] ($(#1,#2)+(-\dx/2,0)$) circle (\rad);
            \draw[thick,fill=black] ($(#1,#2)+(\dx/2,0)$)  circle (\rad);
            \draw[thick]
              ($(#1,#2)+(-\dx/2+\rad,0)$) -- ($(#1,#2)+(\dx/2-\rad,0)$);
          }
        
          \draw[densely dotted] (0,0) -- (0,\eps) node[midway,left=6pt] {$\epsilon$};
          \draw[densely dotted] (2,0) -- (2,\eps) node[midway,right=6pt] {$\epsilon$};
        
          \draw[densely dotted] (0,0) -- (0,-\eps/3) node[midway,left=-2.5pt] {\tiny $\frac{\epsilon}{3}$};
          \draw[densely dotted] (0,-\eps/3) -- (0,-0.666*\eps) node[midway,left=-2.5pt] {\tiny $\frac{\epsilon}{3}$};
          \draw[densely dotted] (0,-0.666*\eps) -- (0,-\eps) node[midway,left=-2.5pt] {\tiny $\frac{\epsilon}{3}$};
        
          \draw[densely dotted] (2.0,0) -- (2.0,-\eps/3) node[midway,left=-2.5pt] {\tiny $\frac{\epsilon}{3}$};
          \draw[densely dotted] (2.0,-\eps/3) -- (2.0,-0.666*\eps) node[midway,left=-2.5pt] {\tiny $\frac{\epsilon}{3}$};
          \draw[densely dotted] (2.0,-0.666*\eps) -- (2.0,-\eps) node[midway,left=-2.5pt] {\tiny $\frac{\epsilon}{3}$};
          
          \draw[densely dotted] (1,0) -- (1,-\eps) node[midway,left=-1pt,yshift=4pt] {$\epsilon$};
          
          \wball{0}{\eps};
          \bball{2}{\eps};
          \node[label=above:{\footnotesize $P$}] at (0,\eps) {};
          \node[label=above:{\footnotesize $Q$}] at (2,\eps) {};
          
          \pokeball{1.0}{-\eps};
          \node[label=above:{\tiny $M_W$}] at (1.0-\dx/2,-\eps) {};
          \node[label=above:{\tiny $M_B$}] at (1.0+\dx/2,-\eps) {};
        
          \pokeball{0.0}{-\eps/3};
          \node[label=left:{\tiny $L_W^1$}] at (0.0-\dx/2,-\eps/3) {};
          \node[label=right:{\tiny $L_B^1$}] at (0.0+\dx/2,-\eps/3) {};
          \pokeball{0.0}{-0.666*\eps};
          \node[label=left:{\tiny $L_W^2$}] at (0.0-\dx/2,-0.666*\eps) {};
          \node[label=right:{\tiny $L_B^2$}] at (0.0+\dx/2,-0.666*\eps) {};
          \pokeball{0.0}{-1.0*\eps};
          \node[label=left:{\tiny $L_W^3$}] at (0.0-\dx/2,-1.0*\eps) {};
          \node[label=right:{\tiny $L_B^3$}] at (0.0+\dx/2,-1.0*\eps) {};
        
          \pokeball{2.0}{-\eps/3};
          \node[label=left:{\tiny $R_W^1$}] at (2.0-\dx/2,-\eps/3) {};
          \node[label=right:{\tiny $R_B^1$}] at (2.0+\dx/2,-\eps/3) {};
          \pokeball{2.0}{-0.666*\eps};
          \node[label=left:{\tiny $R_W^2$}] at (2.0-\dx/2,-0.666*\eps) {};
          \node[label=right:{\tiny $R_B^2$}] at (2.0+\dx/2,-0.666*\eps) {};
          \pokeball{2.0}{-1.0*\eps};
          \node[label=left:{\tiny $R_W^3$}] at (2.0-\dx/2,-1.0*\eps) {};
          \node[label=right:{\tiny $R_B^3$}] at (2.0+\dx/2,-1.0*\eps) {};
        
          \draw[<->] (-0.4,0) -- (2.4,0);
          \foreach \x/\lab in {0, 0.5, 1, 1.5, 2}
          \draw (\x,0) -- ++(0,0.06) node[ticklabel, above=2.5pt] {$\lab$};
      \end{tikzpicture}
    }
    }
  \end{subfigure}
  \hfill
  \begin{subfigure}{0.34\linewidth}
    \centering
    \resizebox{\linewidth}{!}{%
      \begin{tikzpicture}[x=3cm,y=3cm,>=stealth]
  \def\eps{0.20}
  \def\epsp{0.13}

  \def\rad{0.04}
  \newcommand{\bball}[2]{\draw[thick,fill=black] (#1,#2) circle (\rad);}
  \newcommand{\wball}[2]{\draw[thick,fill=white] (#1,#2) circle (\rad);}
  \tikzset{ticklabel/.style={above=-1pt, fill=white, inner sep=1pt, outer sep=0pt, rounded corners=1pt}}
    
  \draw[decorate, decoration={brace, raise=2pt}, line width=0.9pt] (2-0.02,-1+0.02) -- (2-0.02,0-0.02)
  node[midway, left=10pt, yshift=-3pt, ticklabel] {$1$};
  
  \draw[decorate, decoration={brace, mirror, raise=5pt}, thick] (1+0.02,0) -- (1+\eps-0.02,0)
  node[midway, below=15pt, ticklabel] {$\epsilon$};
  \draw[decorate, decoration={brace, mirror, raise=5pt}, thick] (1-0.02,-2) -- (1-0-\eps+0.02,-2)
  node[midway, above=10pt, ticklabel] {$\epsilon$};
  
  \def\ticklen{0.07}
  \draw[-, line width=0.7pt] (0,0) -- (2,0) -- (2, -2) -- (0, -2) -- cycle;
  \draw (1,+0) -- ++(0,-\ticklen);
  \draw (1,-2) -- ++(0,+\ticklen);
  \draw (2,-1) -- ++(-\ticklen,0);
  \draw (0,-1) -- ++(+\ticklen,0);
  
  \draw[-, thick] (+0.5*\epsp, +0.5*\epsp) -- (-0.5*\epsp, -0.5*\epsp);
  \node[label=right:{\small $L_B^1$}] at (+0.4*\epsp,+0.55*\epsp) {};   
  \node[label=below:{\small $L_W^1$}] at (-0.75*\epsp,-0.4*\epsp) {}; 
  \bball{+0.5*\epsp}{+0.5*\epsp};
  \wball{-0.5*\epsp}{-0.5*\epsp};

  \draw[-, thick] (0, +1.1*\epsp) -- (-1.1*\epsp, 0);
  \bball{+0}{+1.1*\epsp};
  \wball{-1.1*\epsp}{0};
  \node[label=above:{\small $L_B^2$}] at (+0.7*\epsp,+0.7*\epsp) {};   
  \node[label=left:{\small $L_W^2$}] at (-0.5*\epsp,-0.5*\epsp) {};   
  
  \draw[-, thick] (2-0.5*\epsp, -0.5*\epsp-2) -- (2+0.5*\epsp, +0.5*\epsp-2);
  \bball{2+0.5*\epsp}{+0.5*\epsp-2};
  \node[label=above:{\small $R_B^1$}] at (2+0.75*\epsp,+0.4*\epsp-2) {};   
  \wball{2-0.5*\epsp}{-0.5*\epsp-2};
  \node[label=left:{\small $R_W^1$}] at (2-0.3*\epsp,-0.55*\epsp-2) {};

  \draw[-, thick] (2+1.1*\epsp, -2) -- (2, -1.1*\epsp-2);
  \bball{2+1.1*\epsp}{-2};
  \node[label=right:{\small $R_B^2$}] at (2+0.7*\epsp,+0.7*\epsp-2) {};   
  \wball{2}{-1.1*\epsp-2};
  \node[label=below:{\small $R_W^2$}] at (2-0.5*\epsp,-0.7*\epsp-2) {};

  \bball{1+\eps}{0};
  \node[label=above:{$P_1$}] at (1+\eps,0) {};  
  
  \wball{2}{-1};
  \node[label=right:{$Q_1$}] at (2,-1) {};  
  
  \bball{1-\eps}{-2};
  \node[label=below:{$P_2$}] at (1-\eps,-2) {};  

  \wball{0}{-1};
  \node[label=left:{$Q_2$}] at (0,-1) {}; 
\end{tikzpicture}
    }
  \end{subfigure}
  \caption{\label{fig:geometric_lower_bounds}
    Geometric constructions providing 
  lower bounds for Theorem \ref{thm:reduction_to_kcut} in the case of \textbf{(left)}
  Euclidean cost ($k$ = 3) and \textbf{(right)} squared Euclidean cost ($k$ = 2). Points in
  ${X}$ are colored black and points in ${Y}$ are colored white. Points connected by a line segment have identical coordinates and are separated
  for ease of visualization.
  }
\end{figure}
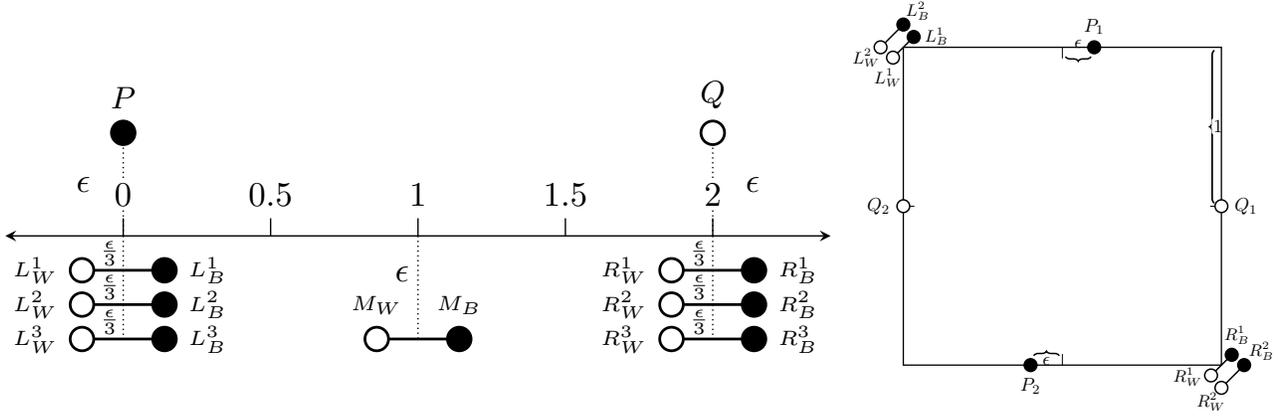

In this section, we construct an explicit family of examples that realize the upper
bounds in Theorem \ref{thm:reduction_to_kcut} for the Euclidean
and squared Euclidean cost functions. Both constructions rely on unstable
arrangements of points, where upon slight perturbation, the Monge map changes
dramatically. Making use of this instability, the constructions
are set up to have the optimal Monge map be a poor choice for co-clustering
while there is a near-optimal non-Monge map that is substantially better
for co-clustering.

\paragraph{Euclidean Metric Cost.}

Fix $\epsilon > 0$.
The construction consists of two datasets ${X}$ and ${Y}$ 
each with $2k + 2$ points placed near the line segment $[0, 2] \times \{0\}$.
The first two points $P = (0, \epsilon)$ and 
$Q = (2, \epsilon)$ are near the ends of the line segment.
The next pair of points $M_W = M_B = (1, -\epsilon)$ are slightly below the middle
of the segment. Finally, $2k$ points 
$L_W^i = L_B^i=(0, -\frac{i\epsilon}{k})$ are at the left end of the segment 
and  $2k$ points
$R_W^i = R_W^i=(2, -\frac{i\epsilon}{k})$ are at the right end of
the segment. Datasets ${X}$ and ${Y}$
are then defined as 
${X} = \{P, M_B\}\cup\{L_B^i\}_{i=1}^k\cup\{R_B^i\}_{i=1}^k$ and 
${Y} = \{Q, M_W\}\cup\{L_W^i\}_{i=1}^k\cup\{R_W^i\}_{i=1}^k$.
A diagram of the construction is provided in Figure \ref{fig:geometric_lower_bounds} 
for the case of $k = 3$.

First, observe that under the Euclidean cost metric, the points have a unique
Monge map $\sigma : {X} \rightarrow {Y}$ defined as 
$\sigma(P) = Q, \sigma(M_B) = M_W, \sigma(L_B^i)=L_W^i,$ and $\sigma(R_B^i)=R_W^i$. 
The preceding Monge map $\sigma$ has cost $2$
since the distance between $P$ and $Q$ is $2$ while the distance between
the remaining mapped points is $0$. Next, consider an optimal Monge map $\sigma'$
with $\sigma'(P) \neq Q$. 
Since there are $k + 1$ black points $\{L_B^i\}_{i=1}^k$ and $P$
and $k$ points $\{L_W^i\}_{i=1}^k$, at least one black point $\{L_B^i\}_{i=1}^k$ or $P$ must map to $M_W$ or $\{R_W^i\}_{i=1}^k$ via $\sigma'$. If
$P$ is such a point, we must have that $P$ maps to $M_W$, as otherwise it would
obtain a cost greater than $2$. However, this yields a contradiction as $M_B$ must 
map to a point of distance at least $1$ and the cost of mapping $\sigma'(P) = M_W$ 
is strictly greater than $1$. Applying a similar argument to the point $L_B^i$
together with the fact that the cost of mapping $\sigma'(P) = L_B^i$ is greater
than $\epsilon$ proves the optimality and uniqueness of $\sigma$.

Second, consider an optimal solution $\mathcal{X}^* = \{X_1, X_2\}, \mathcal{Y}^* = \{Y_1, Y_2\}$ to the
partition reformulation \eqref{eq:comb_reformulation_monge_sep}
of the ($K=2$) low-rank OT problem where
$\sigma(X_i) = Y_i$, $i = 1, 2$, and the cluster sizes are balanced:
$\lvert X_1 \rvert = \lvert X_2\rvert = \lvert Y_1 \rvert = \lvert Y_2\rvert$.
We will argue that the cost of such a solution $\mathcal{J}(\mathcal{X}^*, \mathcal{Y}^*)$ is lower bounded by 
$\frac{4k + 2}{k+1}$. In contrast, taking the solution 
$X_1 = \{P\} \cup \{L_B^i\}_{i=1}^k, X_2 = \{M_B\} \cup \{R_B^i\}_{i=1}^k$ and $Y_1 = \{M_W\} \cup \{L_W^i\}_{i=1}^k, Y_2 = \{Q\} \cup \{R_W^i\}_{i=1}^k$,
which does not satisfy $\sigma(X_i) = Y_i$,
we have that the cost 
\begin{equation*}
    \mathcal{J}(\{X_1, Y_1\}, \{X_2, Y_2\}) = \frac{1}{k+1}\left(\sum_{x\in X_1}\sum_{y\in Y_1}\lVert x - y\rVert_2 + \sum_{x\in X_2}\sum_{y\in Y_2}\lVert x - y\rVert_2 \right) = 2 + \mathcal{O}(\epsilon).
\end{equation*}
Consequently, taking the limits $\epsilon \rightarrow 0$ and
$k \rightarrow \infty$ shows that the constant factor $(1 + \gamma)$ 
stated in Theorem \ref{thm:reduction_to_kcut} is tight
and arbitrarily close to $2$.

Finally, we argue that the cost of any solution $\mathcal{X}^* = \{X_1, X_2\}, \mathcal{Y}^* = \{Y_1, Y_2\}$ with $\sigma(X_i) = Y_i$ has $\mathcal{J}(\mathcal{X}^*, \mathcal{Y}^*) \geq \frac{4k + 2}{k+1}$. Without loss of generality, assume that
$P \in X_1$. Since $\sigma(P) = Q$ and $\sigma(X_1) = Y_1$, it follows that $Q \in Y_1$. Let $l$ denote the size of the set $\{i : L_B^i \in X_1\}$. We analyze
the two cases $M_B \in X_1$ and $M_B\notin X_1$ 
separately.

\textbf{Case 1} ($M_B \in X_1$)\textbf{.} Since the
set sizes are balanced and $M_B \in X_1$, 
we have $p = k + 1 - (l + 2)$
is equal to the size of the set $\{i : R_B^i \in X_1\}$.
Consequently, the cost of the solution is lower bounded
by:
\begin{align*}
    (k+1)\cdot\mathcal{J}(\mathcal{X}^*, \mathcal{Y}^*) &\geq (4+3p+2l(p+1)) + 2(k-p)(k-l)\\
            &= -4 l^2 - (5-4k)l + (5k+1)
\end{align*}
To see this, note that the cost of mapping point $P \in X_1$ to $Y_1$ is at least $3 + 2p$ as $P$ must map to the $p$ points $R_W^i$, $M_W$, and $Q$. The cost of mapping the $l$ points $L_B^i$ to $Y_1$ is at least $2l(p+1)$ 
and the cost of mapping $M_B$ to $Y_1$ is at least 
$p + 1$. Since the size of $X_2 \cap \{L_B^i\}_{i=1}^k$ is $k - l$
and the size of $X_2 \cap \{R_B^i\}_{i=1}^k$ is
$k - p$, the cost of mapping $X_2$ to $Y_2$ is at least 
$2(k-p)(k-l)$. Adding the lower bounds gives the first bound and algebraic manipulation the second.
Since the lower bound is a
concave quadratic function in $l$, it is minimized
at either $l = 0$ or $l = k -1$. Evaluating both
yields the lower bound $\mathcal{J}(\mathcal{X}^*,\mathcal{Y}^*)\geq \frac{4k+2}{k+1}$.

\textbf{Case 2} ($M_B \notin X_1$)\textbf{.} 
Since the
set sizes are balanced and $M_B \notin X_1$, we have $p = k - l$
is equal to the size of the set $\{i : R_B^i \in X_1\}$.
By a similar argument to the previous case, we have that
the cost of the solution is lower bounded by:
\begin{align*}
(k+1)\cdot\mathcal{J}(\mathcal{X}^*, \mathcal{Y}^*) &\geq 2+2p+2l(p+1)+2lp+2(k-l+k-p)+4(k-l)(k-p)\\
                                          &= -8l^2+8kl+(4k+2).
\end{align*}
Since the lower bound is again a concave quadratic function in $l$, it is minimized
at either $l = 0$ or $l = k$. Evaluating both yields the lower bound $\mathcal{J}(\mathcal{X}^*, \mathcal{Y}^*) \geq \frac{4k+2}{k+1}$. This completes the proof of the first part of
Proposition \ref{prop:lowerbound}. \qed

\paragraph{Squared Euclidean Cost.}

Fix $1 > \epsilon > 0$.
The construction consists of two datasets ${X}$ and ${Y}$ 
each with $2k + 2$ points placed along the edges of the square
$[0, 2]\times[0,2]$. The first two points $P_1 = (1 + \epsilon, 2)$ and $P_2 = (1 - \epsilon, 0)$
are on the top and bottom edges of the square. The second two points
$Q_1 = (0, 1)$ and $Q_2 = (2, 1)$ are set on the left and right edges of the square.
Finally, $2k$ points $L_W^i = L_B^i = (0, 2)$ and $R_W^i = R_B^i = (2, 0)$
are placed along the top left and bottom right corners of the square. The
sets $X$ and $Y$ are then defined as 
$X =\{P_1, P_2\}\cup\{L_B^i\}_{i=1}^k\cup\{R_B^i\}_{i=1}^k$
and $Y =\{Q_1, Q_2\}\cup\{L_W^i\}_{i=1}^k\cup\{R_W^i\}_{i=1}^k$. 
A diagram of the construction is provided in Figure \ref{fig:geometric_lower_bounds} 
for the case when $k = 2$.

First, we show that under the squared Euclidean cost function
there is a unique Monge map $\sigma : X\rightarrow Y$
defined as $\sigma(P_i)=Q_i$, $\sigma(L_B^i)=L_W^i$, and $\sigma(R_B^i)=R_B^i$,
up to a relabeling of the corner points $\{L_B^i\}_{i=1}^k\cup \{R_B^i\}_{i=1}^k$.
The preceding Monge map has cost equal to $4 + 2\epsilon^2-4\epsilon < 4$
as $\epsilon^2 < \epsilon$. Next, suppose that there is a distinct (up to a relabeling of the corner points) Monge map $\sigma'$
with  equal (or lower) cost. Note that since $\sigma'$
is an optimal Monge map, then $\sigma'(L_B^i) \neq R_W^j$ for any $j$, as the
cost of mapping point $L_B^i$ to $R_W^j$ is $8$. Similarly, $L_B^i$ cannot 
map to $Q_1$. Therefore, if $\sigma' \neq \sigma$
 it must be the case that $\sigma'(L_B^i)=Q_2$
for some $i$. Then, either $\sigma'(P_1) = Q_1$ or $\sigma'(P_2) = Q_1$, but in either
case mapping the remaining point results in a cost of at least $4$, a contradiction
to the optimality of $\sigma'$.

Second, consider an optimal solution $\mathcal{X}^* = \{X_1, X_2\}, \mathcal{Y}^* = \{Y_1, Y_2\}$ to the
partition reformulation \eqref{eq:comb_reformulation_monge_sep}
of the ($K=2$) low-rank OT problem where
$\sigma(X_i) = Y_i$, $i = 1, 2$, and the cluster sizes are balanced:
$\lvert X_1 \rvert = \lvert X_2\rvert = \lvert Y_1 \rvert = \lvert Y_2\rvert$.
We will argue that the cost of such a solution $\mathcal{J}(\mathcal{X}^*, \mathcal{Y}^*)$ is lower bounded by 
$\frac{12k+4}{k+1}+\mathcal{O}(\epsilon)$. In contrast, taking the solution 
$X_1 = \{P_1\} \cup \{L_B^i\}_{i=1}^k, X_2 = \{P_2\} \cup \{R_B^i\}_{i=1}^k$ and $Y_1 = \{Q_2\} \cup \{L_W^i\}_{i=1}^k, Y_2 = \{Q_1\} \cup \{R_W^i\}_{i=1}^k$,
which does not satisfy $\sigma(X_i) = Y_i$,
we obtain the cost 
\begin{equation*}
    \mathcal{J}(\{X_1, Y_1\}, \{X_2, Y_2\}) = 4+\mathcal{O}(\epsilon).
\end{equation*}
Consequently, taking the limits $\epsilon \rightarrow 0$ and
$k \rightarrow \infty$ shows that the constant factor 
stated in Theorem \ref{thm:reduction_to_kcut} is lower bounded by $3$
in the worst case.

Finally, we argue that the cost of any solution $\mathcal{X}^* = \{X_1, X_2\}, \mathcal{Y}^* = \{Y_1, Y_2\}$ with $\sigma(X_i) = Y_i$ has $\mathcal{J}(\mathcal{X}^*, \mathcal{Y}^*) \geq \frac{12k + 4}{k+1}$. Without loss of generality, assume that
$P_1 \in X_1$. Since $\sigma(P_1) = Q_1$ and $\sigma(X_1) = Y_1$, it follows that $Q_1 \in Y_1$. Let $l$ denote the size of the set $\{i : L_B^i \in X_1\}$. We analyze
the two cases $P_2 \in X_1$ and $P_2 \notin X_1$ 
separately.

\textbf{Case 1} ($P_2 \notin X_1$)\textbf{.}
In this case,
the size of the set $\{i : R_B^i \in X_1\}$ is $p = k -l$
following the fact that $\lvert X_1 \rvert = k + 1$
and $P_2 \notin X_1$. Then, the cost of the solution
is lower bounded by:
\begin{align*}
    (k+1)\cdot\mathcal{J}(\mathcal{X}^*, \mathcal{Y}^*) &\geq &&(8lp+5l)+(8lp+p)+(2+l+5p)+(8(k-l)(k-p)+k-l)+\\&\phantom{\geq} &&(8(k-l)(k-p)+5(k-p))+(2+5(k-l)+k-p)\\
    &= &&- 32 l^2 + 32 k l +4 (1 + 3 k).
\end{align*}
To derive the previous bound, we explicitly tabulate the
cost between all types of points. Specifically, the
cost of mapping $\{L_B^{i}:L_B^i\in X_1\}$ to all points in $Y_1$
is at least $8lp+5l$. The cost of mapping $\{R_B^{i}:R_B^i\in X_1\}$ to all points in $Y_1$
is at least $8lp+p$. The cost of mapping $P_1$ in $Y_1$ is $2 + l + 5p$. Proceeding
in this way for the points in $X_2$ yields the first inequality. Algebra with the substitution $p = k - l$ yields the equality. Since
the lower bound is a concave quadratic function in $l$, it is either optimized
at $l = 0$ or $l = k$. Evaluating the lower bound at these points yields
the inequality $\mathcal{J}(\mathcal{X}^*, \mathcal{Y}^*) \geq \frac{12k+4}{k+1}.$

\textbf{Case 2} ($P_2 \in X_1$)\textbf{.} In this case,
the size of the set $\{i : R_B^i \in X_1\}$ is $p = k - l - 1$
following the fact that $\lvert X_1 \rvert = k + 1$
and $P_2 \notin X_1$.
Then, the cost of the solution
is lower bounded by:
\begin{align*}
    (k+1)\cdot\mathcal{J}(\mathcal{X}^*, \mathcal{Y}^*) &\geq (l+5l+8lp)+(4+l+5p)+(4+5l+p)\\& +(8lp+5p+p)+2\cdot8(k-p)(l-l)\\
    &= - 25 l^2 + (-25 + 25 k) l+ 28 k -4.
\end{align*}
This follows an explicit tabulation of the cost between all types of points.
Since
the lower bound is a concave quadratic function in $l$, it is either optimized
at $l = 0$ or $l = k$ - 1. Evaluating the lower bound at these points yields
the inequality $\mathcal{J}(\mathcal{X}^*, \mathcal{Y}^*) \geq \frac{28k+4}{k+1}\geq\frac{12k+4}{k+1}.$ This completes the proof of the second part of
Proposition \ref{prop:lowerbound}. \qed

\subsection{Proof of connection between low-rank optimal transport and $K$-means}

We provide a brief proof of
the connection between low-rank optimal 
transport and $K$-means stated in the main text. The
statement and proof are attached below.

\begin{remark}\label{remark:LOT_as_2KMeans}
    Suppose $\bm{\mu}_{k}^{X} = \frac{1}{|\mathcal C_k|} \sum_{i \in \mathcal{C}_{X,k}} x_{i}, \, \bm{\mu}_{k}^{Y} = \frac{1}{|\mathcal C_k|} \sum_{j \in \mathcal{C}_{Y,k}} y_{j}$ where $|\mathcal{C}_{X,k} | =|\mathcal{C}_{Y,k}| = |\mathcal{C}_k|$. Then, for the squared Euclidean cost \eqref{eq:hardLOT_partition} is equal to a pair of $K$-means distortions on $X,Y$ and a term quantified the separation between assigned means $\bm{\mu}_{k}^{X}, \bm{\mu}_{k}^{Y}$
    {
    \begin{align*}
        &\sum_{k=1}^K \frac{1}{\lvert X_k \rvert}\sum_{i\in X_k}\sum_{j \in Y_k} \lVert x_i - y_j\rVert_{2}^{2}
        \\
        &= \sum_{k=1}^{K}\biggl(\sum_{ i \in \mathcal{C}_{X,k} } \lVert x_{i} - \bm{\mu}_{k}^{X} \rVert_{2}^{2} + \sum_{j \in \mathcal{C}_{Y,k}}\lVert y_{j} - \bm{\mu}_{k}^{Y} \rVert_{2}^{2}
+|\mathcal{C}_k | \lVert \bm{\mu}_{k}^{X} - \bm{\mu}_{k}^{Y} \rVert_{2}^{2}\biggr)
    \end{align*}
    }
\end{remark}
\begin{proof}
Starting from the definition of the partition form of the
low-rank cost $\mb{C}^{\ell_{2}^{2}}$, we have
\begin{align*}
\sum_{k=1}^K\frac{1}{\lvert X_k \rvert}\sum_{i\in X_k}\sum_{j \in Y_k} \lVert x_i - y_j\rVert_{2}^{2}
= \sum_{k=1}^{K}  \sum_{i \in \mathcal{C}_{X,k}}  \lVert x_{i} \rVert_{2}^{2} + 
\sum_{j \in \mathcal{C}_{Y,k}}\lVert y_{j} \rVert_{2}^{2}
- 2\cdot|\mathcal{C}_{k}| \langle \bm{\mu}_{k}^{X}, \bm{\mu}_{k}^{Y} \rangle
\end{align*}
By adding and subtracting $|\mathcal{C}_k |\lVert \bm{\mu}_{k}^{X} \rVert_{2}^{2}$ and $|\mathcal{C}_k | \lVert \bm{\mu}_{k}^{Y} \rVert_{2}^{2}$, we find the right hand side is equal to
\begin{align*}
\sum_{k=1}^{K} \biggl( \sum_{i \in \mathcal{C}_{X,k}}  \lVert x_{i} \rVert_{2}^{2} 
  - |\mathcal{C}_k | \lVert \bm{\mu}_{k}^{X} \rVert_{2}^{2}
  + &
\sum_{j \in \mathcal{C}_{Y,k}}\lVert y_{j} \rVert_{2}^{2} - |\mathcal{C}_k | \lVert \bm{\mu}_{k}^{Y} \rVert_{2}^{2} \\
&+|\mathcal{C}_k | \left( \lVert \bm{\mu}_{k}^{X} \rVert_{2}^{2}
- 2\cdot \langle \bm{\mu}_{k}^{X}, \bm{\mu}_{k}^{Y} \rangle + \lVert \bm{\mu}_{k}^{Y} \rVert_{2}^{2} \right) \biggr)
\end{align*}
We conclude by observing $\sum_{k=1}^{K}  \sum_{i \in \mathcal{C}_{X,k}}  \lVert x_{i} \rVert_{2}^{2} 
  - |\mathcal{C}_k | \lVert \bm{\mu}_{k}^{X} \rVert_{2}^{2} = \sum_{k=1}^{K}  \sum_{i \in \mathcal{C}_{X,k}}  \lVert x_{i} - \bm{\mu}_{k}^{X} \rVert_{2}^{2}$ (resp. for $Y$) and identifying $\lVert \bm{\mu}_{k}^{X} \rVert_{2}^{2}
- 2\cdot \langle \bm{\mu}_{k}^{X}, \bm{\mu}_{k}^{Y} \rangle + \lVert \bm{\mu}_{k}^{Y} \rVert_{2}^{2}$ as a difference between means. 
This results in the following form for the right hand side:
\begin{align*}
\sum_{k=1}^{K} \biggl(\sum_{i \in \mathcal{C}_{X,k}}  \lVert x_{i} - \bm{\mu}_{k}^{X} \rVert_{2}^{2} 
  + 
\sum_{j \in \mathcal{C}_{Y,k}}\lVert y_{j} - \bm{\mu}_{k}^{Y} \rVert_{2}^{2} 
+|\mathcal{C}_k | \lVert \bm{\mu}_{k}^{X} - \bm{\mu}_{k}^{Y} \rVert_{2}^{2}\biggr),
\end{align*}
and completes the proof.
\end{proof}

\subsection{Generalized $K$-Means  Algorithms}

\subsubsection{Mirror Descent (\texttt{GKMS})}
\label{sec:GKMS}

Here we present an algorithm for generalized $K$-means -- which we call (\texttt{GKMS}) -- that solves generalized $K$-means locally using mirror-descent. \texttt{GKMS} consists of a sequence of mirror-descent steps with the neg-entropy mirror map $\psi(q) = -\sum q_{ij} \log q_{ij}$ with $\mathrm{KL}$ as the proximal function. This results in a sequence of exponentiated gradient steps with Sinkhorn projections onto a single marginal \cite{sinkhorn1966relationship}. Notably, Lloyd's algorithm for $K$-means, which is the most popular local heuristic for minimizing the $K$-means objective, alternates between an update to means $\{\bm{\mu}_{k}\}_{k=1}^K \subset \mathbb{R}^{d}$ and hard-cluster assignments $\bm{Z} \in \{0, 1\}^{n \times K}$. This algorithm only optimizes cluster assignments for a fixed cost $\mb{C}$, lacking an explicit notion of points or centers in $\mathbb{R}^{d}$. Moreover, it permits dense initial conditions $\mb{Q}^{(0)}$ and represents $\mb{Q} \in \mathbb{R}^{n \times K}_{+}$. As the loss lacks entropic regularization, in theory the sequence $(\mb{Q}^{(n)})_{n=1}^\infty$ converges in $\ell_{2}$ to sparse solutions, but requires a final
rounding step to ensure it lies in the set of hard couplings.


We state the generalized $K$-means problem with its constraints explicitly as:
\begin{align*}
    &\min_{\mb{Q} \in \mathbb{R}^{n \times K}} \,\,\left\langle \mb{C}^{\dagger}, 
    \mb{Q} \diag(1/\mb{Q}^{\top} \mb{1}_{n}) \mb{Q}^{\top}
    \right\rangle_{F} \\
      &\quad\textrm{s.t.} \quad \mb{Q}\mb{1}_{K} = \bm{u}_{n}, \mb{Q} \geq \mb{0}_{n \times K}.
\end{align*}
To derive the associated KKT conditions, one introduces the associated dual variables $\bm{\lambda} \in \mathbb{R}^{n}$ and a non-negative matrix $\bm{\Omega} \in \mathbb{R}_{+}^{n \times K}$. From this, we derive a lower bound to the primal by constructing the Lagrangian $L$ as
\begin{align*}
    L(\mb{Q}, \,\,\bm{\Omega},\,\, \bm{\lambda}) = \left\langle \mb{C}^{\dagger}, 
    \mb{Q} \diag(1/\mb{Q}^{\top} \mb{1}_{n}) \mb{Q}^{\top}
    \right\rangle_{F} + \left\langle 
    \bm{\lambda}, \mb{Q}\mb{1}_{K} - \bm{u}_{n}
    \right\rangle - \tr\bm{\Omega}^{\top}\mb{Q}
\end{align*}
Denote $\mb{D}^{-1} = \diag(1/\mb{Q}^{\top} \mb{1}_{n})$. For arbitrary directions $\mb{V} \in \mathbb{R}^{n \times K}$ one has the direction derivative in $\mb{Q}$ for $\mathcal{F}$ is
\begin{align*}
    D \left\langle \mb{C}^{\dagger}, 
    \mb{Q} \mb{D}^{-1} \mb{Q}^{\top}
    \right\rangle_{F} [\mb{V}] = \left\langle \mb{C}^{\dagger}, 
    \mb{V} \mb{D}^{-1} \mb{Q}^{\top}
    \right\rangle_{F} + \left\langle \mb{C}^{\dagger}, 
    \mb{Q} \mb{D}^{-1} \mb{V}^{\top}
    \right\rangle_{F} + \left\langle \mb{C}^{\dagger}, 
    \mb{Q} \mb{D}\diag(1/\mb{Q}^{\top} \mb{1}_{n})[\mb{V}] \mb{Q}^{\top}
    \right\rangle_{F} 
\end{align*}
by symmetry in $\mb{Q}$, this is
\begin{align*}
    &= \left\langle \mb{C}^{\dagger, \top} + \mb{C}^{\dagger},  
    \mb{Q} \mb{D}^{-1} \mb{V}^{\top}
    \right\rangle_{F} + \frac{1}{2}\left\langle (\mb{C}^{\dagger} + \mb{C}^{\dagger, \top}), 
    \mb{Q} D\diag(1/\mb{Q}^{\top} \mb{1}_{n})[\mb{V}] \mb{Q}^{\top}
    \right\rangle_{F} \\
    &= \left\langle \mb{S} \mb{Q} \mb{D}^{-1}, \mb{V} \right\rangle - \frac{1}{2}\left\langle \mb{1}_{n}\diag^{-1}( \mb{Q}^{\top} \mb{D}^{-1} \mb{S} 
    \mb{Q} \mb{D}^{-1} ) \mb{V}
    \right\rangle_{F}
\end{align*}
So that for $\mb{S} = (1/2)(\mb{C}^{\dagger, \top} + \mb{C}^{\dagger})$ the symmetrization of $\mb{C}^{\dagger}$ we find
\begin{align*}
    \nabla_{\mb{Q}} \,L =  \,\mb{S} \mb{Q} \mb{D}^{-1} - \frac{1}{2}\mb{1}_{n}\diag^{-1}( \mb{Q}^{\top} \mb{D}^{-1} \mb{S} 
    \mb{Q} \mb{D}^{-1} )  - \bm{\lambda}\mb{1}_{K}^{\top} - \bm{\Omega}
\end{align*}
Thus, one may summarize the KKT conditions for generalized $K$-means by
\begin{align*}
     \,\mb{S} \mb{Q} \mb{D}^{-1} - \frac{1}{2}\mb{1}_{n} \diag^{-1}( \mb{Q}^{\top} \mb{D}^{-1} \mb{S} 
    \mb{Q} \mb{D}^{-1} )  & - \bm{\lambda}\mb{1}_{K}^{\top} - \bm{\Omega} = \mb{0}_{n \times K}, \\
    \mb{Q}\mb{1}_{K} = \frac{1}{n} \mathbf{1}_{n},& \\
    \mb{Q} \geq \mb{0}_{n \times K},& \\
    \bm{\Omega} \odot \mb{Q} = \mb{0}_{n \times K}.
\end{align*}

Next, let us suppose we consider a mirror-descent on the generalized $K$-means loss with respect to the mirror-map given by the neg-entropy $\psi=-H(p)$, and associated divergence $D_{\psi}(p \mid q ) = \mathrm{KL}(p \mid\mid q)$.

\begin{proposition}\label{prop:GKMS_Derivation}
Suppose $(\gamma_k)_{k \geq 0}$ is a positive sequence of step-sizes for a mirror-descent with respect to the $\mathrm{KL}$ divergence on the variable $\mb{Q}$ in \eqref{def:GenKMeans}. That is, using
the update rule
\begin{align}\label{eq:MD_GKM}
\mb{Q}^{(k+1)} \defeq \argmin _{ \mb{Q} \in \Pi(\bm{u}_{n},\cdot)} \langle \mb{Q}, \nabla_\mb{Q}\mathcal{F}\mid_{\mb{Q}^{(k)}}
    \rangle_{F} + \frac{1}{\gamma_k} \mathrm{KL} ( \mb{Q} \| \mb{Q}^{(k)} ),
\end{align}
where we define the generalized $K$-means loss function as \[\mathcal{F}(\mb{Q})\defeq\left\langle \mb{C}^{\dagger}, 
    \mb{Q} \diag(1/\mb{Q}^{\top} \mb{1}_{n}) \mb{Q}^{\top}
    \right\rangle_{F}.
    \]
    Then the updates are of exponentiated-gradient form with one-sided Sinkhorn projections,
\begin{align}\label{alg:gen_KMeans_supp}
\mb{Q}^{(k+1)}\gets P_{\bm{u}_{n}, \cdot}\left(
\mb{Q}^{(k)} \,
\odot \, \exp\left(-
\gamma_{k} \nabla_{\mb{Q}}\mathcal{F}\mid_{\mb{Q}^{(k)}}
\right)
\,\,\right),
\end{align}
where $P_{\bm{u}_{n}, \cdot} \left( \mb{X}\right) = \diag(\bm{u}_{n}/ \mb{X}\mb{1}_K)\mb{X}$.
\end{proposition}

\begin{proof}
From the first-order stationary condition, we have that
\begin{align*}
    &\,\nabla_{\mb{Q}} \mathcal{F} \mid_{\mb{Q} \,=\,\mb{Q}^{(k)}}   - \bm{\lambda}\mb{1}_K^{\top} - \bm{\Omega} + \frac{1}{\gamma_{k}} \log \left[\frac{\mb{Q}}{\mb{Q}^{(k)}} \right]
    = \mb{0}_{n \times K} \\
    & \mb{Q} = \mb{Q}^{(k)} \odot \exp\left( \gamma_{k} \left( -\nabla_{\mb{Q}} \mathcal{F} \mid_{\mb{Q} \,=\,\mb{Q}^{(k)}}   + \bm{\lambda}\mb{1}_{K}^{\top} + \bm{\Omega} \right) \right)
\end{align*}
For notational simplicity, denote $\bm{K}^{\Omega}= e^{-\gamma_{k} \nabla_{\mb{Q}} \mathcal{F} \mid_{\mb{Q}^{(k)}}} \odot e^{\gamma_{k} \bm{\Omega}} $.
From the constraint, we deduce
\begin{align*}
    & \mb{Q}\mb{1}_{K} =\diag(e^{\gamma_{k}\bm{\lambda}}) \,\,\,\left[ \mb{Q}^{(k)} \odot \bm{K}^{\Omega} \,\, \right]\mb{1}_{K} = \bm{u}_{n} \\
    & \mb{Q}^{(k)} \odot \bm{K}^{\Omega}\mb{1}_K = \diag(e^{-\gamma_{k}\bm{\lambda}})\bm{u}_{n} = \frac{1}{n}\,e^{-\gamma_{k}\bm{\lambda}} 
\end{align*}
So that we find the exponential of the dual variable to be
\begin{align*}
&e^{\gamma_{k}\bm{\lambda}}  = (1/n) \mb{1}_{n} \oslash \left( \mb{Q}^{(k)} \odot \bm{K}^{\Omega}\,\mb{1}_K \right) = \bm{u}_{n} \oslash \left( \mb{Q}^{(k)} \odot \bm{K}^{\Omega}\,\mb{1}_{K} \right)
\end{align*}
Thus, in the identification of $\mb{Q} \in \Pi_{\bm{u}_{n}, \cdot}$ we may evaluate the value of $e^{\gamma_{k}\bm{\lambda}}$ for dual variable $\bm{\lambda} \in \mathbb{R}^{n}$ and find the following update
\begin{align*}
    &\mb{Q} = \mb{Q}^{(k)} \odot \exp\left( \gamma_{k} \left( -\nabla_{\mb{Q}} \mathcal{F} \mid_{\mb{Q} \,=\,\mb{Q}^{(k)}}   + \bm{\lambda}\mb{1}_K^{\top} + \bm{\Omega} \right) \right) \\
    &=\diag(e^{\gamma_{k}\bm{\lambda}}) \mb{Q}^{(k)} \odot \exp\left( \gamma_{k} \left( -\nabla_{\mb{Q}} \mathcal{F} \mid_{\mb{Q} \,=\,\mb{Q}^{(k)}} + \,\,\bm{\Omega} \right) \right) \\
    &= \diag\left( \bm{u}_{n} \,\,/ \left( \mb{Q}^{(k)} \odot \bm{K}^{\Omega} \,\mb{1}_K \right) \right) \left( \mb{Q}^{(k)} \odot \bm{K}^{\Omega} \right)
\end{align*}
Supposing $Q_{ij}^{(0)} > 0$ and supposing $\nabla_{Q}\mathcal{F}$ is bounded, it directly follows that the entries of $\mb{Q}$ are positive and thus from the complementary slackness condition, $\bm{\Omega}_{ij} Q_{ij} = 0$, we find that the dual multiplier $\bm{\Omega}_{ij} = 0$. It follows that $Q_{ij}^{(k)}>0$ for all iterations $k$, and likewise $\bm{\Omega}^{(k)}_{ij} = 0$, so that $\bm{K}^{\Omega} = e^{-\gamma_{k} \nabla_{\mb{Q}} \mathcal{F} \mid_{\mb{Q}^{(k)}}} \odot e^{\gamma_{k} \bm{0}} = e^{-\gamma_{k} \nabla_{\mb{Q}} \mathcal{F} \mid_{\mb{Q}^{(k)}}}$. Thus, for
\begin{align}
    \bm{K}^{(k)} = \mb{Q}^{(k)} \odot e^{-\gamma_{k} \nabla_{\mb{Q}} \mathcal{F} \mid_{\mb{Q}^{(k)}}}
\end{align}
we conclude that the update is given by
\begin{equation}\label{eq:projection_GKMS}
   \mb{Q}^{(k+1)} = \diag\left( \bm{u}_{n} \,\,/ \bm{K}^{(k)}\mb{1}_K \right) \bm{K}^{(k)}
\end{equation}
\end{proof}
Since the mirror-descent \eqref{eq:MD_GKM} is a case of the classical exponentiated gradient and Bregman-projection on the KL-proximal function \cite{peyre2019computational}, one can also derive this by the stationary condition for the kernel
    \begin{align*}
    &0= \nabla_{\mb{Q}} \mathcal{F} \mid_{\mb{Q} \,=\,\mb{Q}^{(k)}}+ \frac{1}{\gamma_{k}} \log\bm{K} \oslash \mb{Q}^{(k)} \\
    &- \frac{1}{\gamma_{k}}\log\bm{K} \oslash \mb{Q}^{(k)} =  \nabla_{\mb{Q}} \mathcal{F} \mid_{\mb{Q} \,=\,\mb{Q}^{(k)}} 
    \end{align*}
    With an update for the positive kernel $\bm{K}$ given by $\bm{K} := \mb{Q}^{(k)} \, \odot \,\exp\left(-
\gamma_{k} \nabla_{\mb{Q}}\mathcal{F}(\mb{Q}^{(k)})
\right)$, The associated Bregman projection with respect to the KL-divergence \cite{peyre2019computational} is therefore $\min_{\mb{Q} \in \Pi_{\bm{u}_{n}, \cdot}} \mathrm{KL}(\mb{Q} \mid \mid \bm{K})$, which coincides with the projection in \eqref{eq:projection_GKMS}. 

To address convergence, let us recall the definition of relative smoothness.
\begin{definition}[Relative smoothness]\label{def:beta_smooth}
    Let $L > 0$ and let $f \in \mathcal{C}^{1}(\mathbb{R}^n, \mathbb{R})$. Additionally, for reference-function $\omega$ let $D_{\omega}$ be its associated distance generating (prox) function. $f$ is then \emph{$L$-smooth relative to $\omega$} if:
    $$
    f(y) \leq f(x) + \langle \nabla f(x), x - y \rangle + L D_{\omega}(y,x)
    $$
\end{definition}

In general, if an objective $f$ is $L$-relatively smooth to $\psi$, the descent lemma applied to mirror-descent guarantees that for $\gamma_{k} \leq 1/L$ one decreases the loss. In particular, one has
\begin{align}
    f(x^{k+1}) \leq f(x^{k}) + \langle \nabla f(x^{k}) , \, x^{k+1} - x^{k} \rangle + L D_{\psi}(x^{k+1},x^{k})
\end{align}
Where, since we have
\begin{align*}
    &x^{k+1} := \argmin_{x} \langle \nabla f(x^{k}), x \rangle + \frac{1}{\gamma_{k}}D_{\psi}(x,x^{k}) \\
    &\langle \nabla f(x^{k}), x^{k+1} \rangle + \frac{1}{\gamma_{k}}D_{\psi}(x^{k+1},x^{k}) \leq \langle \nabla f(x^{k}), x^{k} \rangle + \frac{1}{\gamma_{k}}D_{\psi}(x^{k},x^{k}) = \langle \nabla f(x^{k}), x^{k} \rangle
\end{align*}
The property of $L$-smoothness and taking $\gamma_{k} \leq 1/L$ implies descent, as
\begin{align*}
    &f(x^{k+1}) \leq f(x^{k}) + \langle \nabla f(x^{k}) , \, x^{k+1} \rangle + \frac{1}{\gamma_{k}} D_{\psi}(x^{k+1},x^{k}) - \langle  \nabla f(x^{k}) , \, x^{k} \rangle + \left(L - \frac{1}{\gamma_{k}} \right) D_{\psi}(x^{k+1},x^{k}) \\
    &\leq f(x^{k}) +
    \langle \nabla f(x^{k}), x^{k} \rangle - \langle  \nabla f(x^{k}) , \, x^{k} \rangle + \left(L - \frac{1}{\gamma_{k}} \right) D_{\psi}(x^{k+1},x^{k}) \\
    &= f(x^{k}) + \left(L - \frac{1}{\gamma_{k}} \right) D_{\psi}(x^{k+1},x^{k})
\end{align*}
Where $D_{\psi}(x^{k+1},x^{k})\geq0$ and $\gamma_{k} \leq 1/L$ implies a decrease. Thus, we next aim to show that the proposed mirror-descent, under light regularity conditions, is $L$-smooth and thus guarantees local descent for appropriate choice of step-size $\gamma_{k}$.

\begin{proposition}\label{prop:descent_GKMS}
Suppose that for the neg-entropy mirror-map, $\psi(\mb{Q}) = \sum_{ij} \mb{Q}_{ij} \log \mb{Q}_{ij}$, one considers the loss $\mathcal{F}:=\left\langle \mb{S}, \mb{Q} \mb{D}^{-1} \mb{Q}^{\top}\right\rangle_{F}$ for $\mb{Q}$ in the set $\Pi(\bm{u}_{n}, \cdot)$ and $\mb{D}^{-1}=\diag(1/\mb{Q}^{\top} \mb{1}_{n})$. Moreover, suppose the following floor conditions hold: $\mb{Q}_{ij} \geq \epsilon > 0, \,\,
        \left( \mb{Q}^{\top}\mb{1}_{n} \right)_{k} \geq \delta >0 $. Then, $\mathcal{F}$ is $L$-smooth relative to $\psi$:
\begin{align}
    &\lVert \nabla_\mb{Q}^{(k+1)}-\nabla_\mb{Q}^{(k)} \rVert_{F}
    \leq \left( L_{A} + \frac{L_{A} \sqrt{n}}{\delta} \right)\,\, \lVert \nabla \psi^{(k+1)}-\nabla \psi^{(k)} \rVert_{F}, \\
    & L_{A}:= \left( \frac{\lVert\mb{S} \rVert_{F}}{\delta} + \frac{\sqrt{n}\,\,\lVert\mb{S} \rVert_{F}}{\delta^{2}}\right)
\end{align}
\end{proposition}
\begin{proof}
    Following \cite{Scetbon2021LowRankSF} or \cite{FRLC}, by either strictly enforcing a lower-bound on the entries of $\bm{g}$ or adding an entropic regularization (e.g. a KL-divergence to a fixed marginal, such as $\bm{u}_{r}$, with a sufficiently high penalty $\tau$), one may assume floors of the form
    \begin{align*}
        &\mb{Q}_{ij} \geq \epsilon > 0, \quad
        \left( \mb{Q}^{\top}\mb{1}_{n} \right)_{k} \geq \delta >0 \\
        &\lVert 
        \diag(\mb{Q}^{\top}\mb{1}_{n})^{-1}
        \rVert_{\mathrm{op}} \leq \frac{1}{\delta}
    \end{align*}
    Additionally, note that $\mb{Q}_{ik} \in [0,1/n]$ and $\lVert \mb{Q} \rVert_{F}^{2} = \sum_{ik} \mb{Q}_{ik}^{2} < \sum_{ik} \mb{Q}_{ik} =1$. Now, starting from \begin{align*}\nabla_{\mb{Q}}:=\nabla_{\mb{Q}}^{(A)} + \nabla_{\mb{Q}}^{(B)}= \mb{S} \mb{Q} \mb{D}^{-1} - \frac{1}{2}\mb{1}_{n} \diag^{-1}( \mb{Q}^{\top} \mb{D}^{-1} \mb{S} 
    \mb{Q} \mb{D}^{-1} )
    \end{align*}
    We see
    \begin{align*}
        &\lVert \nabla_{\mb{Q}^{(k+1)}} - \nabla_{\mb{Q}^{(k)}} \rVert_{F} \leq \underbrace{\lVert \mb{S} \mb{Q}^{(k+1)} \mb{D}_{k+1}^{-1} - \mb{S} \mb{Q}^{(k)} \mb{D}_{k}^{-1} \rVert_{F}}_{\text{Term 1}} \\
        &+ \frac{1}{2}\underbrace{\left\lVert \mb{1}_{n} \left( 
        \diag^{-1}( \mb{Q}^{(k+1),\top} \mb{D}_{k+1}^{-1} \mb{S} 
    \mb{Q}^{(k+1)} \mb{D}_{k+1}^{-1} ) -
        \diag^{-1}( \mb{Q}^{(k),\top} \mb{D}_{k}^{-1} \mb{S} 
    \mb{Q}^{(k)} \mb{D}_{k}^{-1} )
        \right) \right\rVert_{F}}_{\text{Term 2}}
    \end{align*}
    From the first term, one finds
    \begin{align*}
        &\lVert \mb{S} \mb{Q}^{(k+1)} \mb{D}_{k+1}^{-1} - \mb{S} \mb{Q}^{(k)} \mb{D}_{k}^{-1} \rVert_{F} \leq \lVert \mb{S} \rVert_{F} \lVert \mb{Q}^{(k+1)} \mb{D}_{k+1}^{-1} - \mb{Q}^{(k)} \mb{D}_{k}^{-1} \rVert_{F} \\
        & \leq \lVert \mb{S} \rVert_{F} \left( \lVert \mb{Q}^{(k+1)} \rVert_{F} \lVert \mb{D}_{k+1}^{-1} - \mb{D}_{k}^{-1} \rVert_{F} +  \lVert \mb{D}_{k}^{-1} \rVert_{F} \lVert \mb{Q}^{(k+1)} - \mb{Q}^{(k)} \rVert_{F} \right)
    \end{align*}
    As one has $\lVert \mb{D}_{k+1}^{-1} - \mb{D}_{k}^{-1} \rVert_{F} = \lVert \mb{D}_{k}^{-1}\mb{D}_{k+1}^{-1} (
    \mb{D}_{k+1} - \mb{D}_{k}) \rVert_{F} \leq \delta^{-2} \lVert \mb{D}_{k+1} - \mb{D}_{k} \rVert_{F}$ and since
    \begin{align}
        \lVert (\mb{Q}^{(k+1)}-\mb{Q}^{(k)})^{\top}\mb{1}_{n} \rVert_{2} \leq \sqrt{n} \lVert \mb{Q}^{(k+1)}-\mb{Q}^{(k)} \rVert_{F} 
    \end{align}
    One collects a bound on the first term of the form
    \begin{align}\label{eq:term1_bound}
        \leq  \left( \frac{\lVert\mb{S} \rVert_{F}}{\delta} + \frac{\sqrt{n}\,\,\lVert\mb{S} \rVert_{F}}{\delta^{2}}\right) \lVert \mb{Q}^{(k+1)}-\mb{Q}^{(k)} \rVert_{F} := L_{A} \lVert \mb{Q}^{(k+1)}-\mb{Q}^{(k)} \rVert_{F}
    \end{align}
For the second, observe that
\begin{align*}
\lVert \bm{1}_{n} \diag^{-1}\mathbf{X} \rVert_{F}^{2} &= \tr(\bm{1}_{n} \mathrm{diag}^{-1}\mathbf{X})^\top(\bm{1}_{n} \mathrm{diag}^{-1}\mathbf{X}) \\
&= n \lVert \mathrm{diag}^{-1}\mathbf{X} \rVert_{2}^{2} \leq n \lVert \mathbf{X} \rVert_{F}^{2}
\end{align*}
So that $\lVert \bm{1}_{n} \diag^{-1}\mathbf{X} \rVert_{F} \leq  \sqrt{n} \lVert \mathbf{X} \rVert_{F}$. Thus, we find the second term is bounded by
\begin{align*}
    &\leq \frac{\sqrt{n}}{2}\left\lVert  \mb{Q}^{(k+1),\top} \mb{D}_{k+1}^{-1} \mb{S} 
    \mb{Q}^{(k+1)} \mb{D}_{k+1}^{-1}  -
        \mb{Q}^{(k),\top} \mb{D}_{k}^{-1} \mb{S} 
    \mb{Q}^{(k)} \mb{D}_{k}^{-1} 
    \right\rVert_{F} \\
    &=\frac{\sqrt{n}}{2}\left\lVert  \mb{Q}^{(k+1),\top} \mb{D}_{k+1}^{-1} \nabla_{\mb{Q}}^{(k+1,A)}  -
        \mb{Q}^{(k),\top} \mb{D}_{k}^{-1} \nabla_{\mb{Q}}^{(k,A)}
    \right\rVert_{F} \\
    & \leq \frac{\sqrt{n}}{2} \lVert \mb{Q}^{(k+1)} \rVert_{F} \lVert \mb{D}_{k+1}^{-1} \rVert_{F} \lVert \nabla_{\mb{Q}}^{(k+1,A)} - \nabla_{\mb{Q}}^{(k,A)} \rVert_{F} + \frac{\sqrt{n}}{2} \lVert \nabla_{\mb{Q}}^{(k,A)} \rVert_{F} \lVert \mb{Q}^{(k+1),\top} \mb{D}_{k+1}^{-1} - \mb{Q}^{(k),\top} \mb{D}_{k}^{-1} \rVert_{F}
\end{align*}
Now, since we have already quantified the difference $\lVert \nabla_{\mb{Q}}^{(k+1,A)} - \nabla_{\mb{Q}}^{(k,A)} \rVert_{F}$ with smoothness constant $L_{A}$ in \eqref{eq:term1_bound}, and also quantified $\lVert \mb{Q}^{(k+1),\top} \mb{D}_{k+1}^{-1} - \mb{Q}^{(k),\top} \mb{D}_{k}^{-1} \rVert_{F}$, we simply invoke both bounds from above to conclude the bound on the second term as
\begin{align*}
    &\leq \frac{L_{A}\sqrt{n} }{2 \delta} \lVert \mb{Q}^{(k+1)}-\mb{Q}^{(k)} \rVert_{F}  + \frac{\sqrt{n}}{2} \frac{1}{\delta} \left( \frac{\lVert \mb{S} \rVert_{F}}{\delta} + \frac{\sqrt{n} \lVert \mb{S} \rVert_{F}}{\delta^{2}} \right) \lVert \mb{Q}^{(k+1)}-\mb{Q}^{(k)} \rVert_{F} \\
    &= \left( \frac{L_{A}\sqrt{n}}{2\delta} +\frac{L_{A} \sqrt{n}}{2\delta} \right) \lVert \mb{Q}^{(k+1)}-\mb{Q}^{(k)} \rVert_{F} = \frac{L_{A} \sqrt{n}}{\delta}\lVert \mb{Q}^{(k+1)}-\mb{Q}^{(k)} \rVert_{F}
\end{align*}
Thus, we find the objective to be $L$-smooth with constant given in terms of $L_{A}$ \eqref{eq:term1_bound}
\begin{align*}
    \leq L\,\, \lVert \mb{Q}^{(k+1)}-\mb{Q}^{(k)} \rVert_{F}, \quad L = \left( L_{A} + \frac{L_{A} \sqrt{n}}{\delta} \right)
\end{align*}
Lastly, for the entropy mirror-map $\psi$ observe that for $\nabla \psi (x) = \log x$ one has for $\xi \in [\epsilon, 1]$ by the mean value theorem that $\log'(\xi) |u-v|=\xi^{-1}|u-v| = |\log u - \log v|$, so that following \cite{Scetbon2021LowRankSF} one concludes relative smoothness in $\psi$ via the upper bound
\begin{align*}
\leq L\,\, \lVert \nabla \psi^{(k+1)}-\nabla \psi^{(k)} \rVert_{F}&\qedhere
\end{align*}
\end{proof}

\begin{corollary}[Guaranteed Descent]
    By Proposition~\ref{prop:descent_GKMS} one can ensure that $\mathcal{F}$ is smooth relative to the entropy mirror-map $\psi$ with constant $L$ given in Proposition~\ref{prop:descent_GKMS}. For $\gamma_{k} \leq 1/L$, this guarantees descent on the objective and ensures the initialization guarantees of Theorem~\ref{thrm:K_means_init_guarantee} are upper bounds on the final solution cost.
\end{corollary}

\subsubsection{Semidefinite Programming}
\label{subsec:alg_sdp}

Here we present an algorithm for generalized $K$-means via
semidefinite programming. The basic idea is that the semidefinite programming approaches
for $K$-means \citep{peng2005new, peng2007approximating, zhuang2022sketch, zhuang2023statistically} apply immediately to the generalized $K$-means
problem. First, by analyzing the argument in 
\citep{peng2005new, peng2007approximating}
for constructing an equivalent form of $K$-means, one observes
that the generalized $K$-means problem \eqref{eq:GenKMeans} is
equivalent to:
\begin{equation}
\label{eq:symmetric_kcut_reformulation}
    \min_{\mb{P}\geq0}\, \biggl\{\left\langle\mb{P},\mb{C}\right\rangle_F : \tr(\mb{P})=K,\,\, \mb{P}\mb{1}_K=\mb{1}_n, \,\, \mb{P}^2=\mb{P}, \,\, \mb{P}=\mb{P}^\top \biggr\}.
\end{equation}
Replacing the non-convex constraint $\mb{P}^2 = \mb{P}$ with its relaxation $\mb{P} \succeq 0$,
yields the semidefinite relaxation of generalized $K$-means problem \eqref{eq:GenKMeans},
\begin{equation}
\label{eq:symmetric_kcut_semidefinite}
    \min_{\mb{P}\geq0}\, 
    \biggl\{\left\langle\mb{P},\mb{C}\right\rangle_F : \tr(\mb{P})=K,\,\, \mb{P}\mathbbm{1}_K=\mathbbm{1}_n,\,\, \mb{P} \succeq 0 \biggr\}.
\end{equation}
The only difference between the reformulation of generalized $K$-means 
\eqref{eq:symmetric_kcut_reformulation} and the reformulation of $K$-means studied in \citep{peng2005new, peng2007approximating}
is the structure of the cost matrix $\mb{C}$. The advantages of the semidefinite programming approach compared
to \texttt{GKMS}
is that it provides higher quality solutions, does not
depend on the initialization parameters, and provides a lower bound
on the optimal cost. The disadvantage
is the computational cost required to solve large 
semidefinite programming problems. Mildly alleviating the computational
burden, we apply recent approaches from \cite{zhuang2023statistically} for solving the semidefinite
programming problem \eqref{eq:symmetric_kcut_semidefinite}.

\subsection{Exact Reductions of Generalized $K$-Means by Class of Cost $\mb{C}$}\label{sec:exact_red_GenKMeans}

\subsubsection{Nearly negative semidefinite costs}

When $\mb{C}$ is negative semidefinite, the generalized $K$-means problem
\eqref{eq:symmetric_kcut_semidefinite} exactly coincides with the $K$-means problem. 
In these cases, approximation algorithms for $K$-means, such as established $(1+\epsilon)$ approximations \cite{1peps} and poly-time $\log K$ approximations (e.g. \texttt{k-means++} \cite{kmpp}), directly transfer to the low-rank OT setting. However, by definition, such costs express symmetric distances between a dataset and itself and are not relevant to optimal transport between distinct measures. 

Interestingly, direct reduction of generalized $K$-means to to $K$-means holds for a more general class of asymmetric distances which may express costs between distinct datasets. In Proposition~\ref{prop:cost_for_KMeans}, we show such a strong condition: it is sufficient for the symmetrization of any cost $\mb{C}$, $\mathrm{Sym}\mb{C}$, to be conditionally negative semi-definite.

\begin{proposition}
\label{prop:cost_for_KMeans}
    Suppose we are given a cost matrix $\mb{C} \in \mathbb{R}^{n \times n}$ where the symmetrization of $\mb{C}$, $\mathrm{Sym}(\mb{C})\,:= \, \mb{C}^{\top} + \, \mb{C}$ is conditionally negative-semidefinite so that $\mathrm{Sym}\mb{C} \preceq 0$ on $\mb{1}_{n}^{\perp}$. Denote the double-centering $J = \mathbbm{1} - \frac{1}{n} \mb{1}_{n}\mb{1}_{n}^{\top}$ and p.s.d. kernel $\mb{K} := - (1/2)\, J \,\mathrm{Sym}\mb{C} \,J \succeq 0$. Then Problem~\ref{def:GenKMeans} reduces to kernel $k$-means \cite{dhillon2004kernel} on $\mb{K}$:
    \begin{align}
        &\min_{\substack{\mb{Q} \in \Pi_{\bullet}(\bm{u}_{n},\,\bm{g}),\\\,\bm{g} \in \Delta_{r}}}  \,\, \langle \mb{Q}\,\mathrm{diag}(1/\bm{g})\mb{Q}^{\top}\,, \mb{C} \rangle_{F}  \,\equiv \max_{\mb{Q}\in \{ 0 , 1 \}^{n \times r}} \tr \mb{D}^{-1/2} \mb{Q}^{\top} \mb{K} \mb{Q} \mb{D}^{-1/2}
    \end{align}
    $\mb{D} := \mathrm{diag}(\bm{g})$ denotes the diagonal matrix of cluster sizes and $\mb{Q}$ the matrix of assignments.
\end{proposition}

Cost matrices induced by kernels, such as the squared Euclidean distance, are classically characterized by being conditionally negative semidefinite \cite{CND1, CND2}. For a satisfying cost $\mb{C}$, Proposition
implies that \ref{def:GenKMeans} is equivalent to $K$-means with Gram matrix $\mb{K} = -(1/2)\, J\,\mathrm{Sym}\mb{C} \, J$. This is a stronger statement than requiring $\mb{C} \preceq 0$. Observe that the Monge-conjugated matrix, $\mathrm{Sym}\mb{C}^{\dagger}$ turns an asymmetric cost into a symmetric bilinear form on $(i,j)$. Moreover, as $\mathrm{Sym}\mb{C}$ plays the role of a distance matrix in the conversion to $\mb{K}$, we offer it an appropriate name

\begin{definition}[Monge Cross-Distance Matrix]
    Let $\mb{C}^{\dagger}=\mb{C}\mb{P}_{\sigma^{\star}}^\top$ for $\mb{P}_{\sigma^{\star}}$ the optimal Monge permutation. Denote its symmetrization by $\mathrm{Sym}(\mb{C}^{\dagger}) = \mb{C}\mb{P}_{\sigma^{\star}}^\top + \mb{P}_{\sigma^{\star}}\mb{C}^\top$. In the $\mb{1}_{n}^{\perp}$ subspace, each element may be expressed as the cross-difference
    \begin{align}
        \mb{M}^{\dagger}_{ij} = \langle x_{i} - x_{j}, T(x_{i}) - T(x_{j}) \rangle
    \end{align}
    Thus, we refer to $\mb{M}^{\dagger}$ as the Cross-Distance Matrix induced by the Monge map.
\end{definition}
 
Proposition~\ref{prop:cost_for_KMeans} implies that the reduction to $K$-means holds if and only if the bilinear forms of the Monge gram matrix admit an inner product in a Hilbert space $\mathcal{H}$. In other words, if there exists a function $\psi$ so $\langle x , T(y) \rangle:=\langle \psi(x), \psi(y) \rangle$. For clustering on any symmetric cost $\mb{C}$, one has that the Monge map is the identity $T=\mb{I}$, so that $\langle x_{i} - x_{j}, T(x_{i}) - T(x_{j}) \rangle $ immediately reduces to a EDM $\lVert x_{i} - x_{j} \rVert_{2}^{2}$. Notably, this also holds for a more general class of distributions without identity Monge maps -- multivariate Gaussians in Bures-Wasserstein space $\mathrm{BW}(\mathbb{R}^{d})$ \cite{chewi2024statistical}. These automatically admit CND cost matrices following Monge-conjugation for the squared Euclidean distance $\lVert \cdot \rVert_{2}^{2}$.

\begin{remark}
    $\langle x , T(y) \rangle:=\langle \psi(x), \psi(y) \rangle$ holds universally for transport maps between any two multivariate Gaussians \cite{peyre2019computational}. Let $\rho_{1} = \mathcal{N}(\bm{\mu}_{1}, \mb{\Sigma}_{1})$ and $\rho_{2} = \mathcal{N}(\bm{\mu}_{2}, \mb{\Sigma}_{2})$. The transport map $T$ such that $T_{\sharp} \rho_{1} = \rho_{2}$ is given by the affine transformation $T(x) = \mb{A}x + \bm{b}$ with $\mb{A} = \mb{\Sigma}_{1}^{-1/2} \left(\mb{\Sigma}_{1}^{1/2}
    \mb{\Sigma}_{2}
    \mb{\Sigma}_{1}^{1/2}
    \right) \mb{\Sigma}_{1}^{-1/2}
    \succeq 0$ and $\bm{b}=\bm{\mu}_{2} - \mb{A}\bm{\mu}_{1}$. Thus, for $\psi := \sqrt{\mb{A}}$
    \begin{align*}
    \langle x_{i} - x_{j}, \mb{A}x_{i} + \bm{b} - (\mb{A}x_{j} + \bm{b}) \rangle = \lVert \sqrt{\mb{A}} (x_{i} - x_{j}) \rVert_{2}^{2}
    \end{align*} 
\end{remark}
In general, the conjugated cost $\mb{M}^{\dagger}$ shifts $\mb{C}$ to be nearer to a clustering distance matrix after symmetrization: the diagonal entries are zero $\mb{M}^{\dagger}_{jj} = 0$, for squared Euclidean cost \cite{brenier1991polar} the entries $\langle x_{i} - x_{j}, T(x_{i}) - T(x_{j}) \rangle \geq 0$, and $\mb{M}^{\dagger}$ reduces to a matrix of kernel-distances on $x_{i} - x_{j}$ whenever $\mathrm{Sym}(\mb{C}^{\dagger})$ is CND. Moreover, for squared Euclidean cost one has $T= \nabla\varphi$ for a convex potential $\varphi \in \mb{cvx}(\mathbb{R}^{d})$ \cite{brenier1991polar}. Thus, to second-order, all entries may be expressed on $\mb{1}_{n}^{\perp}$ as PSD forms $\langle x_{i} - x_{j}, T(x_{i}) - T(x_{j}) \rangle \sim \langle x_{i} - x_{j}, \nabla^{2}\varphi(x_{j})(x_{i} - x_{j}) \rangle$ for $\nabla^{2}\varphi(x_{j}) \succeq 0$.

While $K$-means reduces to a special case of low-rank optimal transport where $\mb{Q} = \mb{R}$, as has been previously shown \cite{scetbon2022lowrank}, the other direction is significantly less obvious: it often appears that one can only gain by taking $\mb{Q} \neq \mb{R}$ and optimizing over a larger space of solutions when $\mb{C}$ is an asymmetric cost with respect to a pair of \emph{distinct} datasets ${X},{Y}$. We note that when the conditions of Proposition~\ref{prop:cost_for_KMeans} hold, generalized $k$-means exactly reduces to $K$-means, so that step (ii) inherits its existing algorithmic guarantees. In particular, suppose $\mb{C} \in \mathbb{R}^{n \times n }$ satisfies Proposition~\ref{prop:cost_for_KMeans} and one may also solve $K$-means to $(1+\epsilon)$ using algorithm $\mathcal{A}$. For the Gram-matrix $\mb{K} =-(1/2)J\,\mathrm{Sym}\mb{C} \, J$ one may yield the eigen-decomposition $\mb{K} = \mb{U}\mb{\Lambda}\mb{U}^{\top}$ and compute point $\mb{Z} = \mb{U}\mb{\Lambda}^{1/2}$. Then, given a solution to $K$-means on $\mb{Z}$, $\Bar{\mb{Q}} := \mathcal{A}(\mb{Z})$, one automatically inherits $(1+\epsilon)$-approximation of generalized $K$-means by the exact reduction. We detail the algorithm for this special case in Algorithm~\ref{alg:exact_GKMS_reduction} below.

\begin{algorithm}[]
   \caption{}
  \label{alg:exact_GKMS_reduction}
\begin{algorithmic}[1]
   \STATE \textbf{input:} Cost matrix $\mb{C}$ and rank $K$.
   \STATE Symmetrize the Monge-conjugated cost $\mathrm{Sym}(\tilde{\mb{C}})=\mb{C}\mb{P}_{\sigma^{\star}}^\top +\mb{P}_{\sigma^{\star}}\mb{C}^\top$
   \STATE Grammize as $G =-(1/2)\,\,J\,\mathrm{Sym}(\tilde{\mb{C}})J$ for double-centering $J = \mathbbm{1}_{n} - \frac{1}{n} \mb{1}_{n}\mb{1}_{n}^\top$
   \STATE Yield $Z$ from eigen-decomposition of $G = ZZ^{\top}$
   \STATE Run $K$-Means on $Z$ to yield $\mb{Q}$
   \STATE Output the pair $(\mb{Q}, \mb{P}_{\sigma^{\star}}^\top)$
\end{algorithmic}
\end{algorithm}


Thus, for this class of cost, Algorithm~\ref{alg:monge_KMeans} guarantees \emph{optimal} solutions to generalized $K$-means by reduction to optimal solvers for $K$-means.


Observe two valuable invariants of the optimization problem \eqref{def:GenKMeans}: first we have an affine invariance, naturally characterized by the optimal transport constraints; second, the symmetry of the coupling optimized introduces an invariance to asymmetric components of the cost itself, so that the optimization \eqref{def:GenKMeans} is equivalent to one on the symmetrization of the cost.

\begin{lemma}[Invariances of Generalized $K$-Means.]\label{lemma:affine_invar}
    Suppose we are given a cost matrix $\mb{C} \in \mathbb{R}^{n \times n}$. Then the generalized $K$-means problem\begin{align}\label{eq:LR_supp}
    \min_{\substack{\mb{Q} \in \Pi_{\bullet}(\bm{u}_{n},\,\bm{g}),\\\,\bm{g} \in \Delta_{K}}}  \,\, \langle \mb{Q}\,\mathrm{diag}(1/\bm{g})\mb{Q}^{\top}\,, \mb{C} \rangle_{F}
    \end{align}
    Exhibits the following invariances:
    \begin{enumerate}
        \item Invariance to asymmetric components\begin{align}\argmin_{\substack{\mb{Q} \in \Pi_{\bullet}(\bm{u}_{n},\,\bm{g}),\\\,\bm{g} \in \Delta_{K}}}  \,\, \langle \mb{Q}\,\mathrm{diag}(1/\bm{g})\mb{Q}^{\top}\,, \mb{C} \,  \rangle_{F}=\argmin_{\substack{\mb{Q} \in \Pi_{\bullet}(\bm{u}_{n},\,\bm{g}),\\\,\bm{g} \in \Delta_{K}}}  \,\, \langle \mb{Q}\,\mathrm{diag}(1/\bm{g})\mb{Q}^{\top}\,, \mb{S} \,  \rangle_{F} 
    \end{align}
        Where $\mb{C} = \mb{A} + \mb{S}$ for its symmetric component $\mb{S}\defeq(1/2)(\mb{C}+\mb{C}^{\top}) \in \mathbb{S}^{n}$ and its skew-symmetric component $\mb{A}:=(1/2)(\mb{C} -\mb{C}^{\top}) \in \mathbb{A}^{n}$.
        \item Invariance to affine offsets $\bm{f}\mb{1}_{n}^{\top} + \mb{1}_{n}\bm{h}^{\top}$ and shifts $\gamma\mb{1}_{n}\mb{1}_{n}^{\top}$
        \begin{align}&\min_{\substack{\mb{Q} \in \Pi_{\bullet}(\bm{u}_{n},\,\bm{g}),\\\,\bm{g} \in \Delta_{K}}}  \,\, \langle \mb{Q}\,\mathrm{diag}(1/\bm{g})\mb{Q}^{\top}\,, \mb{\Lambda} + \bm{f}\mb{1}_{n}^{\top} + \mb{1}_{n}\bm{h}^{\top} + \gamma\mb{1}_{n}\mb{1}_{n}^{\top} \rangle_{F}\label{eq:invar_affine}\\
        &= \min_{\substack{\mb{Q} \in \Pi_{\bullet}(\bm{u}_{n},\,\bm{g}),\\\,\bm{g} \in \Delta_{K}}}  \,\, \langle \mb{Q}\,\mathrm{diag}(1/\bm{g})\mb{Q}^{\top}\,, \mb{\Lambda}  \rangle_{F} + \bm{f}^{\top} \bm{u}_{n} + \bm{u}_{n}^{\top}\bm{h} + \gamma
        \end{align}
    \end{enumerate}
\end{lemma}

\begin{proof}
    Observe that symmetry of the matrix $\mb{Q}\,\mathrm{diag}(1/\bm{g})\mb{Q}^{\top}$ implies the objective \eqref{eq:LR_supp} is equivalent to the objective on $\mb{C}^{\top}$
   \begin{align}
       \langle \mb{Q}\,\mathrm{diag}(1/\bm{g})\mb{Q}^{\top}\,, \mb{C} \rangle_{F} = \tr  \mb{Q}\,\mathrm{diag}(1/\bm{g})\mb{Q}^{\top} \mb{C} = \langle \mb{Q}\,\mathrm{diag}(1/\bm{g})\mb{Q}^{\top} , \mb{C}^{\top} \rangle
   \end{align}
   This directly implies (1). For (2), if $\mb{C} = \mb{\Lambda} +  \mb{X}\mb{Y}^{\top}$. Then, we have that
    \begin{align*}
        &\langle \mb{Q}\,\mathrm{diag}(1/\bm{g})\mb{Q}^{\top}\,, \mb{C} \rangle_{F} = \langle \mb{Q}\,\mathrm{diag}(1/\bm{g})\mb{Q}^{\top}\,, \mb{\Lambda} \rangle + 
        \tr \mb{X}^{\top}  \mb{Q}\,\mathrm{diag}(1/\bm{g})\mb{Q}^{\top}\,\mb{Y}
    \end{align*}
    Thus, the constraints on $\mb{Q}$ imply for each case of \eqref{eq:invar_affine} that
    \begin{align*}
        &\gamma\tr \mb{1}_{n}^{\top}  \mb{Q}\,\mathrm{diag}(1/\mb{Q}^{\top}\,\mb{1}_{n})\mb{Q}^{\top}\,\mb{1}_{n} = \gamma \mb{1}_{n}^{\top}  \mb{Q}\,\mb{1}_{n} = \gamma \\
        &\tr \bm{f}^{\top}  \mb{Q}\,\mathrm{diag}(1/\bm{g})\mb{Q}^{\top}\,\mb{1}_{n} = \bm{f}^{\top}  \mb{Q}\,\mathrm{diag}(1/\bm{g})\bm{g } =  \bm{f}^{\top}  \mb{Q}\,\mathrm{diag}(1/\bm{g})\bm{g } \, =   \bm{f}^{\top}  \mb{Q}\,\mb{1}_K = \bm{f}^{\top} \bm{u}_{n} \\
        & (\mb{Q} \mb{1}_{n})^{\top}\,\mathrm{diag}(1/\bm{g})\mb{Q}^{\top}\,\bm{h} = \,\,\,(\mathrm{diag}(1/\bm{g})\bm{g})^{\top}\mb{Q}^{\top}\,\bm{h} =  \,\,\,(\mb{Q}\mb{1}_K)^{\top}\,\bm{h} =  \bm{u}_{n}^{\top}\bm{h}
    \end{align*}
\end{proof}

\begin{proposition}[Reduction to $K$-Means for costs with conditionally negative semi-definite symmetrization.]\label{prop:cost_for_KMeans}
    Suppose we are given a cost matrix $\mb{C} \in \mathbb{R}^{n \times n}$. Then generalized $K$-means reduces to $K$-means if $\mathrm{Sym}\,\mb{C}\,\defeq (1/2)\, (\mb{C}^{\top} + \, \mb{C})$ is conditionally negative-semidefinite (CND) so that $\mathrm{Sym}\mb{C} \preceq 0$ on $\mb{1}_{n}^{\perp} = \{\xi: \langle \xi, \mb{1}_{n} \rangle = 0\}$.
\end{proposition}
\begin{proof}
    Owing to invariance of the objective to $\mathrm{Skew}(\mb{C})$ we may replace the minimization in \eqref{eq:LR_supp} with a minimization over the symmetric component of $\mb{C}$, $(1/2)(\mb{C}+\mb{C}^{\top}) = \,\mb{S}$:
   \begin{align}
    &\argmin_{\substack{\mb{Q} \in \Pi_{\bullet}(\bm{u}_{n},\,\bm{g}),\\\,\bm{g} \in \Delta_{K}}}  \,\, \langle \mb{Q}\,\mathrm{diag}(1/\bm{g})\mb{Q}^{\top}\,, \mb{C} \rangle_{F} =\argmin_{\substack{\mb{Q} \in \Pi_{\bullet}(\bm{u}_{n},\,\bm{g}),\\\,\bm{g} \in \Delta_{K}}}  \,\, \langle \mb{Q}\,\mathrm{diag}(1/\bm{g})\mb{Q}^{\top}\,, \mb{S} \rangle_{F}
    \end{align}
    Observe that the solution of the objective is invariant to outer products between constant $\bm{f},\bm{h}\in \mathbb{R}^{n}$ with the one vector $\mb{1}_{n}$, i.e. components of the form $\bm{f}\mb{1}_{n}^{\top} + \mb{1}_{n}\bm{h}^{\top}$. Denote the double-centering $J = \mathbbm{1}_{n} - (1/n) \mb{1}_{n}\mb{1}_{n}^{\top}$. If $\mb{S}$ is conditionally negative semidefinite (CND), then applying this affine invariance implies the objective is equivalent to
    \begin{align}
        \min_{\substack{\mb{Q} \in \Pi_{\bullet}(\bm{u}_{n},\,\bm{g}),\\\,\bm{g} \in \Delta_{K}}}  \,\, \langle \mb{Q}\,\mathrm{diag}(1/\bm{g})\mb{Q}^{\top}\,, J\mb{S}J \rangle_{F} 
    \end{align}
    Thus, for $J 
    \mb{S}J \preceq 0$ we exhibit a positive semidefinite kernel matrix $\mb{K} = - (1/2)\, J \mb{S} J \succeq 0$ and have
    \begin{align}
        &\min_{\substack{\mb{Q} \in \Pi_{\bullet}(\bm{u}_{n},\,\bm{g}),\\\,\bm{g} \in \Delta_{K}}}  \,\, \langle \mb{Q}\,\mathrm{diag}(1/\bm{g})\mb{Q}^{\top}\,, \mb{C} \rangle_{F}  = -2\,\, \langle \mb{Q}\,\mathrm{diag}(1/\bm{g})\mb{Q}^{\top}\,,  \mb{K}\rangle_{F} \\&\equiv \max_{\mb{Q}\in \{ 0 , 1 \}^{n \times K}} \tr \mb{D}^{-1/2} \mb{Q}^{\top} \mb{K} \mb{Q} \mb{D}^{-1/2}
    \end{align}
    Where $\mb{D} \defeq \mathrm{diag}(\bm{g})$ denotes the conventional diagonal matrix of cluster sizes and $\mb{Q}$ the matrix of assignments. Thus, if $\mb{S}=\mathrm{Sym}\mb{C}$ is conditionally negative semidefinite, up to constants Problem~\eqref{def:GenKMeans} reduces to kernel $k$-means \citep{kernel_Kmeans}.
\end{proof}

\newpage
\section{Experimental Details}

\subsection{Implementation Details}
\label{appendix:implementation}

For the synthetic experiments we inferred the Monge map $\mb{P}_{\sigma^*}$
by applying the Sinkhorn algorithm implemented in \texttt{ott-jax} with the
entropy regularization parameter $\epsilon=10^{-5}$ and
a maximum iteration count of 10{,}000. For the
real data experiments, we inferred the Monge map $\mb{P}_{\sigma^*}$
using \texttt{HiRef} \citep{halmos2025hierarchical}, and used a low-rank version of \texttt{GKMS} which uses a factorization of the cost $\mb{C}=\mb{A B}^{\top}$ for scaling. The remaining 
implementation details
are consistent across the synthetic and real data experiments.

For the \texttt{GKMS}
algorithm, we used a \texttt{JAX} implementation of the \texttt{GKMS}
algorithm with step size $\gamma_k = 2$
for a fixed number $250$ of iterations. To construct an initial solution, we first applied 
the $K$-means algorithm implemented in \texttt{scikit-learn} 
on $X$ and $Y$ to obtain clustering matrices $\mb{Q}_X$ and $\mb{Q}_Y$.
Then, using the Monge registered initialization procedure in Algorithm \ref{alg:gen_kmeans_initialization}, we took the best of the two solutions
$\mb{Q}_X$ and $\mb{P}_{\sigma^*}\mb{Q}_Y$ as $\mb{Q}$.
Next, we performed a centering step by
setting $\mb{Q}^{(0)} = \lambda \mb{Q}+(1-\lambda)\mb{Q}'$
where $\mb{Q}'$ is a random matrix in $\Pi(\mb{u}_n,\cdot)$ 
generated from the initialization procedure in \cite{Scetbon2021LowRankSF}
with $\lambda = \frac{1}{2}$.
Finally, we ran \texttt{GKMS}
on the registered cost matrix $\tilde{\mb{C}}=\mb{C}\mb{P}_{\sigma^*}^\top$
with $\mb{Q}^{(0)}$ as the initial solution.

For the synthetic stochastic block model (SBM) example we applied the semidefinite
programming formulation of the generalized $K$-means problem
described in Appendix \ref{subsec:alg_sdp} with the solver from
\cite{zhuang2023statistically} to initialize $\mb{Q}^{(0)}$ prior
to running \texttt{GKMS}.

\subsection{Synthetic Experiments}
\label{appendix:synthetic}

We constructed three synthetic datasets to evaluate low-rank OT methods and evaluated
\ourmethod against three existing low-rank OT methods:
\LOT \citep{Scetbon2021LowRankSF}, \FRLC \citep{FRLC}, and \LIN \citep{lin2021making}. The three synthetic datasets are referred to as 2-Moons and 8-Gaussians (2M-8G) \citep{tong2023improving, Scetbon2021LowRankSF}, shifted Gaussians (SG), and the stochastic block model (SBM). 
The 2M-8G dataset contained three 
instances at noise levels $\sigma^2 \in \{0.1, 0.25, 0.5\}$, the SG dataset contained three instances at noise levels $\sigma^2 \in \{0.1, 0.2, 0.3\}$, and the
SBM dataset contained a single instance. 
Each instance contained 
$n = m = 5000$ points and methods were evaluated across
a range of ranks $K \in \{50, 75, \ldots, 250\}$ and $K \in \{10, \ldots, 100\}$
with five random seeds $s \in\{1,2,3,4,5\}$. In total,
each algorithm was ran on $64$ instances for $5$ random
seeds. 
As each dataset was constructed with $n = m = 5000$ datapoints, the resulting cost matrix $\mb{C} \in \mathbb{R}^{5000\times5000}$.

\textbf{2-Moons and 8-Gaussians (2M-8G).} In this experiment \cite{tong2023improving}, we generated two datasets $X,Y \subset \mathbb{R}^{2}$ representing two spirals ($X$) and 8 isotropic Gaussians ($Y$). In particular, we used the function \texttt{generate\_moons} from the package \texttt{torchdyn.datasets} to generate the two interleaving moons as the first dataset. These are defined as semi-circles with angles $\theta_{1} \sim \mathrm{Unif}(0, \pi ), \, \theta_{2} \sim \mathrm{Unif}(0, \pi )$ and $\begin{pmatrix}
    r \cos \theta_{1} & r \sin\theta_{1}
\end{pmatrix} - \bm{c}$ and $\begin{pmatrix}
    r \cos \theta_{2} & -r \sin\theta_{2}
\end{pmatrix} + \bm{c}$ for constant offset $\bm{c}$. We add isotropic Gaussian noise with variance $0.5$. As in \cite{tong2023improving}, one scales all points with $\Tilde{Y}=aY+b$ for $a=3, b=\begin{pmatrix}
    -1 & -1
\end{pmatrix}$ to overlap visually with the 8 Gaussians. For given variance $\sigma^{2}=1.0$, we generated $k\in[8]$ isotropic Gaussian clusters $\mathcal{N}(\bm{\mu}_{k}, \sigma^{2} \mb{I}_{2})$ with means on the unit circle $S^{2}$, given by
\begin{align*}
    \begin{pmatrix}
        \bm{\mu}_{1} \\
        \cdots \\
        \bm{\mu}_{8}
    \end{pmatrix} =\begin{cases}
    (1,0),\,\\
    (-1,0),\,\\
    (0,1),\, \\
    (0,-1),\,\\
    (\tfrac{1}{\sqrt{2}},\tfrac{1}{\sqrt{2}}),\,\\
    (\tfrac{1}{\sqrt{2}},-\tfrac{1}{\sqrt{2}}),\,\\
    (-\tfrac{1}{\sqrt{2}},\tfrac{1}{\sqrt{2}}),\,\\
    (-\tfrac{1}{\sqrt{2}},-\tfrac{1}{\sqrt{2}})
    \end{cases}
\end{align*}
The 2-moons constitutes a simple non-linear manifold and the 8-Gaussians constitutes a simple dataset with cluster structure.

\textbf{Shifted Gaussians (SG).} 
To construct the SG synthetic datasets we placed 
$K = 250$ Gaussian distributions with means $\bm{\mu}_1,\ldots,\bm{\mu}_k\in\mathbb{R}^{K}$
at the basis vectors $e_1,\ldots, e_K \in \mathbb{R}^K$. Similarly, we constructed
another set of means $\bm{\mu}_1',\ldots,\bm{\mu}_k'$ by perturbing the means $\bm{\mu}_i'=\bm{\mu}_i + \epsilon_i$ with $\epsilon_i \sim \mathcal{N}(0, \frac{0.1}{\sqrt{n}}\mb{I}_K)$.
Then, we randomly sampled groups
of size $m_1,\ldots, m_K$ with $\sum_{k=1}^Km_k=n$ by randomly sampling a partition
of $n$ of size $K$. Then, for both datasets $X$ and $Y$, we assigned cluster $1$ to the first 
$m_1$ points, cluster $2$ to the next $m_2$ points, ..., and cluster $K$ to the final $m_K$ points.
For points in cluster $k$ in dataset $X$, we sample $m_k$ points from 
$\mathcal{N}(\bm{\mu}_k, \frac{\sigma^2}{\sqrt{n}}\mb{I}_K)$. Similarly,
for points in cluster $k$ in dataset $Y$, we sample $m_k$ points from 
$\mathcal{N}(\bm{\mu}_k', \frac{\sigma^2}{\sqrt{n}}\mb{I}_K)$. 

To construct the cost matrix, we take $\mb{C}_{ij}=\lVert x_i - y_j\rVert^2_2$. 
We construct three instances using different noise values $\sigma^2\in\{0.1, 0.2, 0.3\}.$

\textbf{Stochastic Block Model (SBM).}
To construct the SBM instance, we generated a graph $G = (V, E)$ from a
stochastic block model using within cluster probability $p=0.5$ and
between cluster probability $q = 0.25$ over 
$K = 100$ clusters of fixed size $m = 50$. Edge weights $w_e$ were
generated by randomly sampling weights from $\text{Unif}(1.0, 2.0)$.
The cost matrix $\mb{C}_{ij} = d_G(i,j)$ was taken as 
the shortest path distance between vertices $i$ and $j$ in $G$ with 
the weight function $w$.

\subsection{CIFAR10}\label{sec:cifar}

We follow the protocol of \cite{zhuang2023statistically} in this experiment by comparing all low-rank OT methods on the CIFAR-10 dataset, containing 60{,}000 images of size $32 \times 32 \times 3$ across 10 classes. We use a ResNet (resnet18-f37072fd.pth) to embed the images to dimension $d=512$ \cite{he2016deep} and apply a PCA to $d=50$, following the procedure of \cite{zhuang2023statistically}. We then perform a stratified 50/50 split of the images into two datasets of 30{,}000 images with class-label distributions matched. We use a fixed seed for this, as well as for the low-rank OT solvers following the \texttt{ott-jax} implementation of \cite{Scetbon2021LowRankSF}. For low-rank OT, we set the rank to $K=10$ to match the number of class labels. To run \ourmethodshort\, we solve for the coupling $\mb{P}_{\sigma^{\star}}$ with Hierarchical Refinement due to the size of the dataset \cite{halmos2025hierarchical}, and solve generalized $K$-means with mirror-descent. In this experiment, we specialize to the squared-Euclidean cost $\lVert \cdot - \cdot \rVert_{2}^{2}$.

For our evaluation metrics, we first compute the primal OT cost of each low-rank coupling as our primary benchmark. We also evaluate AMI and ARI to the ground-truth marginal clusterings, given by annotated class labels. We compute our predicted labels via the argmax assignment of labels as $\hat{y}(i) = \argmax_{z} \mb{Q}_{i, z}$ and $\hat{y}'(j) = \argmax_{z} \mb{R}_{j, z}$. Lastly, we assess co-clustering performance by using the class-transfer accuracy (CTA). Given a proposed coupling $\mb{P}$, define the class-to-class density matrix for two ground-truth classes $k,k'$ (distinguished from the predicted classes of the arg-max of the low-rank factors) to be 
\[
(\rho)_{k,k'} = \sum_{ij} \mb{P}_{ij} \mathbbm{1}_{i \in \mathcal{C}_{k}} \mathbbm{1}_{j \in \mathcal{C}_{k'}}
\]
The class-transfer accuracy is then defined to be
\begin{equation}
\mathrm{CTA}(\mb{P}) =\frac{\tr \rho }{ \sum \rho_{k,k'} }
\end{equation}
in other words, the fraction of mass transferred between ground-truth classes (i.e. the diagonal of $\rho$) over the total mass transferred between all class pairs.

\subsection{Single-Cell Transcriptomics of Mouse Embryogenesis}\label{sec:single_cell}

We validate \ourmethodshort\ against \LOT \cite{Scetbon2021LowRankSF} and \FRLC \cite{FRLC} on a recent, massive-scale dataset of single-cell mouse embryogenesis measured across 45 timepoint bins with combinatorial indexing (sci-RNA-seq3) \cite{shendure2024}. In aggregate, this dataset contains 12.4 million nuclei across timepoints and various replicates. As our experiment, we align the first replicate across 7 timepoints (E8.5, E8.75, E9.0, E9.25, E9.5, E9.75, E10.0) for a total of 6 adjacent timepoint pairs. For each timepoint pair, we use \texttt{scanpy} to read the h5ad file and follow standard normalization procedures: \texttt{sc.pp.normalize\_total} to normalize counts, \texttt{sc.pp.log1p} to add pseudocounts for stability, and run \texttt{sc.tl.pca} to perform a PCA projection of the raw expression data to the first $d=50$ principle components (using SVD solver "randomized"). As we use \cite{halmos2025hierarchical} as the full-rank OT solver, subsampling each dataset slightly to ensure that $n$ has many divisors for hierarchical partitioning. Similarly to the CIFAR evaluation, we ensure that the two datasets have a balanced proportion of classes -- which, in this case, represent cell-types annotated from \texttt{cell\_id} in the \texttt{df\_cell.csv} metadata provided in \cite{shendure2024}. We set the rank $K$ to be the minimum of the number of cell-types present at timepoint 1 and timepoint 2. We run \LOT, \FRLC, and \ourmethodshort\ on this data with the squared Euclidean cost. In both cases, we input the data as point clouds $X,Y$ as opposed to instantiating the cost $\mb{C}$ explicitly and specialize to the squared Euclidean cost.

\subsection{Estimation of Wasserstein Distances}

We generate sample data from the fragmented hypercube of \cite{forrow19a} to evaluate the performance of low-rank OT techniques on recovering the ground-truth EMD or squared Wasserstein distance. $|\hat{W}_2^2 - W_2^2|$. The fragmented hypercube is generated with $P_{0}$ ($\mu$) given by
\[
    P_{0} = \mathrm{Unif}([-1,1]^{d})
\]
And $P_{1} = T_{\sharp} P_{0}$ the push-forward of a transport map $T$ applied to $P_{0}$. Following \cite{forrow19a}, we take
\[
T(X) = X + 2 \cdot  \mathrm{sgn}(X) \odot (\mathbf{e}_{1} + \mb{e}_{2})
\]
With $\mathrm{sgn}(X)$ defined elementwise. Brenier's theorem \cite{Villani2003} allows one to compute the Wasserstein distance in this case, with a target value of $W_2^2 = 8$. Additional details on this experimental setup can be found in Section 6.1 of \cite{forrow19a}.

\begin{table}[t]
\caption{Estimation error $|\hat{W}_2^2 - W_2^2|$ for Varying Sample Size (d=30, k=10).}
\label{tab:results_n}
\begin{center}
\begin{small}
\begin{sc}
\begin{tabular}{lrrrrrrrr}
\hline
N & 29 & 36 & 44 & 54 & 66 & 80 & 98 & 119 \\
\hline
Full-Rank OT & 14.520 & 14.121 & 13.887 & 13.485 & 13.213 & 12.863 & 12.523 & 12.249 \\
K-Means & 9.616 & 8.540 & 7.887 & 7.087 & 6.729 & 6.196 & 5.828 & 5.579 \\
FactoredOT & 5.560 & 4.493 & 3.742 & 2.774 & 2.173 & 1.899 & 1.291 & 0.827 \\
FRLC & \textbf{2.449} & 2.448 & 2.188 & 1.745 & 0.920 & 0.995 & 0.598 & 0.526 \\
TC (Ours) & 3.009 & \textbf{2.344} & \textbf{1.385} & \textbf{0.784} & \textbf{0.502} & \textbf{0.357} & \textbf{0.342} & \textbf{0.242} \\
\hline
\end{tabular}
\end{sc}
\end{small}
\end{center}
\end{table}

\subsection{Additional Ablations}\label{sec:additional_ablations}

We provide a number of ablations to (1) understand the impact of using an entropy-regularized coupling $\mb{P}_{\epsilon}$ in place of a permutation $\mb{P}_{\sigma}$ in the Monge-registration step (Table~\ref{tab:sensitivity}, Figure~\ref{fig:eps_sensitivity}), (2) evaluating the value of our proposed initialization for existing low-rank OT algorithms (Table~\ref{tab:TC_init}), and (3) verifying the empirical performance of varying $n \neq m$ for Kantorovich registration (Table~\ref{tab:Kant_registration}). While further details may be found in Supplement~\ref{sec:additional_ablations}, we identify that low values of entropy significantly improve the performance of \ourmethod, find using our initialization improves the performance of previous low-rank OT algorithms, and observe that the advantage of Kantorovich registration holds when increasing the asymmetry in the dataset sizes.

\textbf{Kanotorovich registration} While our theoretical results apply to Monge registration, we validate the efficacy of the proposed extension, Kanotorovich registration, empirically. In particular, to assess the effect of asymmetry in the size of dataset one ($n$) and dataset two ($m$), we fix $n=1024$ and vary $m \in \{ 1024, 512, 256, 128, 64 \}$ for samples generated by the dataset of \cite{tong2023improving} (Table~\ref{tab:Kant_registration}). The relative difference between \texttt{TC} and \texttt{FRLC} and \texttt{LOT} is similar for both Monge and Kantorovich registration, even for $m$ over an order of magnitude different from $n$, providing evidence that the performance of \ourmethod\ generalizes to the Kantorovich setting.

\begin{table}[h!]
\centering
\caption{Low-rank OT costs with fixed $|X|=1024$ and varying dataset size $|Y|$ for task of \cite{tong2023improving} (Kantorovich registration).}\label{tab:Kant_registration}
\begin{tabular}{r|ccc}
\toprule
$|Y|$ & TC & FRLC & LOT \\
\midrule
64   & 8.983 & 10.508 & 10.108 \\
128  & 9.407 & 11.211 & 10.607 \\
256  & 8.831 & 10.289 & 10.153 \\
512  & 8.931 & 10.595 & 10.168 \\
1024 & 8.794 & 10.235 & 10.460 \\
\bottomrule
\end{tabular}
\end{table}

\textbf{Effect of Entropy Regularization} 

In practice, the computational complexity of solving for an optimal permutation $\mb{P}_{\sigma}$ ($\mathrm{nnz}(\mb{P}_{\sigma})=n$) or more generally an optimal solution to primal OT with  $\mathrm{nnz}(\mb{P})\leq n+m-1$ \cite{peyre2019computational} is limiting, often requiring that one use entropic regularization \cite{sinkhorn} as an alternative. For an entropic regularization parameter $\epsilon > 0$, we denote its associated solution by $\mb{P}_{\epsilon} = \arg\min_{\,\mb{P} \in \Pi(\bm{a},\bm{b})} \left\langle \mb{C}, \mb{P} \right\rangle_{F} - \epsilon \mathrm{H}(\mb{P})$. We assess the effect of varying the entropy regularization parameter $\epsilon_i=10^{i}$ with magnitudes $i \in \{-5, \cdots, 0, 1\}$ for this Sinkhorn step. Here, as $\epsilon_i$ increases (i.e. as the coupling moves away from optimal), the quality of the transport registration decreases and
    we obtain a higher final cost (See Table~\ref{tab:sensitivity}). Meanwhile, as $\epsilon_i$
    decreases the total final cost decreases, with a 5-order of magnitude gap in $\epsilon_i$ improving solution cost by a factor of two and a 7-order of magnitude gap improving the cost by a factor of three.


\begin{table}[h]
    \centering
\caption{Low-rank OT cost of Transport Clustering as a function of the Sinkhorn regularization~$\varepsilon$ used in the registration step.}\label{tab:sensitivity}
\begin{tabular}{r r}
\toprule
$\varepsilon$ & LR-OT Cost \\
\midrule
$10^{-5}$ & 5.050 \\
$10^{-4}$ & 5.620 \\
$10^{-3}$ & 8.167 \\
$10^{-2}$ & 8.919 \\
$10^{-1}$ & 9.415 \\
$10^{0}$  & 9.576 \\
$10^{1}$  & 14.538 \\
\bottomrule
\end{tabular}
\end{table}

\textbf{Validation of the \ourmethod\ Initialization.} As noted in ``Guarantees from Transport registered initialization with $K$-means and $K$-medians,'' by applying $K$-means to obtain the clusterings $\mb{Q}_{X}$ and $\mb{R}_Y$ and taking the minimum of the transport registered co-clusterings $(\mb{Q}_{X}, \mb{P}_{\sigma^{\star}}^{\top} \mb{Q}_{X})$ and $(\mb{P}_{\sigma^{\star}} \mb{R}_{Y}, \mb{R}_{Y})$
 already ensures a constant factor guarantee and only requires that a solver monotonically decreases the cost from this initialization to a local minimum. As a result, this initialization $(\mb{Q}_{0}, \mb{R}_{0})$ may also be applicable to existing low-rank OT solvers such as \texttt{LOT} \cite{Scetbon2021LowRankSF} or \texttt{FRLC} \cite{FRLC} with $(\mb{Q}_{0}, \mb{R}_{0})$ will also maintain the guarantee. We have added an ablation study in Table~\ref{tab:TC_init} to quantify the effect of using the transport registered initialization in another low-rank OT solver (FRLC). We find transport registration of the initialization accounts for the majority of the practical improvement across the methods (Table~\ref{tab:TC_init}), which is as expected from the theory.

 \begin{table}[h!]
\centering
\caption{Low-rank OT costs on Planted Gaussians ($k=250$, $\sigma = 0.1$, $n=2500$) for varying rank $r$. TC denotes Transport Clustering; FRLC (rand) uses a standard random initialization for low-rank OT, and FRLC (TC-init) uses TC-derived initialization.}\label{tab:TC_init}
\begin{tabular}{r|ccc}
\toprule
$r$ & TC (ours) & FRLC (rand) & FRLC (TC-init) \\
\midrule
50 & 8.0022 & 8.4583 & 8.0088 \\
100 & 7.7327 & 8.4360 & 7.7713 \\
150 & 7.5021 & 8.4550 & 7.5679 \\
200 & 7.2859 & 8.4490 & 7.3745 \\
250 & 7.0762 & 8.4448 & 7.2210 \\
\bottomrule
\end{tabular}
\end{table}

\begin{table}[t]
\centering
\caption{Single-cell transcriptomics alignment on consecutive mouse embryo timepoints. 
We report OT cost (lower is better), AMI/ARI for each split (A/B), and class-transfer accuracy (CTA; higher is better).}
\label{tab:sc-lr-ot}
\resizebox{\linewidth}{!}{%
\begin{tabular}{llc cccccc c}
\toprule
Timepoints & Method & Rank & OT Cost $\downarrow$ & AMI (A/B) $\uparrow$ & ARI (A/B) $\uparrow$ & CTA $\uparrow$ & Runtime (s) \\
\midrule
\multirow{3}{*}{\shortstack{E8.5 $\to$ E8.75 \\ (18{,}819 cells)}} 
  & \ourmethodshort\ & 43 & \textbf{0.506} & \textbf{0.639} / \textbf{0.617} & \textbf{0.329} / \textbf{0.307} & \textbf{0.722} & 63.38 \\
  & \FRLC       & 43 & 0.553 & 0.556 / 0.531 & 0.217 / 0.199 & 0.525 & 16.45 \\
  & \LOT        & 43 & 0.520 & 0.605 / 0.592 & 0.283 / 0.272 & 0.611 & 8.77 \\
\midrule
\multirow{3}{*}{\shortstack{E8.75 $\to$ E9.0 \\ (30{,}240 cells)}} 
  & \ourmethodshort\ & 53 & \textbf{0.384} & \textbf{0.597} / \textbf{0.598} & \textbf{0.231} / \textbf{0.230} & \textbf{0.545} & 177.12
  \\
  & \FRLC       & 53 & 0.405 & 0.534 / 0.541 & 0.174 / 0.178 & 0.492 & 16.92
  \\
  & \LOT        & 53 & 0.390 & 0.559 / 0.567 & 0.193 / 0.197 & 0.487 & 10.88 \\
\midrule
\multirow{3}{*}{\shortstack{E9.0 $\to$ E9.25 \\ (45{,}360 cells)}} 
  & \ourmethodshort\ & 57 & \textbf{0.452} & \textbf{0.563} / \textbf{0.554} & \textbf{0.190} / \textbf{0.187} & \textbf{0.500} & 286.95 \\
  & \FRLC       & 57 & 0.481 & 0.524 / 0.515 & 0.158 / 0.155 & 0.471 &  19.31 \\
  & \LOT        & 57 & --    & -- / -- & -- / -- & -- & -- \\
\midrule
\multirow{3}{*}{\shortstack{E9.25 $\to$ E9.5 \\ (75{,}600 cells)}} 
  & \ourmethodshort\ & 67 & \textbf{0.411} & \textbf{0.562} / \textbf{0.567} & \textbf{0.191} / \textbf{0.194} & \textbf{0.565} & 470.61
  \\
  & \FRLC       & 67 & 0.431 & 0.484 / 0.488 & 0.129 / 0.130 & 0.441 & 33.91 \\
  & \LOT        & 67 & --    & -- / -- & -- / -- & -- & -- \\
\midrule
\multirow{3}{*}{\shortstack{E9.5 $\to$ E9.75 \\ (131{,}040 cells)}} 
  & \ourmethodshort\ & 80 & \textbf{0.389} & \textbf{0.554} / \textbf{0.551} & \textbf{0.172} / \textbf{0.169} & \textbf{0.564} & 806.81 \\
  & \FRLC       & 80 & 0.399 & 0.491 / 0.487 & 0.116 / 0.115 & 0.447 & 58.58 \\
  & \LOT  & 80 & --    & -- / -- & -- / -- & -- & --
  \\
\midrule
\multirow{3}{*}{\shortstack{E9.75 $\to$ E10.0 \\ (120{,}960 cells)}} 
  & \ourmethodshort\ & 77 & \textbf{0.361} & \textbf{0.559} / \textbf{0.560} & \textbf{0.180} / \textbf{0.181} & \textbf{0.475} & 741.91
  \\
  & \FRLC       & 77 & 0.379 & 0.502 / 0.502 & 0.130 / 0.130 & 0.437 & 52.02 \\
  & \LOT        & 77 & --    & -- / -- & -- / -- & -- & -- \\
\bottomrule
\end{tabular}}
\end{table}

\newpage

\begin{figure}[t]
    \centering
    \includegraphics[width=0.48\linewidth]{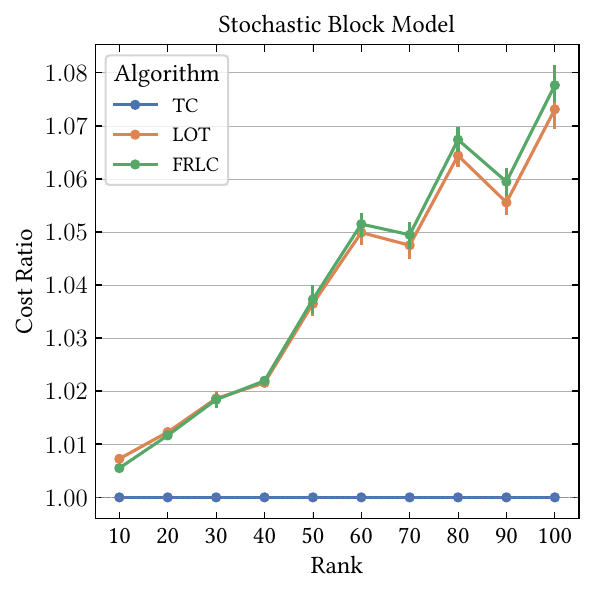}%
    \hfill
    \includegraphics[width=0.48\linewidth]{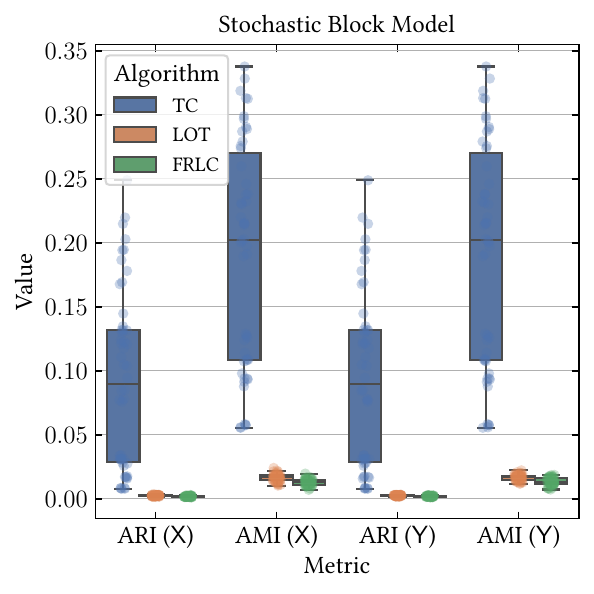}
    \caption{
        Comparison of low-rank OT methods on the stochastic block
        model dataset. \textbf{(Left)} Relative cost of the rank 
        $K \in \{10, \ldots, 100\}$ transport plan inferred by \LOT and \FRLC compared 
        to the cost of the transport plan inferred by \ourmethodshort. 
        \textbf{(Right)} Co-clustering accuracy (AMI/ARI) of \ourmethodshort, \LOT, and \FRLC at rank $K = 100$. 
        The stochastic block model dataset consists of $100$ clusters of size $50$. 
    }
    \label{fig:sbm_cost}
\end{figure}

\begin{figure}[t]
    \centering
    \includegraphics[width=1.0\linewidth]{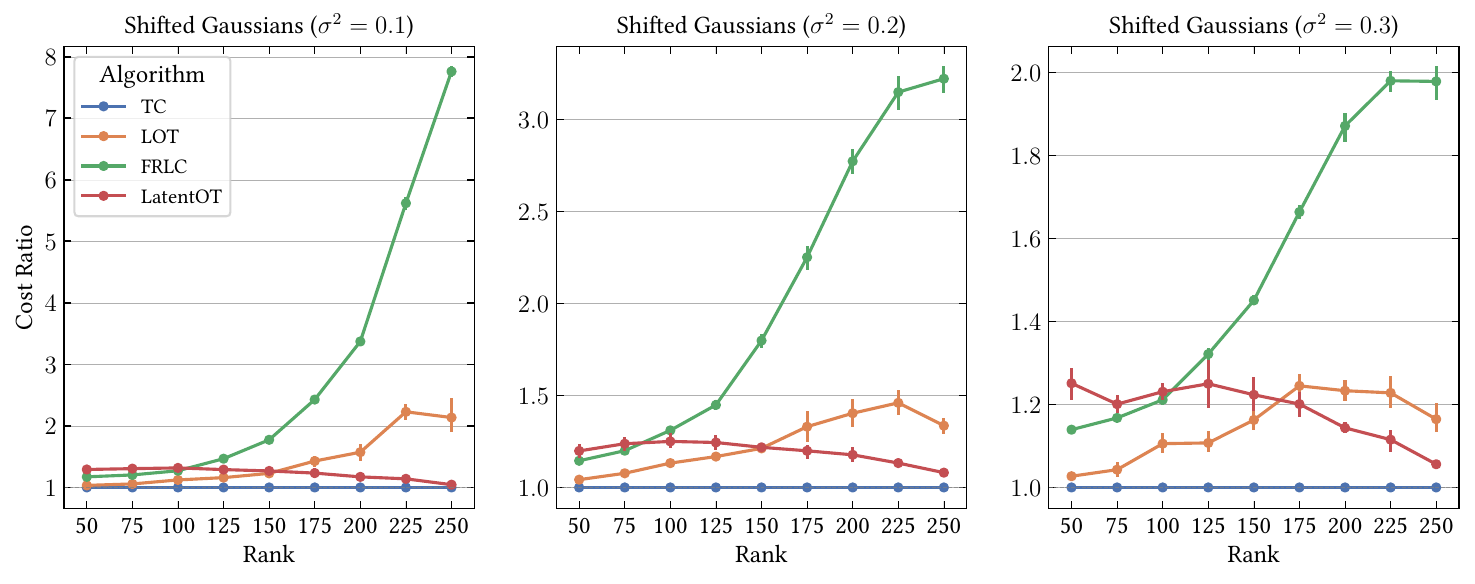}
    \includegraphics[width=1.0\linewidth]{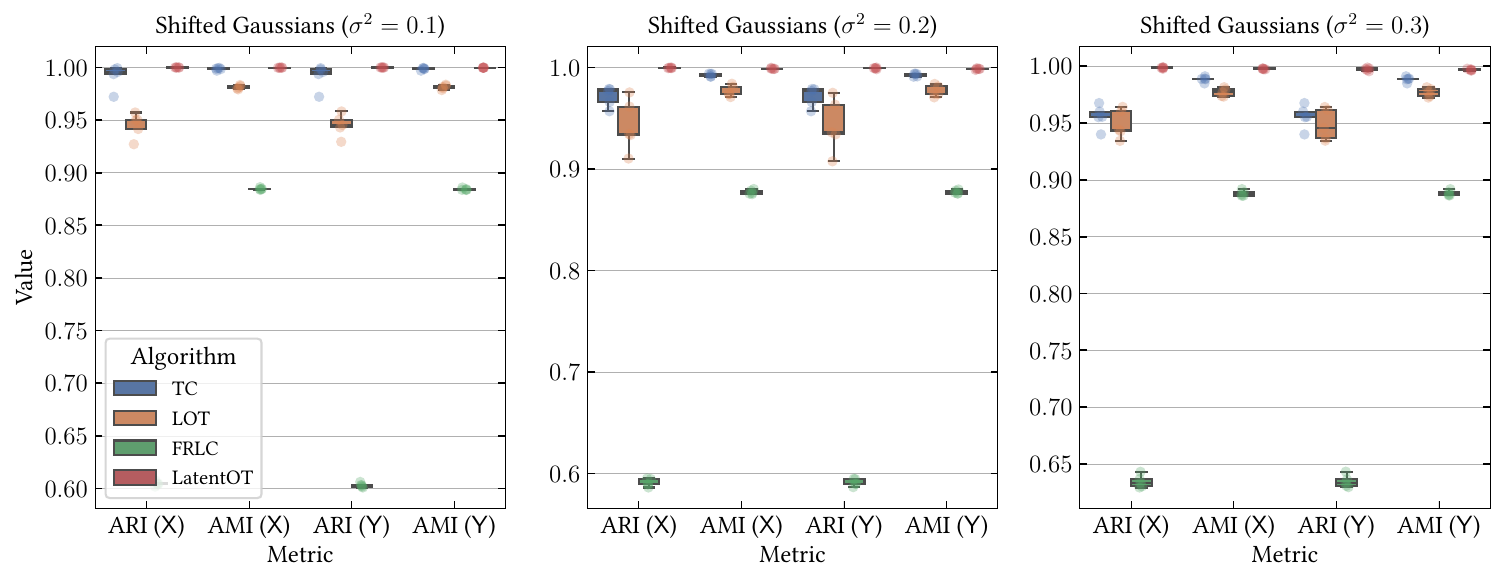}
    \caption{
        Comparison of low-rank OT methods on the shifted
        Gaussians dataset. \textbf{(Top)} Relative cost of the rank 
        $K \in \{50, 75, \ldots, 250\}$ transport plan inferred by \LOT, \FRLC, and \LIN compared 
        to the cost of the transport plan inferred by \ourmethodshort
        across noise levels $\sigma^2\in\{0.1,0.2, 0.3\}$. 
        \textbf{(Bottom)} Co-clustering accuracy (AMI/ARI) of \ourmethodshort, \LOT, \FRLC, and 
        \LIN at rank $K = 250$ across noise levels $\sigma^2\in\{0.1,0.2, 0.3\}$. 
        The shifted Gaussians dataset consists of $250$ clusters of 
        unequal size. 
    }
    \label{fig:shifted_gaussians}
\end{figure}

\begin{figure}[t]
    \centering
    \includegraphics[width=1.0\linewidth]{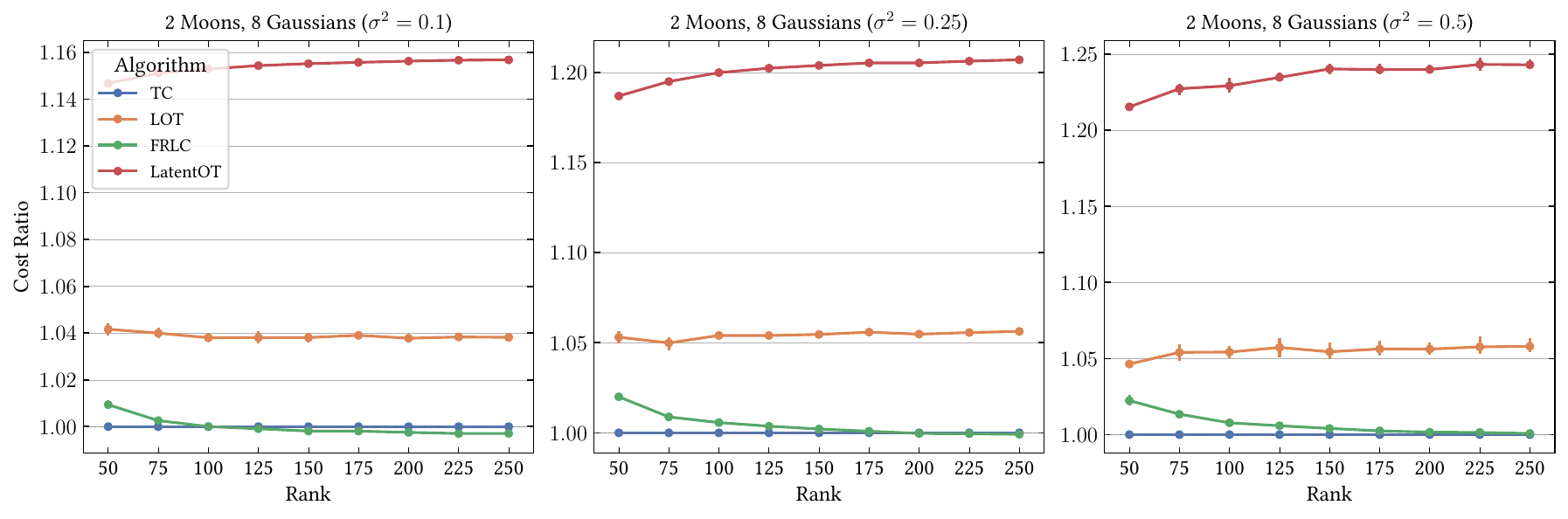}
    \caption{
        Relative cost of the rank 
        $K \in \{50, 75, \ldots, 250\}$ transport plan inferred by \LOT, \FRLC, and \LIN compared 
        to the cost of the transport plan inferred by \ourmethodshort
        across noise levels $\sigma^2\in\{0.1,0.2, 0.3\}$
        for the 2-Moons and 8-Gaussians \citep{tong2023improving}
        dataset.
    }
    \label{fig:twomoons}
\end{figure}

\begin{figure}[t]
    \centering
    \includegraphics[width=1.0\linewidth]{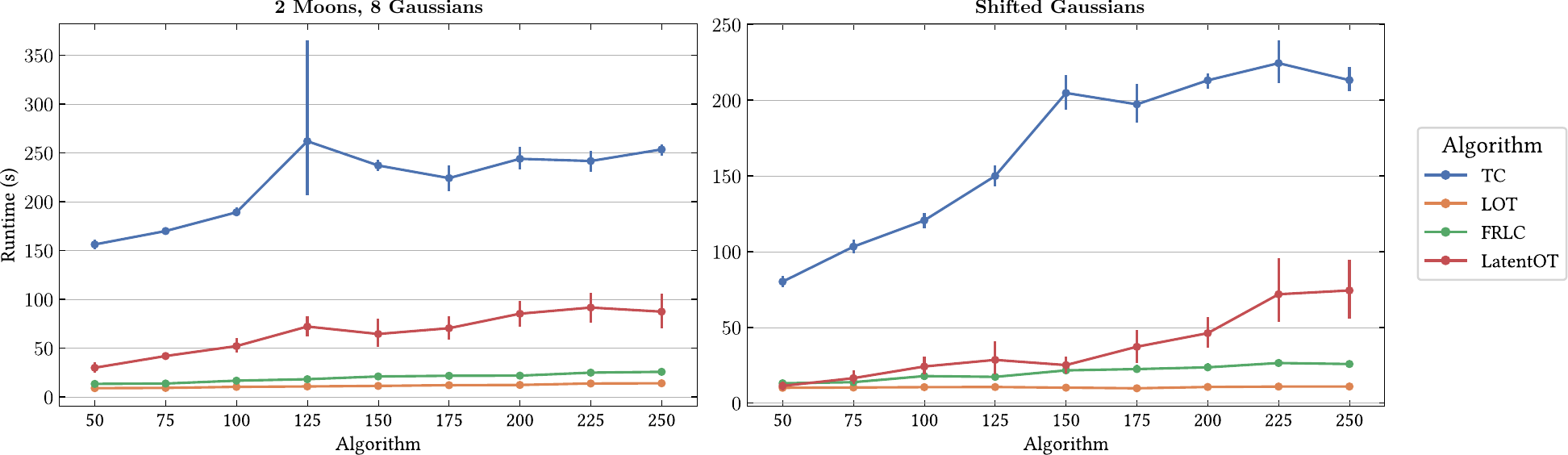}
    \caption{
        Runtime of \ourmethodshort, \LOT, \FRLC, and \LIN versus the rank 
        $K \in \{50, 75, \ldots, 250\}$
        for the 2-Moons and 8-Gaussians \citep{tong2023improving}
        dataset and the Shifted Gaussians dataset across
        all noise levels.
    }
    \label{fig:twomoons}
\end{figure}

\begin{figure}[tbp]
\centerline{\includegraphics[width=\linewidth]{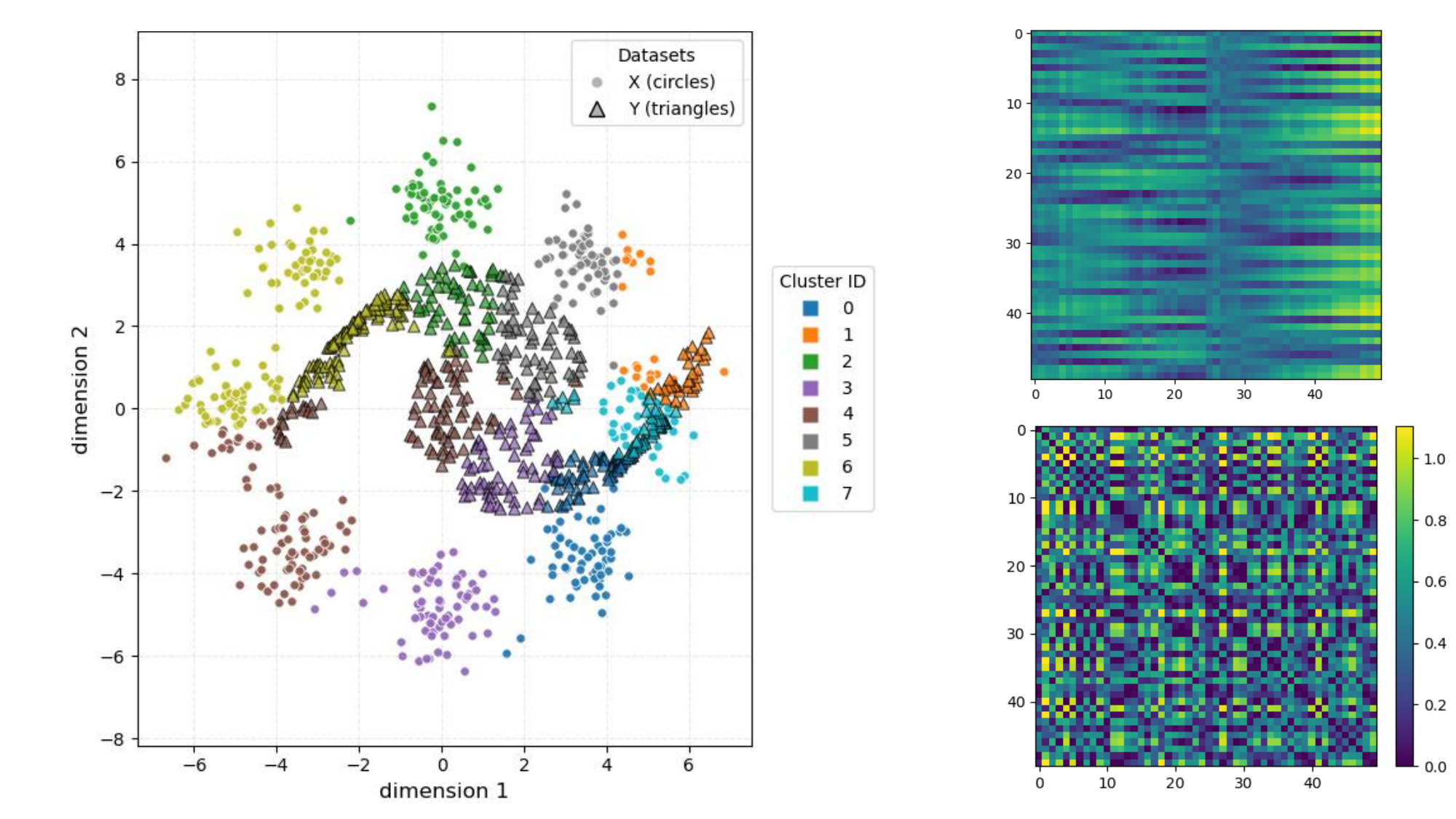}}
\caption{(Left) An example co-clustering of the two-moons 8-Gaussians dataset \cite{tong2023improving} with Algorithm~\ref{alg:exact_GKMS_reduction}. (Right) A comparision between the raw cost matrix $\mb{C}$ (top), and the transport conjugated cost $\mb{M}^{\dagger}$ (bottom).
}
\label{fig:trans_conj_cost}
\end{figure}

\begin{figure}[t]
    \centering
    \includegraphics[width=0.8\linewidth]{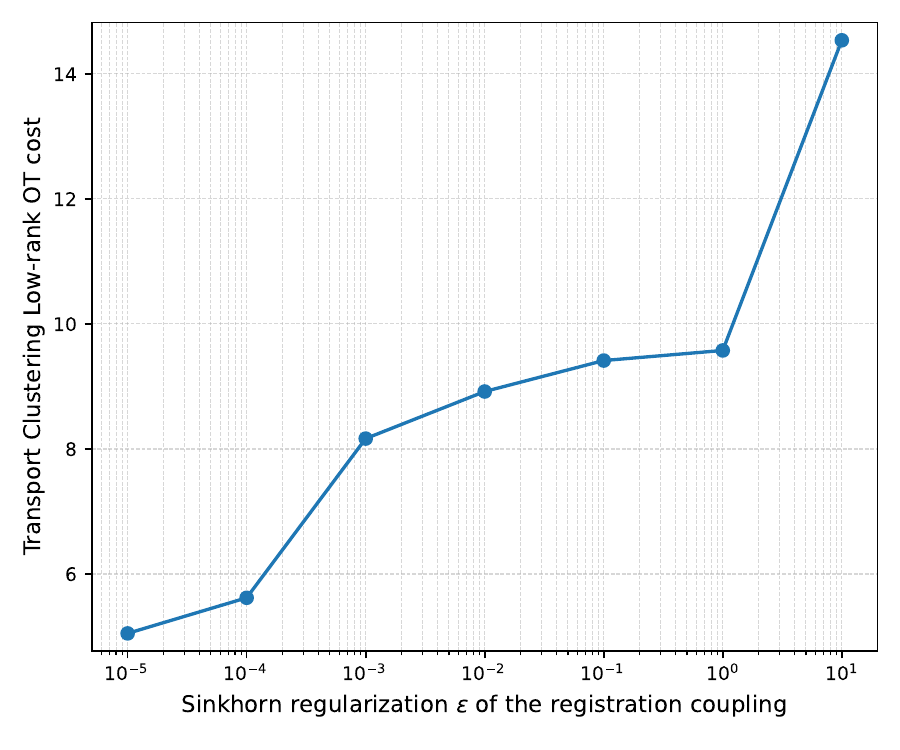}
    \caption{
        (Sensitivity of \ourmethodshort\ to error in the coupling) Low-Rank OT Cost of Transport Clustering as a function of entropy-regularization scale $\epsilon$. Lower $\epsilon$ is closer to an optimal full-rank solution.
    }\label{fig:eps_sensitivity}
\end{figure}

\end{document}